\documentclass[]{mosi}

\usepackage{helvet}
\usepackage[utf8]{inputenc}
\usepackage{url}
\usepackage{array}
\usepackage{float}
\makeatletter
\let\openbox\@undefined
\makeatother
\usepackage{amsthm}
\usepackage{mathtools}
\usepackage{tabularx,makecell,threeparttable}
\usepackage{siunitx}
\usepackage{wrapfig}
\usepackage{needspace}
\usepackage{enumitem}
\usepackage{algpseudocode}
\usepackage[export]{adjustbox}
\usepackage{tikz,pgfplots}
\usetikzlibrary{arrows.meta,positioning,calc,fit,backgrounds,decorations.pathreplacing}
\pgfplotsset{compat=1.18}
\usepackage[hang,flushmargin]{footmisc}
\usepackage{bm}

\input{tex/table_typography}

\newtheorem{theorem}{Theorem}
\newtheorem{proposition}[theorem]{Proposition}
\newtheorem{lemma}[theorem]{Lemma}
\newtheorem{corollary}[theorem]{Corollary}
\theoremstyle{definition}

\theoremstyle{remark}
\newtheorem{remark}[theorem]{Remark}

\definecolor{oursgray}{gray}{0.95}
\definecolor{h3colblue}{HTML}{EAF2FF}
\definecolor{lightblue}{RGB}{200,230,255}
\definecolor{headerblue}{RGB}{150,200,255}

\floatstyle{ruled}
\newfloat{methodalgorithm}{tbp}{loa}
\floatstyle{plain}
\floatname{methodalgorithm}{Algorithm}

\hypersetup{
  pdftitle={Analytic-Walk Rotary Positional Encodings for Graphs},
  pdfauthor={Jiaqing Xie, Yuxin Wang}
}

\title{Analytic-Walk Rotary Positional Encodings for Graphs}

\author{
Jiaqing Xie$^{1,2}$, Yuxin Wang$^{1,\dagger}$, Xipeng Qiu$^{1,2,\dagger}$
\\[2mm]
{\normalfont \normalsize 26113050148@m.fudan.edu.cn, wangyuxin@sii.edu.cn, xpqiu@sii.edu.cn}\\
{\normalfont \normalsize $^{\dagger}$Corresponding author}\\
{\normalfont \normalsize $^{1}$Shanghai Innovation Institute}\\
{\normalfont \normalsize $^{2}$Fudan University}
}

\abstract{
Rotary position encodings make attention sensitive to relative position,
but extending them to graphs requires choosing how graph structure enters
the rotation. Previous works assign each node a rotation from spectral
coordinates, so the rotary
factor between two nodes depends only on their endpoints and cannot
distinguish the routes connecting them. We introduce
\textit{Analytic-Walk Rotary Positional Encodings} (AW-RoPE), which place
the rotations on edges and sum the transported features over all walks,
so contributions along different routes can reinforce or cancel. An
exact variant evaluates the complete sum by a differentiable linear
solve, and a sparse variant truncates it at a finite depth. We prove
forward and parameter-derivative truncation bounds at fixed inputs and
parameters. Both variants act on projected queries and keys, and the
sparse recurrence also augments message-passing networks. Across five
synthetic tasks both variants reduce nRMSE by $15$--$58\%$ relative to
the strongest baseline, and on real superpixel, peptide and OGB
benchmarks the sparse recurrence attains the best mean on every dataset
with Performer kernels and on twelve of thirteen datasets with GIN. Analysis shows that  AW-RoPE can distinguish routes whose only cue is how two endpoints are
connected, while node-wise rotary
encodings cannot.

\checkdata[Code]{\url{https://github.com/jiaqingxie/AW-RoPE}}
}

\begin{document}
\maketitle

\section{Introduction}
\textit{Positional encodings} (PEs) specify how order enters a
Transformer's otherwise permutation-equivariant attention, in two broad
styles: \emph{absolute} encodings attach a descriptor to each input
\citep{vaswani2017attention}, and
\emph{relative} encodings modify the interaction between pairs
\citep{shaw2018relative}. Graphs
stress this distinction \citep{rampasek2022graphgps}. With no canonical node ordering, the absolute
route replaces indices with structural descriptors such as spectral
coordinates and random-walk features
\citep{dwivedi2020graphtransformer,kreuzer2021san,dwivedi2022lspe,lim2023signnet},
while the relative route encodes pairwise graph relationships directly
in attention \citep{ying2021graphormer,ma2023grit}.

\textit{Rotary positional encoding} (RoPE) realizes the relative style by rotating
query and key channel pairs so that their dot product depends on
relative coordinates \citep{su2021roformer}, building on relative
position representations for sequences.
Previous works extend this idea from
sequence positions to \textit{Laplacian} spectral coordinates
\citep{reid2026graphrope}, bringing global graph information into
node-wise rotations while remaining compatible with linear attention.

\begin{figure}[!htb]
 \centering
 \resizebox{.88\linewidth}{!}{
\begingroup
\definecolor{aiBlue}{HTML}{2878B5}
\definecolor{aiOrange}{HTML}{D88038}
\definecolor{aiInk}{HTML}{273645}
\definecolor{aiSum}{HTML}{168879}
\begin{tikzpicture}[x=1cm,y=1cm,
    font=\sffamily\fontsize{9}{10}\selectfont,text=aiInk,
    nodepoint/.style={circle,draw=aiInk!65,fill=white,minimum size=3.9mm,inner sep=0pt},
    route/.style={line width=1.1pt,-{Latex[length=1.2mm]}},
    anglelabel/.style={font=\fontsize{8}{9}\selectfont,inner sep=1pt},
    heading/.style={font=\sffamily\bfseries\fontsize{10}{11}\selectfont,anchor=west},
    vector/.style={line width=1.2pt,-{Latex[length=1.3mm]}}]
 \node[heading] at (.10,.63) {(a) WIRE};
 \node[heading] at (3.80,.63) {(b) AW operator: prescribed edge phases};
 \draw[aiInk!18] (3.50,.40) -- (3.50,-3.70);
 \node at (1.70,.14) {Node-phase differences};
 \node at (5.35,.14) {Same phases};
 \node at (9.00,.14) {Lower edges $+\pi/4$};
 \node at (12.65,.14) {Lower edges $+\pi/2$};

 \foreach \name/\center/\upperleft/\upperright/\lowerleft/\lowerright in {
    w/1.70/{\pi/4}/{\pi/4}/{\pi/4}/{\pi/4},
    c0/5.35/{\pi/4}/{\pi/4}/{\pi/4}/{\pi/4},
    c1/9.00/{\pi/4}/{\pi/4}/{\pi/2}/{\pi/2},
    c2/12.65/{\pi/4}/{\pi/4}/{3\pi/4}/{3\pi/4}}{
  \begin{scope}[xshift=\center cm]
   \node[nodepoint] (\name-u) at (-1.16,-1.10) {$u$};
   \node[nodepoint] (\name-a) at (0,-.49) {$a$};
   \node[nodepoint] (\name-b) at (0,-1.71) {$b$};
   \node[nodepoint] (\name-v) at (1.16,-1.10) {$v$};
   \draw[route,aiBlue] (\name-u) -- node[anglelabel,above,sloped] {$\upperleft$} (\name-a);
   \draw[route,aiBlue] (\name-a) -- node[anglelabel,above,sloped] {$\upperright$} (\name-v);
   \draw[route,aiOrange] (\name-u) -- node[anglelabel,below,sloped] {$\lowerleft$} (\name-b);
   \draw[route,aiOrange] (\name-b) -- node[anglelabel,below,sloped] {$\lowerright$} (\name-v);
  \end{scope}}

 \foreach \center/\thetaone/\lowerangle/\gap in {
   1.70/{\pi/2}/{\pi/2}/0, 5.35/{\pi/2}/{\pi/2}/0,
   9.00/{\pi/2}/\pi/{\pi/2}, 12.65/{\pi/2}/{3\pi/2}/\pi}{
   \node at (\center,-2.11) {\textcolor{aiBlue}{$\theta_1=\thetaone$},\quad\textcolor{aiOrange}{$\theta_2=\lowerangle$}};
   \node[font=\sffamily\bfseries\fontsize{9}{10}\selectfont] at (\center,-2.45) {$\Delta=\gap$};
  }

 \node at (1.70,-3.00) {Node phases $(0,\tfrac{\pi}{4},\tfrac{\pi}{4},\tfrac{\pi}{2})$};
 \node[font=\sffamily\bfseries\fontsize{11}{12}\selectfont] at (1.70,-3.35) {$\Delta\equiv0$};

 \foreach \center/\deg/\power in {5.35/90/1.000,9.00/180/0.500,12.65/270/0.000}{
   \begin{scope}[shift={(\center-.58,-3.18)}]
    \draw[aiInk!14] (0,0) circle (.38);
    \draw[aiInk!14] (-.48,0) -- (.86,0);
    \draw[aiInk!14] (0,-.45) -- (0,.48);
    \draw[vector,aiOrange,line width=2pt] (0,0) -- (\deg:.38);
    \draw[vector,aiBlue] (0,0) -- (90:.38);
    \pgfmathsetmacro{\sumx}{.38*(cos(90)+cos(\deg))}
    \pgfmathsetmacro{\sumy}{.38*(1+sin(\deg))}
    \ifnum\deg=270
     \node[circle,draw=aiSum,fill=white,line width=1.1pt,inner sep=1.9pt] at (0,0) {};
    \else
     \draw[vector,aiSum,densely dashed] (0,0) -- (\sumx,\sumy);
    \fi
   \end{scope}
  \node[font=\sffamily\fontsize{8}{9}\selectfont,anchor=center] at (\center+.85,-2.99) {Power};
  \node[font=\sffamily\bfseries\fontsize{11}{12}\selectfont,text=aiSum] at (\center+.85,-3.33) {\power};
 }
 \draw[aiInk!15] (.10,-3.78) -- (14.20,-3.78);
 \node[anchor=west] at (.10,-4.08) {$\Delta=\theta_2-\theta_1$,\quad $P_{\rm AW}=\left|\tfrac12\mathrm{e}^{\mathrm{i}\theta_1}+\tfrac12\mathrm{e}^{\mathrm{i}\theta_2}\right|^2=\cos^2(\Delta/2)$};
 \draw[vector,aiBlue] (9.00,-4.08) -- (9.42,-4.08);
 \node[anchor=west,font=\sffamily\fontsize{8}{9}\selectfont] at (9.47,-4.08) {upper};
 \draw[vector,aiOrange] (11.00,-4.08) -- (11.42,-4.08);
 \node[anchor=west,font=\sffamily\fontsize{8}{9}\selectfont] at (11.47,-4.08) {lower};
 \draw[vector,aiSum,densely dashed] (13.00,-4.08) -- (13.42,-4.08);
 \node[anchor=west,font=\sffamily\fontsize{8}{9}\selectfont] at (13.47,-4.08) {sum};
\end{tikzpicture}
\endgroup}
  \caption{\textbf{Node-phase cancellation and controlled edge-phase changes.}
  Edge labels are phases in radians ($\omega=1$). (a) With node-wise
  rotations, every route between the same endpoints accumulates the same
  phase, so the route gap $\Delta=\theta_2-\theta_1$ is zero for any node
  phases. (b) The first AW example uses
  exactly the same phases. The next two add $\pi/4$ or $\pi/2$ to each forward
  lower edge. Reverse phases negate. The powers are normalized AW responses
  with prescribed phases, without fitting the scorer.}
 \label{fig:aw-insight}
\end{figure}
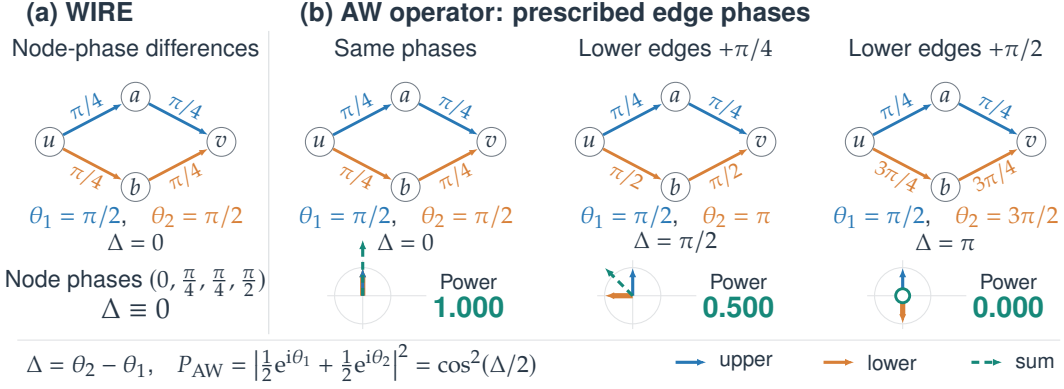

Graph relationships can involve several routes between the same
endpoints. With one phase per node, relative rotations telescope to the
endpoint phase difference, so routes receive identical accumulated
phases (Figure~\ref{fig:aw-insight}). Spectral coordinates still encode
cycles and global structure; the restriction concerns how this
information acts on features, motivating learned edge rotations whose
contributions along different routes combine.

We introduce \emph{Analytic-Walk Rotary Positional Encodings} (AW-RoPE) to
implement this idea. A shared scorer reads the features of both endpoints
and learns antisymmetric edge displacements, so reversing an edge reverses
its rotation. Successive rotate-and-aggregate steps accumulate phases
along walks, and a discounted sum combines their feature contributions,
so different routes can reinforce or cancel. \emph{Exact} AW-RoPE computes
the complete walk action with a differentiable linear solve, and \emph{Sparse}
AW-RoPE truncates the sum after $L$ steps with sparse rotate-and-aggregate
recurrences. The operator transforms projected queries and keys before
attention, and the sparse recurrence also serves as a residual adapter for
message passing (Figure~\ref{fig:aw-pipeline}).

Our contributions are:
\begin{itemize}[leftmargin=*,itemsep=2pt,topsep=3pt,parsep=0pt]
 \item \textbf{A route-sensitive rotary action.} Learned edge rotations
  make a route's phase a sum of edge terms, so phases accumulated around a
  cycle (the cycle circulation) can be nonzero, which node-wise
  phases provably cannot express
  (Proposition~\ref{prop:nodewise-obstruction}).
 \item \textbf{Algorithms with guarantees.} We prove forward and
  parameter-derivative error bounds at fixed inputs and parameters for the
  exact solve and its finite-depth truncation, and empirically verify
  that the bounds hold at all 30 trained checkpoints (sampled local VJPs
  on validation inputs).
  \item \textbf{Evaluation across prediction, transfer, and cost.} Both
   variants obtain lower mean error than the strongest baseline on all
   five synthetic tasks, with nRMSE reductions of $15$--$58\%$, and on
   two constructed route-discrimination tasks only dynamically generated
   edge phases succeed: all phase-free, fixed-field, and node-level
   controls remain at chance, while \emph{Sparse} AW-RoPE reaches
    $98.9\%$. The sparse recurrence improves over the strongest baseline
    in linear attention and over the NoPE baseline in GIN message passing.
\end{itemize}

\section{Related Work}
\paragraph{Graph positional and rotary encodings.}
Laplacian eigenvectors supply node coordinates
\citep{dwivedi2020graphtransformer}, which SAN feeds to spectral attention
\citep{kreuzer2021san} and SignNet makes invariant to sign and basis
choices \citep{lim2023signnet}. LSPE decouples positional from structural
streams, covering random-walk encodings \citep{dwivedi2022lspe}, and
Graph-RoPE's RWPE variant reads features off return probabilities
\citep{reid2026graphrope}. Systematic benchmarks compare these encodings
across GNN and graph-transformer backbones \citep{grotschla2026benchmark}. Graphormer injects pairwise distances and
shortest paths into attention scores \citep{ying2021graphormer}, sparse and tokenized variants extend the family to larger graphs, and
GraphGPS couples such attention with local message passing
\citep{shirzad2023exphormer,chen2023nagphormer,rampasek2022graphgps}. RoPE rotates query and key channel pairs so
that dot products depend on relative coordinates \citep{su2021roformer}.
WIRE lifts the rotation angles to spectral coordinates, linking rotary
phases to grid positions and effective resistance
\citep{reid2026graphrope}. In all of these the rotary factor is
endpoint-determined. AW-RoPE instead learns edge rotations and sums walk
contributions, so phases can vary with intermediate edges.

\paragraph{Learned transport, diffusion, and attention.}
Vector diffusion maps propagate features through edge transformations
\citep{singer2012vdm}, an idea directional graph networks turn into
learned edge directions \citep{beaini2021dgn}. MagNet's magnetic
Laplacian similarly attaches complex phases to edge direction
\citep{zhang2021magnet}, and neural sheaf diffusion learns restriction
maps between node and edge spaces \citep{bodnar2022sheaf}, with polynomial
variants enabling sparse evaluation \citep{borgi2025polynomial}. MagNet prescribes phases from graph directionality for a Hermitian
spectral operator. AW-RoPE instead generates antisymmetric phases from
current endpoint features and applies this input-dependent connection to
rotary query and key channels through discounted walks. One
parameterization supports an exact resolvent and sparse truncation with
approximation guarantees, as PPNP and APPNP do for undirected diffusion \citep{klicpera2019ppnp}
(Appendix~\ref{sec:connection-relation}).
Transport is orthogonal to the attention kernel: AW acts on queries and
keys before Performer's random features or any other feature map
\citep{choromanski2021performer}, and a gated residual adapter equips GIN
\citep{xu2019gin} with optional ReZero initialization
\citep{bachlechner2021rezero}.

\section{Methods}
\label{sec:methods}
\subsection{Preliminaries: From RoPE to WIRE}
\label{sec:core-preliminaries}

Let $\mathcal{G}=(\mathcal{V},\mathcal{E})$ have node set $\mathcal V$ and
directed-edge set $\mathcal E$, with $N=|\mathcal V|$ and $M=|\mathcal E|$.
An undirected edge is stored in both directions.
We use single-head notation and omit projection biases.
An attention layer receives node features $\mathbf{H}\in\mathbb{R}^{N\times d}$
of even rotary width $d$ and forms $\mathbf{Q}=\mathbf{H}\mathbf{W}_Q$,
$\mathbf{K}=\mathbf{H}\mathbf{W}_K$, $\mathbf{V}=\mathbf{H}\mathbf{W}_V$ with
$\mathbf W_Q,\mathbf W_K,\mathbf W_V\in\mathbb R^{d\times d}$; column
$u$ of each transposed feature matrix is written
$\mathbf h_u,\mathbf q_u,\mathbf k_u,\mathbf v_u$. We use $\mathbf{K}$ for
keys and the nonnegative integer $L$ for maximum walk depth.

RoPE rotates the $d/2$ adjacent query and key channel pairs.
For pair $\ell$, write $\mathbf q_{u\ell},\mathbf k_{u\ell}\in\mathbb R^2$
at node $u$, and let
$\mathbf{R}(\alpha)=\bigl[\begin{smallmatrix}\cos\alpha&-\sin\alpha\\
\sin\alpha&\cos\alpha\end{smallmatrix}\bigr]$ denote the
counterclockwise rotation. Conjugation cancels the leading angle,
$\mathbf{R}(\alpha)^\top\mathbf{R}(\beta)=\mathbf{R}(\beta-\alpha)$.
If node $u$ has coordinates $\mathbf{r}_u$, let $\boldsymbol{\nu}_\ell$
be the rotary frequency vector for channel pair $\ell$, with the same
dimension as $\mathbf{r}_u$ and shared across nodes and between Q and K.
The resulting phase is $\phi_{u\ell}=\boldsymbol{\nu}_\ell^\top \mathbf{r}_u$.
Applying this identity to the rotated inputs gives
\begin{equation}
 (\mathbf{R}(\phi_{u\ell})\mathbf{q}_{u\ell})^\top
 (\mathbf{R}(\phi_{v\ell})\mathbf{k}_{v\ell})
 =\mathbf{q}_{u\ell}^\top \mathbf{R}(\phi_{v\ell}-\phi_{u\ell})\mathbf{k}_{v\ell}.
 \label{eq:prelim-relative}
\end{equation}
Thus an absolute node phase produces a relative rotary factor.
For sequences, $\mathbf{r}_u$ is the token position. WIRE uses graph spectral
coordinates, for example components of selected Laplacian eigenvectors,
optionally with eigenvalue-dependent scaling
\citep{su2021roformer,reid2026graphrope}. This identity concerns the
raw dot product. The nonlinear feature map used in linear attention does
not commute with rotations, so the relative phase is exact only for
dot-product attention.

This construction induces an endpoint-dependent rotary phase. Along a route
$p=(v_0=u,\ldots,v_k=v)$ of length $k$ (the number of traversed edges),
the differences of node phases telescope:
\begin{equation}
 \sum_{j=0}^{k-1}(\phi_{v_{j+1},\ell}-\phi_{v_j,\ell})
 =\phi_{v,\ell}-\phi_{u,\ell}.
 \label{eq:prelim-telescope}
\end{equation}
Routes with the same endpoints therefore have identical accumulated
phases regardless of their intermediate nodes and edges
(Proposition~\ref{prop:nodewise-obstruction}): the route gap
$\Delta=\theta_2-\theta_1$ in Figure~\ref{fig:aw-insight}(a) is always zero.
AW-RoPE instead assigns edge rotations and aggregates features along walks,
allowing these routes to contribute different phases.

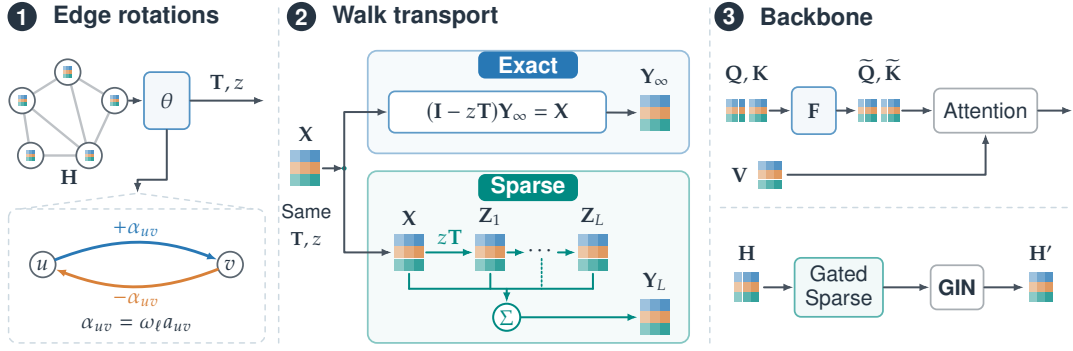
\begin{figure}[!htb]
 \centering
 \resizebox{0.88\linewidth}{!}{
\begingroup
\definecolor{hwInk}{HTML}{263442}
\definecolor{hwBlue}{HTML}{2878B5}
\definecolor{hwTeal}{HTML}{008A82}
\definecolor{hwOrange}{HTML}{D88038}
\definecolor{hwEdge}{HTML}{CDD9E1}
\begin{tikzpicture}[x=1cm,y=1cm,
 font=\sffamily\fontsize{8.5}{9.5}\selectfont,text=hwInk,
 line cap=round,line join=round,
 flow/.style={-{Latex[length=1.45mm]},draw=hwInk!85,line width=.8pt},
 guide/.style={draw=hwEdge,dash pattern=on 2pt off 2.4pt,line width=.6pt},
 title/.style={anchor=west,font=\sffamily\bfseries\fontsize{9.5}{10.5}\selectfont},
 badge/.style={circle,fill=hwInk,text=white,inner sep=0pt,minimum size=4.1mm,
   font=\sffamily\bfseries\fontsize{9}{10}\selectfont},
 vertex/.style={circle,draw=hwInk!75,fill=white,line width=.7pt,
   inner sep=0pt,minimum size=3.6mm},
 block/.style={draw=hwInk!40,fill=white,line width=.7pt,rounded corners=.9mm,
   minimum height=6.2mm,inner sep=3pt},
 tile/.style={minimum width=4.8mm,minimum height=6mm,inner sep=0pt,outer sep=0pt},
 tilelabel/.style={font=\fontsize{8.3}{9.3}\selectfont,inner sep=1pt},
 lane/.style={rounded corners=1.3mm,line width=.7pt},
 choice/.style={rounded corners=1mm,text=white,minimum width=14mm,
   minimum height=3.8mm,inner sep=1pt,font=\sffamily\bfseries\fontsize{9}{10}\selectfont},
 pics/feature/.style={code={
   \foreach \yy/\col in {0/hwBlue,1/hwOrange,2/hwTeal}{
     \foreach \xx/\tone in {0/45,1/65,2/80}
       \fill[\col!\tone] (-.21+.14*\xx,.25-.17*\yy) rectangle ++(.13,-.16);
   }
 }}]

\foreach \x/\n/\heading in {.20/1/{Edge rotations},4.12/2/{Walk transport},10.12/3/{Backbone}}{
 \node[badge] at (\x,0) {\n};
 \node[title] at (\x+.31,0) {\heading};
}
\draw[guide] (3.80,.12) -- (3.80,-4.58);
\draw[guide] (9.85,.12) -- (9.85,-4.58);

\foreach \name/\xx/\yy in {u/.20/-1.18,a/.84/-.78,v/1.50/-1.18,b/1.14/-1.95,c/.32/-1.95}{
 \coordinate (input-\name) at (\xx,\yy);
}
\foreach \s/\t in {u/a,a/v,v/b,b/c,c/u,a/b,u/b}
 \draw[hwInk!30,line width=.95pt] (input-\s) -- (input-\t);
\foreach \name in {u,a,v,b,c}{
 \node[vertex] at (input-\name) {};
 \pic[scale=.30] at (input-\name) {feature};
}
\node[tilelabel] at (.86,-2.25) {$\mathbf H$};
\node[block,minimum width=6.4mm,minimum height=8mm,
 draw=hwBlue!80,fill=hwBlue!5] (field-scorer) at (2.23,-1.18) {$\theta$};
\draw[flow] (1.70,-1.18) -- (field-scorer);
\draw[flow] (field-scorer.east) -- node[above=2pt,tilelabel] {$\mathbf T,z$} (3.60,-1.18);
\draw[flow] (field-scorer.south) -- (2.23,-2.29) -- (1.82,-2.29) -- (1.82,-2.48);
\draw[guide] (.18,-2.70) -- (1.82,-2.48) -- (3.44,-2.70);
\draw[guide,rounded corners=1.2mm] (.01,-2.70) rectangle (3.61,-4.58);
\node[vertex] (edge-u) at (.49,-3.48) {$u$};
\node[vertex] (edge-v) at (3.12,-3.48) {$v$};
\draw[flow,hwBlue,line width=1.15pt] (edge-u) to[bend left=21]
 node[above=1pt,tilelabel] {$+\alpha_{uv}$} (edge-v);
\draw[flow,hwOrange,line width=1.15pt] (edge-v) to[bend left=21]
 node[below=1pt,tilelabel] {$-\alpha_{uv}$} (edge-u);
\node[tilelabel] at (1.81,-4.33) {$\alpha_{uv}=\omega_\ell a_{uv}$};

\node[tile] (eval-input) at (4.18,-2.13) {};
\pic at (eval-input.center) {feature};
\node[tilelabel,above=1pt] at (eval-input.north) {$\mathbf X$};
\node[align=center,font=\sffamily\fontsize{7.6}{9}\selectfont] at (4.19,-2.94) {Same\\[2pt]$\mathbf T,z$};
\coordinate (eval-split) at (4.72,-2.13);
\draw[flow] (eval-input) -- (eval-split);
\fill[hwTeal] (eval-split) circle (.033cm);

\draw[lane,draw=hwBlue!48,fill=hwBlue!4] (5.05,-.48) rectangle (9.56,-1.92);
\node[choice,fill=hwBlue] at (7.30,-.68) {Exact};
\node[block,minimum width=30.5mm,minimum height=5.8mm,draw=hwBlue!75]
 (exact-solve) at (6.87,-1.34) {$(\mathbf I-z\mathbf T)\mathbf Y_\infty=\mathbf X$};
\node[tile] (exact-out) at (9.08,-1.35) {};
\pic at (exact-out.center) {feature};
\node[tilelabel,above=1pt] at (exact-out.north) {$\mathbf Y_\infty$};
\draw[flow] (eval-split) |- (exact-solve.west);
\draw[flow] (exact-solve) -- (exact-out);

\draw[lane,draw=hwTeal!48,fill=hwTeal!4] (5.05,-2.17) rectangle (9.56,-4.58);
\node[choice,fill=hwTeal] at (7.30,-2.40) {Sparse};
\node[tile] (state-0) at (5.63,-3.31) {};
\node[tile] (state-1) at (6.77,-3.31) {};
\node[inner sep=1pt] (state-mid) at (7.49,-3.31) {$\cdots$};
\node[tile] (state-L) at (8.21,-3.31) {};
\foreach \name in {state-0,state-1,state-L}
 \pic at (\name.center) {feature};
\node[tilelabel,above=1pt] at (state-0.north) {$\mathbf X$};
\node[tilelabel,above=1pt] at (state-1.north) {$\mathbf Z_1$};
\node[tilelabel,above=1pt] at (state-L.north) {$\mathbf Z_L$};
\draw[flow] (eval-split) |- (state-0.west);
\draw[flow,hwTeal] (state-0) -- node[above=2pt,tilelabel] {$z\mathbf T$} (state-1);
\draw[flow,hwTeal] (state-1) -- (state-mid);
\draw[flow,hwTeal] (state-mid) -- (state-L);
\foreach \name in {state-0,state-1,state-L}
 \draw[hwTeal,line width=.8pt] (\name.south) -- (\name.center |- 0,-3.82);
\draw[hwTeal,line width=.7pt,densely dotted] (state-mid.south) -- (state-mid.center |- 0,-3.82);
\draw[hwTeal,line width=.8pt] (5.63,-3.82) -- (8.21,-3.82);
\node[circle,draw=hwTeal,fill=white,line width=.8pt,minimum size=3.8mm,
 inner sep=0pt,text=hwTeal] (sparse-sum) at (7.00,-4.22) {$\Sigma$};
\draw[flow,hwTeal] (7.00,-3.82) -- (sparse-sum);
\node[tile] (sparse-out) at (9.08,-4.22) {};
\pic at (sparse-out.center) {feature};
\node[tilelabel,above=1pt] at (sparse-out.north) {$\mathbf Y_L$};
\draw[flow,hwTeal] (sparse-sum) -- (sparse-out);

\node[tile,minimum width=6.4mm] (attn-qk) at (10.37,-1.32) {};
\pic[scale=.64] at ($(attn-qk.center)+(-.16,0)$) {feature};
\pic[scale=.64] at ($(attn-qk.center)+(.16,0)$) {feature};
\node[tilelabel,above=1pt] at (attn-qk.north) {$\mathbf Q,\mathbf K$};
\node[block,minimum width=6.2mm,minimum height=6.6mm,
 draw=hwBlue!75,fill=hwBlue!5] (attn-aw) at (11.31,-1.32) {$\mathbf F$};
\node[tile,minimum width=6.4mm] (attn-qk-out) at (12.23,-1.32) {};
\pic[scale=.64] at ($(attn-qk-out.center)+(-.16,0)$) {feature};
\pic[scale=.64] at ($(attn-qk-out.center)+(.16,0)$) {feature};
\node[tilelabel,above=1pt] at (attn-qk-out.north) {$\widetilde{\mathbf Q},\widetilde{\mathbf K}$};
\node[block,minimum width=14.5mm] (attn-backbone) at (13.72,-1.32) {Attention};
\draw[flow] (attn-qk) -- (attn-aw);
\draw[flow] (attn-aw) -- (attn-qk-out);
\draw[flow] (attn-qk-out) -- (attn-backbone);
\draw[flow] (attn-backbone.east) -- (14.94,-1.32);
\node[tile] (attn-v) at (10.70,-2.22) {};
\pic[scale=.78] at (attn-v.center) {feature};
\node[tilelabel,left=1pt] at (attn-v.west) {$\mathbf V$};
\draw[flow] (attn-v) -| (attn-backbone.south);
\draw[guide] (10.00,-2.68) -- (14.94,-2.68);

\node[tile] (gin-h) at (10.37,-3.80) {};
\pic[scale=.82] at (gin-h.center) {feature};
\node[tilelabel,above=1pt] at (gin-h.north) {$\mathbf H$};
\node[block,draw=hwTeal!65,fill=hwTeal!5,align=center,
 minimum width=12.4mm,minimum height=7.7mm] (gin-adapter) at (11.66,-3.80) {Gated\\Sparse};
\node[block,minimum width=6.6mm,font=\sffamily\bfseries\fontsize{8.5}{9.5}\selectfont]
 (gin-backbone) at (13.32,-3.80) {GIN};
\node[tile] (gin-out) at (14.50,-3.80) {};
\pic[scale=.82] at (gin-out.center) {feature};
\node[tilelabel,above=1pt] at (gin-out.north) {$\mathbf H'$};
\draw[flow] (gin-h) -- (gin-adapter);
\draw[flow] (gin-adapter) -- (gin-backbone);
\draw[flow] (gin-backbone) -- (gin-out);
\end{tikzpicture}
\endgroup}
 \caption{\textbf{AW-RoPE implementation.}
 (1) A shared scorer maps node features $\mathbf H$ to edge displacements
 $a_{uv}$, with $\omega_\ell$ the channel frequency and $\theta$ collecting
 learned parameters. (2) For fixed input $\mathbf X$, one-step
 transport $\mathbf T$ and decay $z$, \emph{Exact} solves for $\mathbf Y_\infty$, with identity $\mathbf I$.
 \emph{Sparse} sums $\mathbf Z_k=z^k\mathbf T^k\mathbf X$ from depth zero through
 $L$ to approximate it by $\mathbf Y_L$. (3) Either operator $\mathbf F$ transforms queries and keys,
 while V bypasses transport. GIN uses a gated \emph{Sparse} residual with output
 $\mathbf H'$. Tiles denote features, and channel indices are suppressed.}
 \label{fig:aw-pipeline}
\end{figure}

\subsection{From Edge Rotations to Multi-Route Transport}
\label{sec:core-overview}

The subscript $\theta$ denotes dependence on trainable AW parameters
 including scorer weights, frequencies and decay. A shared scalar scorer $s_\theta$
reads the current endpoint features, and $[\cdot,\cdot]$ denotes concatenation
along the feature dimension. With a fixed bound $a_{\max}>0$, define the
score difference $b_{uv}$ and bounded edge displacement $a_{uv}$ by
\begin{equation}
 b_{uv}=s_\theta([\mathbf{h}_u,\mathbf{h}_v])-s_\theta([\mathbf{h}_v,\mathbf{h}_u]),\qquad
 a_{uv}=a_{\max}\tanh(b_{uv}/a_{\max}).
 \label{eq:core-edge}
\end{equation}
Oddness of $\tanh$ gives $a_{vu}=-a_{uv}$, so reversing an edge
inverts its rotation. Sharing $s_\theta$ makes the displacements equivariant
to node relabeling (Lemma~\ref{lem:field-equivariance}) and the operator
permutation-equivariant (Proposition~\ref{prop:permutation}). The field is computed once
per transport call and held fixed across its walk depths and query and key actions.

Nonnegative edge weights $w_{uv}$ scale neighbor $v$'s contribution to
node $u$. Row normalization gives $\mathbf P\in\mathbb R^{N\times N}$,
$\mathbf P_{uv}=w_{uv}/\sum_x w_{ux}$; absent edges and rows without
outgoing weight remain zero. Channel $\ell$ uses a real frequency
$\omega_\ell$ and rotation $\mathbf R(\omega_\ell a_{uv})$.
Pairing real coordinates as $x_{\rm r}+\mathrm i x_{\rm i}$ (real/imaginary
parts, $\mathrm i^2=-1$) represents this rotation by
$\mathrm e^{\mathrm i\omega_\ell a_{uv}}$ without a complex-valued network.
Suppressing $\ell$, let $\mathbf X\in\mathbb C^N$ collect channel inputs
across nodes. The one-step matrix $\mathbf T_\theta\in\mathbb C^{N\times N}$ is
\begin{equation}
 [\mathbf{T}_\theta]_{uv}=\mathbf{P}_{uv}\mathrm{e}^{\mathrm{i}\omega a_{uv}},\qquad
 (\mathbf{T}_\theta \mathbf{X})_u=\sum_{v:(u,v)\in\mathcal{E}}
          \mathbf{P}_{uv}\mathrm{e}^{\mathrm{i}\omega a_{uv}}\mathbf{X}_v .
 \label{eq:core-one-step}
\end{equation}
This is one rotate-and-aggregate step. An index $u\to v$ denotes a
walk leaving $u$, but the row action gathers $v$'s feature into $u$:
we compute $\mathbf{T}\mathbf{X}$, not $\mathbf{T}^\top \mathbf{X}$.

Applying the same step twice makes the path interpretation explicit:
\begin{equation}
 (\mathbf{T}_\theta^2\mathbf{X})_u
 =\sum_v\sum_w \mathbf{P}_{uv}\mathbf{P}_{vw}
       \mathrm{e}^{\mathrm{i}\omega(a_{uv}+a_{vw})}\mathbf{X}_w .
 \label{eq:core-two-step}
\end{equation}
Each $v$ defines a length-two route from $u$ to $w$.
For a general walk $p=(v_0=u,\ldots,v_k=v)$ of $|p|=k$ edges, define
its weight $\pi(p)=\prod_{j=0}^{k-1}\mathbf{P}_{v_jv_{j+1}}$ and
accumulated displacement $A(p)=\sum_{j=0}^{k-1}a_{v_jv_{j+1}}$. Then
\begin{equation}
 (\mathbf{T}_\theta^k\mathbf{X})_u
 =\sum_v\ \sum_{\substack{p:u\leadsto v\\ |p|=k}}
       \pi(p)\mathrm{e}^{\mathrm{i}\omega A(p)}\mathbf{X}_v .
 \label{eq:core-walks}
\end{equation}
Walks may repeat vertices and edges (induction in
Lemma~\ref{lem:walk-expansion}).
Unlike Eq.~\eqref{eq:prelim-telescope}, general learned $A(p)$ need not
depend only on the endpoints. Summing complex contributions can reinforce
or cancel different routes. The total action also mixes amplitudes and
is generally not a single orthogonal rotation.

\subsection{Complete Walk Action and Sparse Truncation}
\label{sec:core-aw}

A learned decay $0\le z_\theta<1$ weights length-$k$ walks by
$z_\theta^k$, with edge weights inside $\mathbf T_\theta^k$.
Set $\mathbf B_\theta=z_\theta\mathbf T_\theta$; the $N\times N$ identity
$\mathbf I=\mathbf T_\theta^0$ retains $\mathbf X$ at depth zero.
Omitting $\theta$, define $\mathbf F_L=\sum_{k=0}^L\mathbf B^k\in\mathbb C^{N\times N}$
and $\mathbf Y_L=\mathbf F_L\mathbf X\in\mathbb C^N$. Cancellation gives
\begin{equation}
 (\mathbf{I}-\mathbf{B})\sum_{k=0}^L\mathbf{B}^k=\mathbf{I}-\mathbf{B}^{L+1}.
 \label{eq:core-telescope}
\end{equation}
The limit $L\to\infty$ follows from a geometric bound
(Proposition~\ref{prop:neumann}). For matrices,
$\|\cdot\|_\infty$ denotes the induced infinity norm, the maximum absolute
row sum. Because every phase has modulus one,
$\|\mathbf{T}\|_\infty=\max_u\sum_v|\mathbf{T}_{uv}|\le1$, and hence
$\|\mathbf{B}^{L+1}\|_\infty\le z^{L+1}\to0$.
The same geometric bound makes the sums Cauchy, defining the complete
walk operator $\mathbf F_\infty=\lim_{L\to\infty}\mathbf F_L$ and its
output $\mathbf Y_\infty$. Taking limits in Eq.~\eqref{eq:core-telescope} gives
\begin{equation}
 \mathbf{F}_\infty=(\mathbf{I}-\mathbf{B})^{-1},\qquad
 \mathbf{Y}_\infty=\mathbf{F}_\infty \mathbf{X},\qquad (\mathbf{I}-\mathbf{B})\mathbf{Y}_\infty=\mathbf{X} .
 \label{eq:core-exact}
\end{equation}
\emph{Exact} AW-RoPE evaluates this action up to numerical error with a
differentiable linear solve per graph and rotary channel.

\emph{Sparse} AW-RoPE evaluates the finite sum without constructing a
dense matrix. Let $\mathbf S^{(k)}$ be the state after $k$ transport steps,
before walk-length decay, and $\mathbf Y^{(k)}$ the accumulated output through step $k$.
Initialize $\mathbf{S}^{(0)}=\mathbf{X}$ and $\mathbf{Y}^{(0)}=\mathbf{X}$, then compute
\begin{equation}
 \mathbf{S}^{(k+1)}=\mathbf{T}_\theta \mathbf{S}^{(k)},\qquad
 \mathbf{Y}^{(k+1)}=\mathbf{Y}^{(k)}+z_\theta^{k+1}\mathbf{S}^{(k+1)}
 \quad (k=0,\ldots,L-1).
 \label{eq:core-recurrence}
\end{equation}
Induction gives $\mathbf{S}^{(k)}=\mathbf{T}_\theta^k\mathbf{X}$. The weighted
states of Figure~\ref{fig:aw-pipeline} are $\mathbf Z_k=z_\theta^k\mathbf S^{(k)}$. The returned
$\mathbf{Y}^{(L)}=\mathbf Y_L=\mathbf{F}_L\mathbf{X}$ includes every depth from zero through $L$.
Both versions share the field, frequency and decay parameterization,
initialize frequencies with the usual RoPE schedule and cap learned decay
below one; they evaluate the complete or truncated sum, respectively.

The Exact implementation uses $\mathcal{O}(N^3d)$ forward work and $\mathcal{O}(N^2d)$
storage. For $L\ge1$, \emph{Sparse} uses $\mathcal{O}(L(N+M)d)$ work and $\mathcal{O}((N+M)d)$ streaming
storage (Proposition~\ref{prop:complexity}).
These costs exclude field construction, the backbone, and additional autodiff storage.
\emph{Sparse} transport has graph-linear cost at fixed width and depth, although
backpropagation may retain intermediate sparse states and dense attention
still costs $\mathcal{O}(N^2d)$, so graph-linear end-to-end scaling requires a
linear-attention or sparse message-passing backbone.

\subsection{Using AW-RoPE in the Backbone}
\label{sec:core-integration}

For a \textit{Graph Transformer}, $\mathbf{F}=\mathbf{F}_\infty$ or
$\mathbf{F}_L$ transforms projected queries and keys with the same edge
field. Let $\mathbf{q}^{(\ell)},\mathbf{k}^{(\ell)}\in\mathbb C^N$ collect
one paired channel across nodes, and $\mathbf F^{(\ell)}$ denote the
operator at frequency $\omega_\ell$. Tildes mark transported features:
\begin{equation}
 \widetilde{\mathbf{q}}^{(\ell)}=\mathbf{F}^{(\ell)}\mathbf{q}^{(\ell)},\qquad
 \widetilde{\mathbf{k}}^{(\ell)}=\mathbf{F}^{(\ell)}\mathbf{k}^{(\ell)},\qquad
 \widetilde{\mathbf{V}}=\mathbf{V}.
 \label{eq:core-qk}
\end{equation}
Converting these channels back to real features gives the matrices
$\widetilde{\mathbf Q},\widetilde{\mathbf K}$. Ordinary dense attention uses
$\operatorname{softmax}(\widetilde{\mathbf{Q}}\widetilde{\mathbf{K}}^{\top}/\sqrt d)\mathbf{V}$.
For linear attention, the chosen feature map $\varphi$ is applied to the
transported queries and keys before the usual factorized contraction.
Transport precedes $\varphi$ whatever feature map it implements
(offsets and batching in Appendix~\ref{app:linear-attention-interface}).

For a real feature matrix, operator application means pairing channels,
applying each channel's operator and unpacking. Our evaluated GIN adapter
uses \emph{Sparse} AW-RoPE with a learned scalar gate $\gamma$. Write
$\mathbf H_{\rm in}$ for the GIN block's input and $\mathbf H_{\rm next}$
for its output:
\begin{equation}
 \mathbf{H}_{\rm in}=\mathbf{H}+\gamma(\mathbf{F}_L\mathbf{H}-\mathbf{H}),\qquad
 \mathbf{H}_{\rm next}=\operatorname{GIN}(\mathbf{H}_{\rm in}).
 \label{eq:core-gin}
\end{equation}
The gate $\gamma$ controls transport; ReZero initializes it to zero,
recovering the base network. The GIN convolution, residual and
normalization are retained.

\subsection{Exact--Sparse Convergence with Increasing Walk Depth}
\label{sec:core-guarantees}

The comparison fixes parameters and input. Let $0<r<1$ bound
$\|\mathbf B_\theta\|_\infty$. Subtracting the partial sum from the
complete action leaves only omitted walks:
\begin{equation}
 \mathbf{F}_\infty-\mathbf{F}_L=\mathbf{B}^{L+1}(\mathbf{I}-\mathbf{B})^{-1},\qquad
 \|\mathbf{F}_\infty-\mathbf{F}_L\|_\infty
 \le\frac{r^{L+1}}{1-r}\quad\text{if }\|\mathbf{B}_\theta\|_\infty\le r<1.
 \label{eq:core-forward-bound}
\end{equation}

Thus \emph{Sparse} converges to \emph{Exact} in operator norm.
At differentiability points, $\mathrm D\mathbf B_\theta[\mathbf v]$
denotes the parameter-directional derivative for Euclidean unit
$\mathbf v$. On a parameter region with uniform
$\|\mathbf B_\theta\|_\infty\le r<1$, let $G\ge0$ bound
$\|\mathrm D\mathbf B_\theta[\mathbf v]\|_\infty$ for every such direction.
The matrix product rule gives
$\mathrm{D}(\mathbf{B}^k)[\mathbf{v}]=\sum_{j=0}^{k-1}\mathbf{B}^j \mathrm{D}\mathbf{B}[\mathbf{v}]\mathbf{B}^{k-1-j}$.
Applying this rule to the remainder gives
\begin{equation}
 \|\mathrm{D}(\mathbf{F}_\infty-\mathbf{F}_L)[\mathbf{v}]\|_\infty
 \le G\,\frac{r^L((L+1)-Lr)}{(1-r)^2}.
 \label{eq:core-guarantees}
\end{equation}
Both bounds vanish for fixed $r<1$ and bounded $G$ (proof in
Appendix~\ref{app:exact-aw-training}); near-unit decay can require greater
depth. The derivative covers learned fields, frequencies and decay through
$\mathrm D\mathbf B_\theta[\mathbf v]$, but parameter-dependent inputs
add product-rule terms outside its scope.
Section~\ref{sec:numerical-theory} checks outputs and sampled local
vector--Jacobian products (VJPs: fixed readout vectors contracted with
local Jacobians) at trained checkpoints with fixed parameters and local
inputs. The theorem concerns the ordinary unnormalized pair;
Appendix~\ref{app:method-details} gives normalization, non-backtracking
and gating variants. These are same-parameter approximation guarantees;
training and prediction are compared experimentally.

\section{Experiments}
\label{sec:experiments}

\subsection{Experimental Questions and Common Protocol}
\label{sec:evaluation-protocol}
We ask whether AW improves prediction (Q1), whether phases add gains
beyond phase-free propagation (Q2), how closely \emph{Sparse} approximates
\emph{Exact} at fixed parameters (Q3), and how AW transfers across datasets
and backbones and scales to large graphs (Q4). Variants combine a walk
evaluation (\emph{Exact} or \emph{Sparse}) with a field (learned
unrestricted, gradient-derived, zero, or fixed). We use the public
GraphGPS and Graph-RoPE stack and follow the Graph-RoPE protocol on
the synthetic and real-data benchmarks
\citep{rampasek2022graphgps,reid2026graphrope}, with \emph{Sparse} in
real-data attention and GIN studies. Test metrics
use validation-best checkpoints; independently trained predictive
comparisons are separate from same-parameter operator checks, and
published WIRE values are external references. Appendices~\ref{app:hyperparameters}
and \ref{app:exact-aw-experiments} give hyperparameters and selection rules.

\begin{table}[H]
 \centering
  \caption{\textbf{Synthetic graph regression.} Test nRMSE $\downarrow$;
  published SEs, local sample SDs. \emph{Exact} and \emph{Sparse} use a
  four-layer, width-32 dense-attention GraphGPS backbone
  (Table~\ref{tab:synthetic-hyperparameters}); NoPE and WIRE are
  published reference means under this protocol \citep{reid2026graphrope}. Bold compares
  the matched pair.}
  \label{tab:synthetic-results}
  \footnotesize
  \setlength{\tabcolsep}{3pt}
  \begin{tabular}{@{}lllll@{}}
  \toprule
  Dataset & NoPE & WIRE & \emph{Exact} & \emph{Sparse}\\
 \midrule
 Monochromatic-0 & $0.060_{\pm 0.001}$ & $0.053_{\pm 0.002}$ & $\mathbf{0.044}_{\pm 0.006}$ & $0.045_{\pm 0.009}$ \\
Monochromatic-5 & $0.087_{\pm 0.001}$ & $0.068_{\pm 0.005}$ & $\mathbf{0.036}_{\pm 0.003}$ & $0.036_{\pm 0.002}$ \\
Monochromatic-10 & $0.081_{\pm 0.001}$ & $0.063_{\pm 0.002}$ & $\mathbf{0.027}_{\pm 0.000}$ & $0.032_{\pm 0.009}$ \\
Monochromatic-15 & $0.068_{\pm 0.002}$ & $0.056_{\pm 0.004}$ & $\mathbf{0.026}_{\pm 0.003}$ & $0.028_{\pm 0.003}$ \\
Watts--Strogatz SPD & $0.065_{\pm 0.005}$ & $0.038_{\pm 0.006}$ & $\mathbf{0.016}_{\pm 0.001}$ & $0.017_{\pm 0.005}$ \\
\bottomrule

 \end{tabular}
\end{table}
\subsection{Predictive Performance on Controlled Synthetic Tasks}
\label{sec:synthetic-experiments}

\paragraph{Setup.}

Four Graph-RoPE tasks predict the largest monochromatic connected
subgraph on edge-deleted grids; the fifth predicts Watts--Strogatz
shortest-path distances. All AW variants receive the same auxiliary LapPE
input and are evaluated by normalized RMSE (nRMSE; lower is better).
\emph{Exact} and \emph{Sparse} train separately with identical initializations,
frozen task-specific configurations and equal parameter counts;
\emph{Sparse} uses $L=8$ or $16$
(Table~\ref{tab:synthetic-hyperparameters}).

\begin{figure}[!htb]
 \centering
 \begin{minipage}[c]{0.52\textwidth}
  \centering
  \includegraphics[width=.98\linewidth, trim=240 0 0 0, clip]{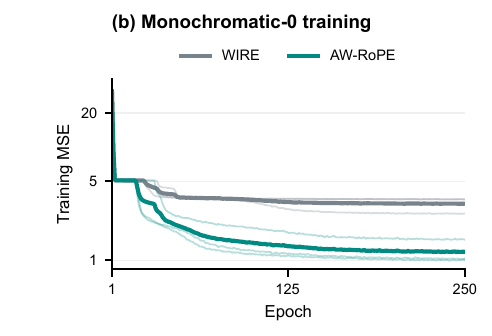}
  \captionof{figure}{\textbf{Training behavior on Monochromatic-0.} Training MSE
  of the matched WIRE and AW mixing study over all 250 epochs, thin
  individual curves and thick means, without smoothing.}
  \label{fig:diagnostics}
 \end{minipage}\hfill
 \begin{minipage}[c]{0.45\textwidth}
  \centering
  \includegraphics[width=\linewidth]{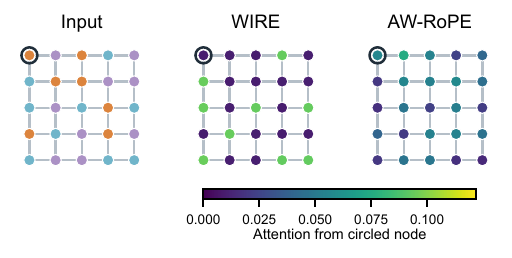}
  \captionof{figure}{Monochromatic-0 attention from the circled node.}
  \label{fig:synthetic-main-visuals}
 \end{minipage}
\end{figure}

\paragraph{Matched Exact--Sparse results.}
Table~\ref{tab:synthetic-results} answers Q1. Both arms improve on the published NoPE and WIRE means on all five
tasks, and \emph{Sparse}, the practical variant, retains most of
\emph{Exact}'s predictive benefit despite truncating the walk sum; the
archived fixed-field controls likewise improve on these references
(Appendix~\ref{app:static-flat}). \emph{Exact} is faster on small graphs
(Table~\ref{tab:exact-sparse-efficiency}); \emph{Sparse}'s motivation is
scaling to larger graph sizes.

\subsection{What Does Multi-Route Phase Modeling Add?}
\label{sec:phase-mechanism}

\paragraph{Ablation variants.}
Tables~\ref{tab:inductive-fields} and~\ref{tab:route-interference}
separate propagation, node phases and edge-field structure.
Mixing (Zero) uses $D_{L,z}=\sum_{k=0}^{L}z^k\mathbf P^k$ without
phases; WIRE uses node-wise rotations, and WIRE~+~mixing rotates after
applying $D_{L,z}$. Table~\ref{tab:inductive-fields} also tests
input-independent transport through the learned static polynomial and
random fixed field. Dynamic local gradient uses
$a_{uv}=b(\mathbf h_v)-b(\mathbf h_u)$, retaining feature dependence
but forcing zero circulation. Dynamic nonlocal Flat projects AW's field
onto gradients using the whole cycle, testing zero circulation without
the pointwise-potential restriction. ``Phases zeroed'' sets $a_{uv}=0$
in a fitted AW model while freezing all learned parameters, testing
reliance on phases rather than retraining Mixing (Zero).
\emph{Sparse} AW truncates the walk sum; \emph{Exact} AW uses the
complete resolvent.

\begin{table}[!t]
\centering
\small
\setlength{\tabcolsep}{4pt}
\caption{\textbf{Feature-conditioned prediction on unseen cycle sizes.}
Test accuracy (\%), mean$_{\pm\mathrm{SD}}$ over five matched runs.
All arms use a width-16, one-block full softmax-attention backbone.
The edge-field arms use maximum walk depth $L=2$, $z=0.8$, and
$\mathbf F^{(\ell)}=\sum_{k=0}^{2}z^k(\mathbf T^{(\ell)})^k$ on
queries and keys, where
$[\mathbf T^{(\ell)}]_{uv}=\mathbf P_{uv}\mathrm{e}^{\mathrm{i}\omega_\ell a_{uv}}$.
Here $D=I+z\mathbf P+z^2\mathbf P^2$.
The learned static polynomial and WIRE arms use the operators shown
below; the WIRE formulas also apply to keys.
Architecture and protocol in Appendix~\ref{app:inductive-fields}.}
\label{tab:inductive-fields}
\begin{adjustbox}{max width=\linewidth}
\begin{tabular}{@{}lll@{}}
\toprule
Method & Field or transport & Accuracy $\uparrow$\\
\midrule
Mixing (Zero) & $a_{uv}\equiv 0$ (phase-free $D$) & $50.00_{\pm 0.00}$\\
WIRE & $\mathbf q'_{u\ell}=\mathbf R(\boldsymbol\nu_\ell^{\top}\mathbf r_u)\mathbf q_{u\ell}$ & $49.97_{\pm 0.21}$\\
WIRE + mixing & $\mathbf q'_{u\ell}=\mathbf R(\boldsymbol\nu_\ell^{\top}\mathbf r_u)(D\mathbf q)_{u\ell}$ & $50.47_{\pm 1.07}$\\
Learned static polynomial & $\mathbf Q'=\sum_{k\le 2}(\mathbf P^k\mathbf Q)\mathbf M_k^{Q}$, $\mathbf K$ likewise & $50.00_{\pm 0.00}$\\
Random fixed field & $a_{uv}\sim\mathcal N(0,0.8^{2})$, $a_{vu}=-a_{uv}$ & $49.88_{\pm 0.17}$\\
Dynamic local gradient & $a_{uv}=b(\mathbf h_v)-b(\mathbf h_u)$ & $50.00_{\pm 0.00}$\\
Dynamic nonlocal Flat & $a_i=\tilde a_i-\tfrac{1}{n}\sum_{j}\tilde a_j$ & $94.83_{\pm 11.50}$\\
\emph{Sparse} AW, phases zeroed & fitted AW with $a_{uv}\equiv 0$ & $50.00_{\pm 0.00}$\\
\emph{Sparse} AW & $a_{uv}=a_{\max}\tanh(b_{uv}/a_{\max})$ & $\mathbf{98.67}_{\pm 2.92}$\\\bottomrule

\end{tabular}
\end{adjustbox}
\end{table}

\begin{wraptable}{r}{0.56\textwidth}
 \centering
 \vspace{-6pt}
  \captionof{table}{\textbf{Constructed route-interference task.} Test
  accuracy (\%), mean$_{\pm\mathrm{SD}}$ over five runs. Endpoint
  reads a marked node after transport; Attention uses one full block.
  Truncated walks use $L=8$, $z=0.8$; \emph{Exact} uses the complete
  resolvent. Protocol in Appendix~\ref{app:constructed-prediction}.}
  \label{tab:route-interference}
  \vspace{2pt}
  \footnotesize
  \setlength{\tabcolsep}{3pt}
  \begin{tabular}{@{}lll@{}}
   \toprule
   Method & Endpoint & Attention\\
   \midrule
   Mixing (Zero) & $50.00_{\pm 0.00}$ & $97.92_{\pm 1.77}$\\
   WIRE & $50.00_{\pm 0.00}$ & $82.43_{\pm 0.75}$\\
   WIRE + mixing & $50.00_{\pm 0.00}$ & $93.89_{\pm 6.97}$\\
   Dynamic local gradient & $49.99_{\pm 0.02}$ & $\mathbf{100.00}_{\pm 0.00}$\\

   \emph{Sparse} AW, phases zeroed & $50.00_{\pm 0.00}$ & $50.00_{\pm 0.00}$\\
   \emph{Exact} AW, phases zeroed & $50.00_{\pm 0.00}$ & $50.00_{\pm 0.00}$\\

   \emph{Sparse} AW & $\mathbf{98.85}_{\pm 2.58}$ & $\mathbf{99.94}_{\pm 0.13}$\\
\emph{Exact} AW & $\mathbf{100.00}_{\pm 0.00}$ & $\mathbf{100.00}_{\pm 0.00}$\\
   \bottomrule
  \end{tabular}
  \vspace{-4pt}
\end{wraptable}

\paragraph{From fixed-graph intervention to a learned task.}
On a fixed diamond graph with two routes between the same endpoints
(Figure~\ref{fig:two-route-response},
Appendix~\ref{app:two-route-intervention}), varying only the route phase
gap $\Delta$ gives response power $\cos^2(\Delta/2)$, from reinforcement
to cancellation; an endpoint-factorized rotary factor cannot change
its magnitude. The learned task uses two length-eight routes joining
marked endpoints and distinguishes two internal motifs on one route
from one on each. Under endpoint readout, phase-free,
WIRE and local-gradient controls remain at chance, while \emph{Sparse}
AW reaches $98.85\%$ and \emph{Exact} AW reaches $100\%$
(Table~\ref{tab:route-interference}). A full attention block also lets
mixing and gradient controls solve this task, so it isolates the
route-sensitive transport capability.

\paragraph{Feature dependence in a complete attention network.}
A second task pairs different feature arrangements on the same anonymous
cycle and tests graph sizes absent from training
(Table~\ref{tab:inductive-fields}). The variants share a one-block
attention backbone and change the Q/K transformation as listed in the
table, with pointwise adapters matching parameter budgets. AW learns
its field from local endpoint features (Eq.~\eqref{eq:core-edge}).
The walk depth is $L=2$; the learned static polynomial is degree two,
and pure WIRE performs no walk propagation. In matched complete attention
networks, phase-free, topology-only and pointwise-potential fields
provably collide on this construction
(Proposition~\ref{prop:inductive-collision}), and WIRE, WIRE~+~mixing
and a random fixed field sit at chance empirically. \emph{Sparse} AW retains
$98.67\pm2.92\%$, and zeroing its fitted phases returns every run to
$50\%$. The two tasks expose complementary capabilities, and AW covers
both with one local edge-conditioned parameterization. A dynamic nonlocal Flat field also
solves the second task, with high per-run variation. Protocols, proofs
and a per-arm ladder are collected in Appendices~I.1--I.3.

\paragraph{Phases beyond phase-free propagation.}
In the matched study of Table~\ref{tab:multiroute-completed},
WIRE, phase-free mixing $D_{L,z}$, WIRE composed with mixing, and AW-RoPE
share backbone, budget and initialization. The matched controls
attribute most of the benchmark gains to walk transport itself, with
phase learning contributing task-dependent improvements. A matched
field study (Appendix~\ref{app:field-control-results}) similarly finds
the gradient field best on Monochromatic-0; on the real-data field
ablations the unrestricted field attains the highest PascalVOC-SP mean,
with contributions varying by task on Peptides-struct and MolHIV.
Non-flat feature-conditioned transport thus provides additional
modeling capacity whose benchmark payoff varies with the task. On fixed
validation graphs
(Figure~\ref{fig:synthetic-main-visuals}), WIRE separates input color
classes more sharply than AW.

\subsection{Exact--Sparse Fidelity and Scalability}
\label{sec:numerical-theory}
\begin{figure}[!htb]
 \centering
  \includegraphics[width=0.8\linewidth]{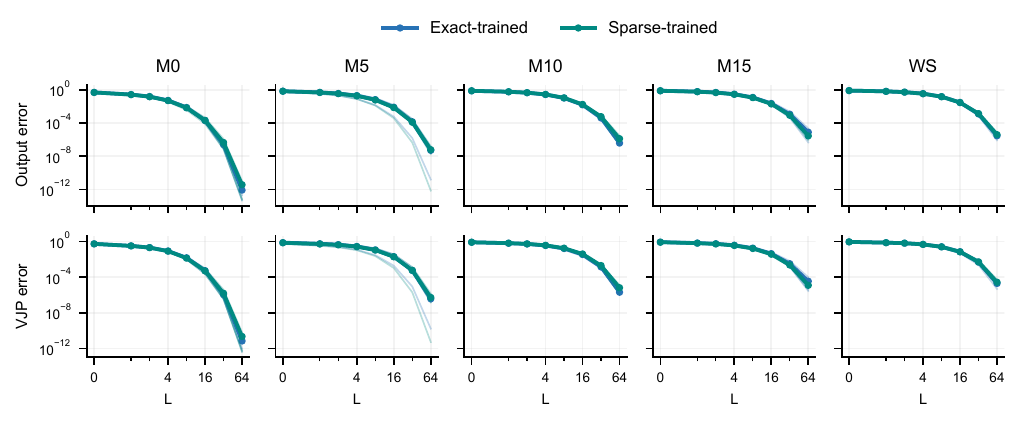}
 \caption{\textbf{Approximation at all 30 trained runs.}
 Columns are the five tasks. Rows are relative output and local VJP error.
 Each curve compares \emph{Sparse} with \emph{Exact} at identical
 parameters and local inputs. Thin lines retain all three seeds, and
 thick lines average seed-level context and probe means. Values below $10^{-16}$
 use a plotting floor. Full settings: Appendix~\ref{app:trained-checkpoint}.}
 \label{fig:trained-checkpoint-main}
\end{figure}

\paragraph{Approximation of learned operators.}
We reuse all 30 validation-best models, comparing \emph{Exact} with
\emph{Sparse} at $L\in\{0,1,2,4,8,16,32,64\}$ on fixed validation
inputs in float64. Relative Euclidean errors cover outputs and three
fixed-readout vector--Jacobian products (VJPs) with respect to local
inputs and AW parameters (Figure~\ref{fig:trained-checkpoint-main}).

Across the ten task--checkpoint-origin groups, at $L=8$ the mean relative output errors range from $0.674\%$ to $15.023\%$, and mean VJP errors from $1.345\%$ to $25.174\%$. At $L=64$, the largest group-mean output and VJP errors are $7.851\times10^{-6}$ and $3.504\times10^{-5}$, respectively. Learned layer decays span $0.468$--$0.867$. These observations support useful trained-operator approximation at sufficient depth on the sampled inputs, while showing that shallow truncation need not already be close to \emph{Exact}. Group means are not worst-case error guarantees.

\paragraph{Exact--Sparse size scaling.}
We time the production positional-attention layer on isolated H200 GPUs
with shared fields, frequencies, decay and inputs, varying graph size,
degree and \emph{Sparse} depth (Appendix~\ref{app:scaling}).
\emph{Exact} is faster on small graphs, but its memory grows
quadratically: it needs 72 GiB at 4,096 nodes and runs out of memory at
8,192 and 16,384 nodes, where $L=16$ \emph{Sparse} needs only 2.7 GiB.
Tables~\ref{tab:matched-efficiency} and
\ref{tab:scaling-favor} supply small-graph comparisons with NoPE and WIRE.

\subsection{Generalization Across Datasets and Backbones}
\label{sec:nonsynthetic-experiments}
\paragraph{Transport without attention.}
A four-layer GIN with width 32 and dropout 0.2 tests transport without
attention. A validation-only 110-trial grid freezes one configuration
per task and method, confirmed over 40 runs. Ordinary AW
attains the lowest mean on four of the five synthetic tasks and every AW
variant improves on NoPE (Table~\ref{tab:gin-synthetic-results}), with
CLUSTER controls in Appendix~\ref{app:cluster-stability},
Table~\ref{tab:cluster-stability-ablation}.

\paragraph{Tasks and attention kernels.}
Table~\ref{tab:kernel-replay} compares Performer and GIN across datasets,
including LRGB \citep{dwivedi2022lrgb} and OGB \citep{hu2020ogb}, using
benchmark accuracy, F1, average precision (AP), AUROC and mean absolute
error (MAE). Performer references are NoPE and WIRE means from
\citet{reid2026graphrope}. Local runs use official splits, metrics and
benchmark GraphGPS and Graph-RoPE code paths, following the Graph-RoPE
training configurations.

\begin{table*}[!htbp]
 \centering
 \captionsetup{font=small}
  \caption{\textbf{AW-RoPE across Performer and GIN backbones.}
  NoPE and WIRE Performer means are from \citet{reid2026graphrope},
  except COCO-SP (three local ReLU runs) and MalNet-Tiny (local matched
  runs for all three methods).
  AW uses ReLU except PascalVOC-SP and COCO-SP ($\S$), which use FAVOR+.
  Local entries show mean$_{\pm\mathrm{SD}}$ at validation best; bold
  compares values within each backbone.
  PATTERN and CLUSTER use accuracy-SBM for Performer and
  accuracy for GIN.}
  \label{tab:kernel-replay}
  \label{tab:gin-real-results}
  \small
  \renewcommand{\arraystretch}{0.95}
  \setlength{\tabcolsep}{4pt}
 \begin{adjustbox}{max width=\linewidth}
 \begin{tabular}{@{}lllll@{\hspace{7pt}}ll@{}}
 \toprule
 \multirow{2}{*}{Dataset} & \multirow{2}{*}{Metric}
 & \multicolumn{3}{c}{\cellcolor{oursgray}\textbf{Performer backbone}}
 & \multicolumn{2}{c}{\cellcolor{oursgray}\textbf{GIN backbone}}\\
 \cmidrule(lr){3-5}\cmidrule(lr){6-7}
 & & NoPE & WIRE & AW-RoPE
 & NoPE & AW-RoPE\\
 \midrule
 MNIST & Acc. (\%) $\uparrow$ & 97.560 & 98.100 & $\mathbf{98.300}_{\pm 0.140}$ & $75.615_{\pm 1.181}$ & $\mathbf{83.890}_{\pm 0.665}$ \\
CIFAR10 & Acc. (\%) $\uparrow$ & 70.610 & 71.150 & $\mathbf{72.057}_{\pm 0.182}$ & $47.240_{\pm 0.240}$ & $\mathbf{47.520}_{\pm 0.523}$ \\
PATTERN & Acc. (\%) $\uparrow$ & 85.710 & 86.630 & $\mathbf{86.689}_{\pm 0.000} $ & $\mathbf{88.574}_{\pm 1.907}$ & $86.443_{\pm 0.527}$ \\
CLUSTER & Acc.  (\%) $\uparrow$ & 76.900 & 77.530 & $\mathbf{77.845}_{\pm 0.140}$ & $48.549_{\pm 0.359}$ & $\mathbf{64.779}_{\pm 0.060}$ \\
Peptides-func & AP (\%) $\uparrow$ & 64.400 & 64.900 & $\mathbf{65.181}_{\pm 0.240}$ & $47.565_{\pm 0.854}$ & $\mathbf{51.937}_{\pm 0.213}$ \\
Peptides-struct & MAE $\downarrow$ & 0.262 & 0.257 & $\mathbf{0.253}_{\pm 0.003}$ & $0.314_{\pm 0.019}$ & $\mathbf{0.291}_{\pm 0.013}$ \\
PascalVOC-SP$^\S$ & Macro-F1 $\uparrow$ & 0.367 & 0.376 & $\mathbf{0.385}_{\pm 0.004}$ & $0.039_{\pm 0.000}$ & $\mathbf{0.094}_{\pm 0.010}$ \\
MalNet-Tiny & Acc. (\%) $\uparrow$ & 91.87 & 91.97 & $\mathbf{92.73}_{\pm 0.55}$ & $81.150_{\pm 0.212}$ & $\mathbf{82.200}_{\pm 1.414}$ \\
ogbg-molhiv & ROC-AUC $\uparrow$ & 0.776 & 0.785 & $\mathbf{0.792}_{\pm 0.001}$ & $0.760_{\pm 0.002}$ & $\mathbf{0.762}_{\pm 0.023}$ \\
ogbg-molpcba & AP $\uparrow$ & 0.238 & 0.264 & $\mathbf{0.280}_{\pm 0.003}$ & $0.169_{\pm 0.009}$ & $\mathbf{0.193}_{\pm 0.001}$ \\
ogbg-code2 & F1 $\uparrow$ & 0.173 & 0.173 & $\mathbf{0.185}_{\pm 0.002}$ & $0.176_{\pm 0.002}$ & $\mathbf{0.181}_{\pm 0.001}$ \\
COCO-SP$^\S$ & Macro-F1 $\uparrow$ & $0.373_{\pm 0.004}$ & $0.370_{\pm 0.007}$ & $\mathbf{0.380}_{\pm 0.003}$ &  $0.025_{\pm 0.003}$  & $\mathbf{0.077}_{\pm 0.010}$ \\
\bottomrule

 \end{tabular}
 \end{adjustbox}
\end{table*}

\paragraph{Attention results.}
On the Performer panel, AW-RoPE improves on the published WIRE means
on every dataset, most strongly on CIFAR10 and CLUSTER
(Table~\ref{tab:kernel-replay}), and its COCO-SP FAVOR+ average exceeds
the local ReLU baselines. Matched zero-field and gradient controls with
sparse propagation retained
(Appendix~\ref{app:field-control-results}) isolate the phase
contribution on PascalVOC-SP, Peptides-struct and MolHIV.
Figure~\ref{fig:diagnostics} shows the matched mixing study under one
backbone and training setting, with AW lower in training MSE
throughout much of the trajectory.

\paragraph{GIN transfer results.}
The GIN columns of Table~\ref{tab:kernel-replay} contain all 52 runs across
thirteen datasets. AW improves the mean on twelve tasks, most strongly on CLUSTER and
MNIST, while PATTERN favors
NoPE and MolHIV is comparable to it. Together with the synthetic GIN
results, walk transport transfers to message passing without attention.

\subsection{Limitations and Practical Trade-offs}
\label{sec:limitations}
Walk propagation adds computational overhead, so AW trains noticeably
slower than WIRE, and \emph{Exact} scales quadratically in memory with
graph size (Appendix~\ref{app:future-grid}). Phase utility is
task-dependent (Section~\ref{sec:phase-mechanism}), the exact
relative-phase reading applies to dot-product attention, and
shuffled-edge controls remain future work.

\section{Conclusion}
Node-wise rotary phases telescope along every route and cannot
distinguish how the graph connects two endpoints. AW-RoPE instead
learns antisymmetric edge rotations and sums transported features over
walks, letting routes reinforce or cancel. \emph{Exact} evaluates the
full walk action by a linear solve, and \emph{Sparse} truncates it
with provable forward and derivative bounds, matching \emph{Exact} at
all 30 trained checkpoints. Both variants improve on the strongest
baseline across the five synthetic tasks ($15$--$58\%$ nRMSE
reduction) and transfer to real superpixel, peptide and OGB
benchmarks, attaining the best mean on twelve of thirteen GIN tasks
and every Performer task. Future work
includes richer fields and reducing walk-propagation complexity and
runtime.

\section{Acknowledgement}

This work is supported by SCION (Scientific Collaborative Innovation with Agentic Organizational Nexus) from Shanghai Innovation Institute.

\bibliographystyle{unsrtnat}
\bibliography{main}

\clearpage
\appendix
\numberwithin{equation}{section}
\numberwithin{theorem}{section}
\renewcommand{\thefigure}{\thesection.\arabic{figure}}
\renewcommand{\thetable}{\thesection.\arabic{table}}
\counterwithin{figure}{section}
\counterwithin{table}{section}
The appendices present convergence bounds (\ref{app:exact-aw-training}),
operator properties with proofs (\ref{app:method-details}), implementation
(\ref{app:notation}) and numerical checks (\ref{app:numerical-audit}), followed
by the matched synthetic study (\ref{app:exact-aw-experiments}), scaling
(\ref{app:scaling}) and further ablations.

\section{Exact--Sparse AW-RoPE: A Step-by-Step Convergence Proof}
\label{app:exact-aw-training}

We establish forward and parameter-derivative convergence of the
finite walk sum. All comparisons use the same
parameters $\theta$, graph, and input. For one complex rotary channel,
write $\mathbf{B}_\theta=z_\theta \mathbf{T}_\theta$, $\mathbf{F}_L=\sum_{k=0}^L \mathbf{B}_\theta^k$,
and $\mathbf{F}_\infty=(\mathbf{I}-\mathbf{B}_\theta)^{-1}$ once existence is established.
The proof concerns ordinary, unnormalized AW-RoPE.

\subsection{Step 1: Specify the norms and contraction}
For a complex vector, $\|\mathbf{x}\|_\infty=\max_u|\mathbf{x}_u|$; the induced matrix
norm is $\|\mathbf{A}\|_\infty=\max_u\sum_v|\mathbf{A}_{uv}|$.
For several channels we take the maximum across channels as well.
These norms satisfy $\|\mathbf{A}\mathbf{x}\|_\infty\le\|\mathbf{A}\|_\infty\|\mathbf{x}\|_\infty$
and $\|\mathbf{A}\mathbf{C}\|_\infty\le\|\mathbf{A}\|_\infty\|\mathbf{C}\|_\infty$.
The latter follows by expanding each row:
\[
 \sum_v|(\mathbf{A}\mathbf{C})_{uv}|
 \le\sum_v\sum_w|\mathbf{A}_{uw}||\mathbf{C}_{wv}|
 \le\left(\sum_w|\mathbf{A}_{uw}|\right)\|\mathbf{C}\|_\infty .
\]
Since $|\mathrm{e}^{\mathrm{i}\omega a_{uv}}|=1$, nonnegative row-substochastic $\mathbf{P}$ gives
\begin{equation}
 \|\mathbf{T}_\theta\|_\infty
 \le\max_u\sum_v \mathbf{P}_{uv}\le1,\qquad
 \|\mathbf{B}_\theta\|_\infty\le z_\theta.
 \label{eq:exact-aw-contraction}
\end{equation}
Parallel edges are summed before forming a matrix entry; the first
inequality remains valid by the triangle inequality. Isolated rows are
zero. Assume a uniform cap $z_\theta\le r<1$ on the parameter region
of interest. It follows inductively that
$\|\mathbf{B}_\theta^k\|_\infty\le r^k$.

\subsection{Step 2: Show that the finite sums have a limit}
For integers $L'>L$, subtracting the partial sums and using Step 1 gives
\[
 \|\mathbf{F}_{L'}-\mathbf{F}_L\|_\infty
 \le\sum_{k=L+1}^{L'}\|\mathbf{B}_\theta^k\|_\infty
 \le\sum_{k=L+1}^\infty r^k
 =\frac{r^{L+1}}{1-r}.
\]
The right-hand side tends to zero independently of $L'$.
Finite-dimensional matrix spaces are complete, so $(\mathbf{F}_L)$ is Cauchy
and has a limit $\mathbf{F}$. Also,
\begin{equation}
 \|\mathbf{F}_L\|_\infty\le\sum_{k=0}^Lr^k
 =\frac{1-r^{L+1}}{1-r},\qquad
 \|\mathbf{F}\|_\infty\le\frac1{1-r}.
 \label{eq:exact-aw-sum-norm}
\end{equation}

\subsection{Step 3: Identify the limit with the complete linear solve}
This identification does not require assuming the inverse beforehand.
Multiplying the finite polynomial and canceling adjacent powers gives
\begin{align}
 (\mathbf{I}-\mathbf{B}_\theta)\mathbf{F}_L
 &= (\mathbf{I}+\mathbf{B}_\theta+\cdots+\mathbf{B}_\theta^L)
    -(\mathbf{B}_\theta+\mathbf{B}_\theta^2+\cdots+\mathbf{B}_\theta^{L+1})\\
 &= \mathbf{I}-\mathbf{B}_\theta^{L+1}.
\end{align}
The same identity holds for $\mathbf{F}_L(\mathbf{I}-\mathbf{B}_\theta)$.
Matrix multiplication is continuous, and
$\|\mathbf{B}_\theta^{L+1}\|_\infty\le r^{L+1}\to0$.
Taking limits therefore gives $(\mathbf{I}-\mathbf{B}_\theta)\mathbf{F}=\mathbf{F}(\mathbf{I}-\mathbf{B}_\theta)=\mathbf{I}$.
Thus $\mathbf{I}-\mathbf{B}_\theta$ is invertible and
\begin{equation}
 \mathbf{F}=\mathbf{F}_\infty=(\mathbf{I}-\mathbf{B}_\theta)^{-1}
 =\sum_{k=0}^{\infty}\mathbf{B}_\theta^k,\qquad
 (\mathbf{I}-\mathbf{B}_\theta)\mathbf{Y}_\infty=\mathbf{X}.
 \label{eq:exact-aw-pair}
\end{equation}
The \emph{Exact} implementation solves this finite-dimensional system.
The infinitely many walks describe its meaning, not its algorithm.

\subsection{Step 4: Derive the forward truncation error}
Define the remainder $\mathbf{E}_L=\mathbf{F}_\infty-\mathbf{F}_L$. The finite identity above yields
\[
 (\mathbf{I}-\mathbf{B}_\theta)\mathbf{E}_L
 =\mathbf{I}-(\mathbf{I}-\mathbf{B}_\theta)\mathbf{F}_L=\mathbf{B}_\theta^{L+1}.
\]
Multiplying by the inverse and applying Step 2 gives
\begin{equation}
 \mathbf{E}_L=\mathbf{F}_\infty \mathbf{B}_\theta^{L+1},\qquad
 \|\mathbf{E}_L\|_\infty\le
 \underbrace{\frac{r^{L+1}}{1-r}}_{\epsilon_L}.
 \label{eq:exact-aw-forward-remainder}
\end{equation}
Consequently $\|\mathbf{Y}_\infty-\mathbf{Y}_L\|_\infty\le\epsilon_L\|\mathbf{X}\|_\infty$.
This proves convergence of the feature action as $L$ increases.

\subsection{Step 5: Differentiate a matrix power without commuting factors}
Let $\mathrm{D}\mathbf{B}_\theta[\mathbf{v}]$ denote a directional derivative in a parameter
direction $\mathbf{v}$ with $\|\mathbf v\|_2=1$. Matrix products are generally noncommutative:
\[
 \mathrm{D}(\mathbf{B}^2)[\mathbf{v}]=\mathrm{D}\mathbf{B}[\mathbf{v}]\mathbf{B}+\mathbf{B}\mathrm{D}\mathbf{B}[\mathbf{v}],\qquad
 \mathrm{D}(\mathbf{B}^3)[\mathbf{v}]=\mathrm{D}\mathbf{B}[\mathbf{v}]\mathbf{B}^2+\mathbf{B}\mathrm{D}\mathbf{B}[\mathbf{v}]\mathbf{B}+\mathbf{B}^2\mathrm{D}\mathbf{B}[\mathbf{v}].
\]
Inductively, differentiating $\mathbf{B}^{k+1}=\mathbf{B}^k \mathbf{B}$ gives
\begin{equation}
 \mathrm{D}(\mathbf{B}^k)[\mathbf{v}]=\sum_{j=0}^{k-1}\mathbf{B}^j\mathrm{D}\mathbf{B}[\mathbf{v}]\mathbf{B}^{k-1-j}.
 \label{eq:exact-aw-power-derivative}
\end{equation}
The induction step appends $\mathbf{B}$ to every old term and adds $\mathbf{B}^k\mathrm{D}\mathbf{B}[\mathbf{v}]$,
which gives exactly the $k+1$ required terms. Suppose
$\|\mathrm{D}\mathbf{B}_\theta[\mathbf{v}]\|_\infty\le G$. There are $k$ summands, each bounded by
$r^jG r^{k-1-j}=G r^{k-1}$, so
\begin{equation}
 \|\mathrm{D}(\mathbf{B}^k)[\mathbf{v}]\|_\infty\le kG r^{k-1}.
 \label{eq:exact-aw-power-bound}
\end{equation}

\subsection{Step 6: Differentiate the complete action and the remainder}
Differentiate $(\mathbf{I}-\mathbf{B}_\theta)\mathbf{F}_\infty=\mathbf{I}$ using the product rule:
\[
 -\mathrm{D}\mathbf{B}_\theta[\mathbf{v}]\mathbf{F}_\infty+(\mathbf{I}-\mathbf{B}_\theta)\mathrm{D}\mathbf{F}_\infty[\mathbf{v}]=0.
\]
Solving for the derivative gives
\[
 \mathrm{D}\mathbf{F}_\infty[\mathbf{v}]=\mathbf{F}_\infty \mathrm{D}\mathbf{B}_\theta[\mathbf{v}]\mathbf{F}_\infty,\qquad
 \|\mathrm{D}\mathbf{F}_\infty[\mathbf{v}]\|_\infty\le\frac{G}{(1-r)^2}.
\]
Now differentiate the finite remainder identity from Step 4:
\[
 \mathrm{D}\mathbf{E}_L[\mathbf{v}]=\mathrm{D}\mathbf{F}_\infty[\mathbf{v}]\mathbf{B}_\theta^{L+1}
              +\mathbf{F}_\infty \mathrm{D}(\mathbf{B}_\theta^{L+1})[\mathbf{v}].
\]
Applying Steps 2 and 5 to the two terms separately yields
\begin{align}
 \|\mathrm{D}\mathbf{E}_L[\mathbf{v}]\|_\infty
 &\le \frac{G r^{L+1}}{(1-r)^2}
       +\frac{(L+1)G r^L}{1-r}\\
 &=G\,\frac{r^L\{r+(L+1)(1-r)\}}{(1-r)^2}\\
 &=\underbrace{G\,\frac{r^L((L+1)-Lr)}{(1-r)^2}}_{\eta_L}.
 \label{eq:exact-aw-gradient-bound}
\end{align}
The derivative follows from the finite remainder identity and the inverse
map on nonsingular matrices. For fixed $0<r<1$, both $r^L$ and $Lr^L$
tend to zero, so $\eta_L\to0$. At $r=0$, take the continuous limit.

\subsection{Step 7: Include parameter-dependent features}
In a network, $\mathbf{X}=\mathbf{X}_\theta$ may itself depend on the parameters.
Writing $\mathbf{Y}_\infty-\mathbf{Y}_L=\mathbf{E}_L\mathbf{X}_\theta$ gives
\[
 \mathrm{D}(\mathbf{Y}_\infty-\mathbf{Y}_L)[\mathbf{v}]=\mathrm{D}\mathbf{E}_L[\mathbf{v}]\mathbf{X}_\theta+\mathbf{E}_L\mathrm{D}\mathbf{X}_\theta[\mathbf{v}].
\]
If $\|\mathbf{X}_\theta\|_\infty\le C_X$ and
$\|\mathrm{D}\mathbf{X}_\theta[\mathbf{v}]\|_\infty\le C_D$, then
\begin{equation}
 \|\mathbf{Y}_\infty-\mathbf{Y}_L\|_\infty\le\epsilon_L C_X,\qquad
 \|\mathrm{D}(\mathbf{Y}_\infty-\mathbf{Y}_L)[\mathbf{v}]\|_\infty
 \le\eta_L C_X+\epsilon_L C_D.
 \label{eq:exact-aw-feature-gradient-bound}
\end{equation}
The second term accounts for preceding learnable layers.

\begin{theorem}[Same-parameter forward and derivative convergence]
\label{thm:exact-aw-training}
On a parameter region with $\|\mathbf{B}_\theta\|_\infty\le r<1$, assume
$\mathbf{B}_\theta$ is differentiable at the points considered and
$\|\mathrm{D}\mathbf{B}_\theta[\mathbf{v}]\|_\infty\le G$ for every unit direction $\mathbf{v}$.
Then the forward and derivative remainders satisfy
Eqs.~\eqref{eq:exact-aw-forward-remainder} and
\eqref{eq:exact-aw-gradient-bound}, uniformly wherever these assumptions
hold. If the input and its derivative are bounded as in Step 7,
Eq.~\eqref{eq:exact-aw-feature-gradient-bound} holds as well.
All these errors tend to zero as $L\to\infty$.
\end{theorem}
The result follows from Steps 1--7.

\subsection{Step 8: Check what the assumptions mean for learned AW-RoPE}
The implemented decay is a sigmoid multiplied by a cap below one;
this supplies a uniform $r$, although a cap close to one makes the
worst-case bound loose. The derivative assumption is additional.
For fixed $\mathbf{P}$, write $\alpha_{uv,\ell}=\omega_\ell a_{uv}$. Then
\[
 \mathrm{D}\mathbf{B}[\mathbf{v}]=(\mathrm{D}z[\mathbf{v}])\mathbf{T}+z\,\mathrm{D}\mathbf{T}[\mathbf{v}],\qquad
 \mathrm{D}\alpha_{uv,\ell}[\mathbf{v}]
   =(\mathrm{D}\omega_\ell[\mathbf{v}])a_{uv}+\omega_\ell \mathrm{D}a_{uv}[\mathbf{v}],
\]
and
\[
 [\mathrm{D}\mathbf{T}[\mathbf{v}]]_{uv}
   =\mathrm{i}\mathbf{P}_{uv}\mathrm{e}^{\mathrm{i}\alpha_{uv,\ell}}\mathrm{D}\alpha_{uv,\ell}[\mathbf{v}].
\]
If $\mathbf{P}$ is learnable, add $(\mathrm{D}\mathbf{P}_{uv}[\mathbf{v}])\mathrm{e}^{\mathrm{i}\alpha_{uv,\ell}}$.
These formulas identify the field, frequency, decay, and optional
transition-weight derivatives that $G$ must bound.
Bounding $a_{uv}$ alone does not bound arbitrary scorer or frequency
derivatives. Smooth scorers satisfy the differentiability requirement;
piecewise-smooth networks require restricting the statement to
differentiability points.

\subsection{Step 9: A conditional downstream loss-gradient bound}
For a scalar objective $\ell_\theta(\mathbf{Y})$, use the feature norm above and
its dual (the sum of two-dimensional block norms when complex channels
are viewed as real pairs). Assume the downstream derivative
$g_\theta(\mathbf{Y})=\mathrm{D}_Y\ell_\theta(\mathbf{Y})$ has dual norm at most $A$ and is
$C_L$-Lipschitz in $\mathbf{Y}$. Assume the explicit parameter derivative
$h_\theta(\mathbf{Y})[\mathbf{v}]=\partial_\theta\ell_\theta(\mathbf{Y})[\mathbf{v}]$ is
$C_\theta$-Lipschitz, and $\|\mathrm{D}\mathbf{Y}_L[\mathbf{v}]\|_\infty\le C_Y$.
By the chain rule, the difference of the two total derivatives is
\begin{align}
 &\mathrm{D}[\ell_\theta(\mathbf{Y}_\infty)-\ell_\theta(\mathbf{Y}_L)][\mathbf{v}]\\
 &=h_\theta(\mathbf{Y}_\infty)[\mathbf{v}]-h_\theta(\mathbf{Y}_L)[\mathbf{v}]
   +g_\theta(\mathbf{Y}_\infty)\mathrm{D}\mathbf{Y}_\infty[\mathbf{v}]
   -g_\theta(\mathbf{Y}_L)\mathrm{D}\mathbf{Y}_L[\mathbf{v}]\\
 &=h_\theta(\mathbf{Y}_\infty)[\mathbf{v}]-h_\theta(\mathbf{Y}_L)[\mathbf{v}]
   +g_\theta(\mathbf{Y}_\infty)(\mathrm{D}\mathbf{Y}_\infty[\mathbf{v}]-\mathrm{D}\mathbf{Y}_L[\mathbf{v}])\\
 &\quad +(g_\theta(\mathbf{Y}_\infty)-g_\theta(\mathbf{Y}_L))\mathrm{D}\mathbf{Y}_L[\mathbf{v}].
\end{align}
The first term is bounded by $C_\theta\epsilon_L C_X$, because
$h_\theta$ is $C_\theta$-Lipschitz in $\mathbf{Y}$ and
Eq.~\eqref{eq:exact-aw-feature-gradient-bound} gives
$\|\mathbf{Y}_\infty-\mathbf{Y}_L\|_\infty\le\epsilon_L C_X$. The second is
bounded by $A(\eta_L C_X+\epsilon_L C_D)$, by the dual norm $A$ of
$g_\theta$ and the derivative bound of
Eq.~\eqref{eq:exact-aw-feature-gradient-bound}. The third is bounded by
$C_L C_Y\epsilon_L C_X$, because $g_\theta$ is $C_L$-Lipschitz, the same
forward bound applies to $\|g_\theta(\mathbf{Y}_\infty)-g_\theta(\mathbf{Y}_L)\|$,
and $\|\mathrm{D}\mathbf{Y}_L[\mathbf{v}]\|_\infty\le C_Y$. Therefore
\begin{equation}
 |\mathrm{D}[\ell_\theta(\mathbf{Y}_\infty)-\ell_\theta(\mathbf{Y}_L)][\mathbf{v}]|
 \le A(\eta_L C_X+\epsilon_L C_D)
      +(C_L C_Y+C_\theta)\epsilon_L C_X .
 \label{eq:exact-aw-loss-gradient}
\end{equation}
For a multilayer model, these constants must hold for the composed
downstream network, including attention normalization. The transport
theorem alone does not supply them.

\paragraph{Scope.}
The bounds compare the complete and truncated operators at common
parameters and inputs. They apply to zero, gradient and non-flat fields;
independent optimization trajectories and test performance are evaluated
experimentally.

\section{Operator Properties and Proofs}
\label{app:method-details}
\label{app:proofs}

\subsection{Problem Setup and Construction}
\label{sec:method-overview}

Let $\mathcal{G}=(V,E)$ be a graph with $N=|V|$ nodes and $M=|E|$
\emph{directed} edges. An undirected edge is stored in both orientations.
The cycle-flatness statements below use an
undirected support with symmetric nonnegative weights $w_{uv}=w_{vu}$.
Connectivity means connectivity through positive-weight edges. The sparse
walk identities also apply to directed nonnegative row-substochastic
transitions, without implying that their connection Laplacian is Hermitian.
The node features and paired-channel attention interface follow
Section~\ref{sec:methods}; the results below use the same operators with
decay and frequency displayed explicitly. Each statement is followed by its proof.

\subsection{Antisymmetric Rotary Edge Connections}
\label{sec:edge-field}

The connection interface accepts any permutation-equivariant antisymmetric scalar field.
Learned AW-RoPE constructs one from the current hidden state;
AW-static and AW-flat instead use the explicitly specified fixed fields.
For every
directed edge $u\to v$, a shared scalar network $s_\theta$ first computes an
ordered score, which we antisymmetrize explicitly:
\begin{align}
    c_{uv} &=s_\theta([\mathbf{h}_u,\mathbf{h}_v]),                                      \label{eq:ordered-score}\\
    \bar{a}_{uv} &=c_{uv}-c_{vu},                                        \label{eq:score-difference}\\
    a_{uv} &=a_{\max}\tanh(\bar{a}_{uv}/a_{\max}).                       \label{eq:antisymmetric-field}
\end{align}
Here $c_{uv}$ is the ordered score and $\bar a_{uv}$ is the score difference
denoted by $b_{uv}$ in Eq.~\eqref{eq:core-edge}. The odd $\tanh$ bounds the displacement.
Other supported \emph{unclipped} fields are a potential difference
$a_{uv}=g_\theta(\mathbf{h}_v)-g_\theta(\mathbf{h}_u)$ and, when coordinates are available, a
geometric projection $a_{uv}=\mathbf{w}^\top(\mathbf{x}_v-\mathbf{x}_u)$. Applying the nonlinear bound
to either still preserves antisymmetry but need not preserve the exact
gradient property; gradient-field theory and the corresponding ablation
therefore disable clipping. Learned variants use
Eq.~\eqref{eq:antisymmetric-field} unless stated otherwise.

\begin{lemma}[Antisymmetry and relabeling]
\label{lem:field-equivariance}
Equation~\eqref{eq:antisymmetric-field} satisfies $a_{vu}=-a_{uv}$. Because
the same scorer is used on every edge, relabeling the nodes only relabels the
edge field.
\end{lemma}

\label{app:proof-field}

\begin{proof}[Proof of Lemma~\ref{lem:field-equivariance}]
Swapping the edge orientation gives
\begin{align}
    \bar{a}_{vu}
    &=s_\theta([\mathbf{h}_v,\mathbf{h}_u])-s_\theta([\mathbf{h}_u,\mathbf{h}_v])\\
    &=-\bar{a}_{uv}.
\end{align}
Since $\tanh$ is odd,
\begin{align}
    a_{vu}
    &=a_{\max}\tanh(\bar{a}_{vu}/a_{\max})\\
    &=a_{\max}\tanh(-\bar{a}_{uv}/a_{\max})\\
    &=-a_{uv}.
\end{align}
For a permutation $\pi$, the ordered pair associated with the relabeled edge
$\pi(u)\to\pi(v)$ contains exactly the original features $[\mathbf{h}_u,\mathbf{h}_v]$.
Because the scorer parameters are shared, its output is unchanged and
$a'_{\pi(u)\pi(v)}=a_{uv}$.
\end{proof}

Let $w_{uv}\geq0$ be an optional edge weight. For a positive weighted degree, define
the row-normalized transition
\begin{equation}
    \mathbf{P}_{uv}=\frac{w_{uv}}{\sum_{x:(u,x)\in E}w_{ux}}.
    \label{eq:transition}
\end{equation}
A zero-weight row is set to zero. Adjacent real channels are viewed as one
complex channel. Channel $\ell$ uses
\begin{equation}
    \omega_\ell=b^{-2\ell/d},
    \qquad \ell=0,\ldots,d/2-1,
    \qquad b=10{,}000,
    \label{eq:rope-frequency}
\end{equation}
initialized as in RoPE \citep{su2021roformer}. Learned AW-RoPE may
optimize these frequencies after initialization.

\begin{proposition}[Node-wise factorization obstruction]
\label{prop:nodewise-obstruction}
If a scalar rotary transport can be implemented by one phase per node,
$\chi_\ell(u,v)=\overline{g_\ell(u)}g_\ell(v)$ with
$|g_\ell(u)|=1$, then its product around every directed cycle is one.
Consequently, a transport with nontrivial cycle circulation cannot be represented
exactly by standard node-wise Q/K rotations.
\end{proposition}

\label{app:proof-nodewise-obstruction}

\begin{proof}[Proof of Proposition~\ref{prop:nodewise-obstruction}]
Let a directed cycle be
$C=(v_0,v_1,\ldots,v_m)$ with $v_m=v_0$. Assume the pair transport has a
scalar node factorization
$\chi_\ell(u,v)=\overline{g_\ell(u)}g_\ell(v)$ and
$|g_\ell(u)|=1$. Multiplying its factors in cycle order gives
\begin{align}
  \prod_{j=0}^{m-1}\chi_\ell(v_j,v_{j+1})
  &=\prod_{j=0}^{m-1}
    \overline{g_\ell(v_j)}g_\ell(v_{j+1}) \\
  &=\overline{g_\ell(v_0)}g_\ell(v_1)
    \overline{g_\ell(v_1)}g_\ell(v_2)\cdots
    \overline{g_\ell(v_{m-1})}g_\ell(v_0) \\
  &=\prod_{j=0}^{m-1}|g_\ell(v_j)|^2 \\
  &=1.
\end{align}
Every intermediate node phase cancels with its conjugate, including the
initial/final node because $v_m=v_0$. Thus non-unit cycle product contradicts
the assumed node factorization. This is why exact non-flat U(1) transport is
 necessarily relational. AW-RoPE remains node-level in its stored
 features by evaluating the edge operator action; it does not claim that the
 underlying pair kernel factorizes.
\end{proof}

\subsection{AW-RoPE: Sparse Analytic Walk Transport}
\label{sec:analytic-walk}

For one frequency $\omega$, define
\begin{equation}
    [\mathbf{T}_\omega]_{uv}=\mathbf{P}_{uv}\mathrm{e}^{\mathrm{i}\omega a_{uv}}.
    \label{eq:rotary-transition}
\end{equation}
Thus, for a complex feature $\mathbf{X}$,
\begin{equation}
    (\mathbf{T}_\omega \mathbf{X})_u
    =\sum_{v:(u,v)\in E}\mathbf{P}_{uv}\mathrm{e}^{\mathrm{i}\omega a_{uv}}\mathbf{X}_v.
    \label{eq:rotary-transition-action}
\end{equation}
The phase is a rotation, while $\mathbf{P}_{uv}$ controls the magnitude.

\begin{proposition}[Convergence of the analytic walk]
\label{prop:neumann}
For every $\omega$, $\lVert \mathbf{T}_\omega\rVert_\infty\leq1$. Hence for
$0\leq z<1$,
\begin{equation}
    \mathbf{F}_{z,\omega}
    =(\mathbf{I}-z\mathbf{T}_\omega)^{-1}
    =\sum_{k=0}^{\infty}z^k\mathbf{T}_\omega^k
    \label{eq:resolvent}
\end{equation}
exists and converges in the induced infinity norm.
\end{proposition}

\label{app:proof-neumann}

Proposition~\ref{prop:neumann} is the fixed-parameter case of
Steps 1--3 in Appendix~\ref{app:exact-aw-training}, with
$\mathbf{B}=z\mathbf{T}_\omega$. Those steps establish contraction, convergence and
invertibility before identifying the limit; no second proof is needed here.

We parameterize $z$ as a sigmoid times $1-\varepsilon$, with fixed
$0<\varepsilon<1$, to keep it in the convergence interval during training.

\subsection{Matrix Powers as Walk Sums}
\label{sec:walk-expansion}

For a length-$k$ walk
$p=(v_0=u,v_1,\ldots,v_k=v)$, define
\begin{equation}
    \pi(p)=\prod_{j=0}^{k-1}\mathbf{P}_{v_jv_{j+1}},
    \qquad
    A(p)=\sum_{j=0}^{k-1}a_{v_jv_{j+1}}.
    \label{eq:walk-weight-displacement}
\end{equation}

\begin{lemma}[Walk expansion]
\label{lem:walk-expansion}
For every integer $k\geq1$,
\begin{equation}
    [\mathbf{T}_\omega^k]_{uv}
    =\sum_{p:u\leadsto v,\,|p|=k}
      \pi(p)\mathrm{e}^{\mathrm{i}\omega A(p)}.
    \label{eq:fixed-length-walk-sum}
\end{equation}
Consequently,
\begin{equation}
    [\mathbf{F}_{z,\omega}]_{uv}
    =\sum_{p:u\leadsto v}
      z^{|p|}\pi(p)\mathrm{e}^{\mathrm{i}\omega A(p)}.
    \label{eq:walk-sum}
\end{equation}
\end{lemma}

\label{app:proof-walk-expansion}

\begin{proof}[Proof of Lemma~\ref{lem:walk-expansion}]
For $k=1$, Eq.~\eqref{eq:fixed-length-walk-sum} is the definition of
$\mathbf{T}_\omega$. Assume it holds for length $k$. Matrix multiplication gives
\begin{align}
    [\mathbf{T}_\omega^{k+1}]_{uv}
    &=\sum_x[\mathbf{T}_\omega^k]_{ux}[\mathbf{T}_\omega]_{xv}\\
    &=\sum_x\sum_{p:u\leadsto x,\,|p|=k}
      \pi(p)\mathrm{e}^{\mathrm{i}\omega A(p)}\mathbf{P}_{xv}\mathrm{e}^{\mathrm{i}\omega a_{xv}}\\
    &=\sum_x\sum_{p:u\leadsto x,\,|p|=k}
      \bigl(\pi(p)\mathbf{P}_{xv}\bigr)
      \mathrm{e}^{\mathrm{i}\omega(A(p)+a_{xv})}.
\end{align}
Appending $x\to v$ is a bijection between the terms on the last line and all
length-$(k+1)$ walks from $u$ to $v$. This proves the induction step.
The zero-length walk exists only for $u=v$ and has $\pi(p)=1$, $A(p)=0$.
Absolute convergence follows from $\sum_{k\ge0}z^k[\mathbf P^k]_{uv}\le(1-z)^{-1}$.
Substituting the fixed-length identity into Eq.~\eqref{eq:resolvent} yields
\begin{align}
    [\mathbf{F}_{z,\omega}]_{uv}
    &=\sum_{k=0}^{\infty}z^k[\mathbf{T}_\omega^k]_{uv}\\
    &=\sum_{k=0}^{\infty}
      \sum_{p:u\leadsto v,\,|p|=k}
      z^k\pi(p)\mathrm{e}^{\mathrm{i}\omega A(p)}\\
    &=\sum_{p:u\leadsto v}
      z^{|p|}\pi(p)\mathrm{e}^{\mathrm{i}\omega A(p)}.
\end{align}
\end{proof}

Equation~\eqref{eq:walk-sum} makes the modeling semantics explicit. The
magnitude $z^{|p|}\pi(p)$ discounts long or unlikely walks. The phase adds
signed edge displacements. Different routes are summed as complex numbers,
so the graph itself determines whether they reinforce or cancel.

\subsection{Sparse Truncation, Error, and Stability}
\label{sec:sparse-evaluation}

We approximate the resolvent with
\begin{equation}
    \mathbf{F}^{(L)}_{z,\omega}\mathbf{X}
    =\sum_{k=0}^{L}z^k\mathbf{T}_\omega^k\mathbf{X}.
    \label{eq:truncated-resolvent}
\end{equation}
Set $\mathbf{X}^{(0)}=\mathbf{X}$ and $\mathbf{Y}^{(0)}=\mathbf{X}$. In real paired-channel form, one step is
\begin{align}
    \mathbf{X}^{(k+1)}_u
    &=\sum_{v:(u,v)\in E}
      \mathbf{P}_{uv}\mathbf{R}(\omega a_{uv})\mathbf{X}^{(k)}_v,                 \label{eq:sparse-step}\\
    \mathbf{Y}^{(k+1)}
    &=\mathbf{Y}^{(k)}+z^{k+1}\mathbf{X}^{(k+1)},                         \label{eq:sparse-accumulate}
\end{align}
Here $\mathbf{R}(\alpha)$ is the counterclockwise rotation by angle
$\alpha$ (Section~\ref{sec:core-preliminaries}). By induction,
$\mathbf{X}^{(k)}=\mathbf{T}_\omega^k\mathbf{X}$; inserting this identity into
Eq.~\eqref{eq:sparse-accumulate} gives Eq.~\eqref{eq:truncated-resolvent}.

\begin{proposition}[Truncation and feature bounds]
\label{prop:tail-stability}
For $0\leq z<1$,
\begin{equation}
    \left\|\mathbf{F}_{z,\omega}
    -\mathbf{F}^{(L)}_{z,\omega}\right\|_\infty
    \leq\frac{z^{L+1}}{1-z}.                            \label{eq:tail-bound}
\end{equation}
In addition,
\begin{equation}
    \left\|\mathbf{F}^{(L)}_{z,\omega}\mathbf{X}\right\|_\infty
    \leq\frac{1-z^{L+1}}{1-z}\lVert \mathbf{X}\rVert_\infty.   \label{eq:feature-bound}
\end{equation}
\end{proposition}

\label{app:proof-tail-stability}

For the first bound, apply Step 4 of
Appendix~\ref{app:exact-aw-training} with $r=z$.
For the feature bound, multiply the finite-sum norm bound in
Eq.~\eqref{eq:exact-aw-sum-norm} by $\|\mathbf{X}\|_\infty$.
These give Eqs.~\eqref{eq:tail-bound} and \eqref{eq:feature-bound},
respectively.

The bound is worst-case: it ignores cancellations between walk phases and is
usually loose when cycles interfere. It nevertheless exposes why a large
$z$ should be paired with a larger $L$, and why normalizing the finite sum or
using a zero-initialized residual gate can stabilize deep models.

\subsection{Permutation Equivariance}
\label{sec:permutation-equivariance}

\begin{proposition}[Permutation equivariance]
\label{prop:permutation}
Let $\boldsymbol{\Pi}$ be a node permutation matrix. If inputs and edges are relabeled by
$\boldsymbol{\Pi}$, then
\begin{equation}
    \mathbf{T}'_\omega=\boldsymbol{\Pi} \mathbf{T}_\omega\boldsymbol{\Pi}^\top,
    \qquad
    \mathbf{F}_{z,\omega}^{\prime(L)}\boldsymbol{\Pi} \mathbf{X}
    =\boldsymbol{\Pi}\mathbf{F}_{z,\omega}^{(L)}\mathbf{X}.
    \label{eq:permutation-equivariance}
\end{equation}
\end{proposition}

\label{app:proof-permutation}

\begin{proof}[Proof of Proposition~\ref{prop:permutation}]
Lemma~\ref{lem:field-equivariance} gives
$a'_{\pi(u)\pi(v)}=a_{uv}$. Row normalization commutes with the same
relabeling, so $\mathbf{P}'=\boldsymbol{\Pi} \mathbf{P}\boldsymbol{\Pi}^\top$. Entrywise phase lifting gives
$\mathbf{T}'_\omega=\boldsymbol{\Pi} \mathbf{T}_\omega\boldsymbol{\Pi}^\top$. Repeated multiplication yields
\begin{equation}
    (\mathbf{T}'_\omega)^k
    =(\boldsymbol{\Pi} \mathbf{T}_\omega\boldsymbol{\Pi}^\top)^k
    =\boldsymbol{\Pi} \mathbf{T}_\omega^k\boldsymbol{\Pi}^\top,
\end{equation}
because every adjacent $\boldsymbol{\Pi}^\top\boldsymbol{\Pi}$ cancels. Hence
\begin{align}
    \mathbf{F}_{z,\omega}^{\prime(L)}\boldsymbol{\Pi} \mathbf{X}
    &=\sum_{k=0}^{L}z^k(\mathbf{T}'_\omega)^k\boldsymbol{\Pi} \mathbf{X}\\
    &=\sum_{k=0}^{L}z^k\boldsymbol{\Pi} \mathbf{T}_\omega^k\boldsymbol{\Pi}^\top\boldsymbol{\Pi} \mathbf{X}\\
    &=\boldsymbol{\Pi}\mathbf{F}_{z,\omega}^{(L)}\mathbf{X}.
\end{align}
\end{proof}

\subsection{Gradient Fields Recover the Sequence Phase}
\label{sec:path-consistency}

Suppose there is a scalar potential $s:V\to\mathbb{R}$ such that
\begin{equation}
    a_{uv}=s_v-s_u.                                      \label{eq:gradient-field}
\end{equation}
Define the diagonal unitary matrix
\begin{equation}
    \mathbf{G}_\omega=\operatorname{diag}
    \left(\mathrm{e}^{\mathrm{i}\omega s_1},\ldots,\mathrm{e}^{\mathrm{i}\omega s_N}\right).
    \label{eq:gauge-matrix}
\end{equation}

\begin{theorem}[\emph{Exact} gradient-field factorization]
\label{thm:gradient-factorization}
For Eq.~\eqref{eq:gradient-field},
\begin{equation}
    \mathbf{T}_\omega=\mathbf{G}_\omega^\ast \mathbf{P} \mathbf{G}_\omega.                  \label{eq:gauge-identity}
\end{equation}
For every polynomial, or power-series function $f(x)=\sum_{k\geq0}c_kx^k$
whose matrix series converges absolutely in norm (in particular, a power
series with radius greater than $\|\mathbf{P}\|_\infty$),
\begin{equation}
    f(\mathbf{T}_\omega)=\mathbf{G}_\omega^\ast f(\mathbf{P})\mathbf{G}_\omega.             \label{eq:analytic-gauge-identity}
\end{equation}
Thus every entry satisfies
\begin{equation}
    [f(\mathbf{T}_\omega)]_{uv}
    =\mathrm{e}^{\mathrm{i}\omega(s_v-s_u)}[f(\mathbf{P})]_{uv}.                    \label{eq:exact-relative-phase}
\end{equation}
\end{theorem}

\label{app:proof-gradient-factorization}

\begin{proof}[Proof of Theorem~\ref{thm:gradient-factorization}]
First expand one entry:
\begin{align}
    [\mathbf{G}_\omega^\ast \mathbf{P} \mathbf{G}_\omega]_{uv}
    &=e^{-i\omega s_u}\mathbf{P}_{uv}\mathrm{e}^{\mathrm{i}\omega s_v}\\
    &=\mathbf{P}_{uv}\mathrm{e}^{\mathrm{i}\omega(s_v-s_u)}\\
    &=\mathbf{P}_{uv}\mathrm{e}^{\mathrm{i}\omega a_{uv}}=[\mathbf{T}_\omega]_{uv}.
\end{align}
Thus $\mathbf{T}_\omega=\mathbf{G}_\omega^\ast \mathbf{P}\mathbf{G}_\omega$. Since
$\mathbf{G}_\omega \mathbf{G}_\omega^\ast=\mathbf{I}$,
\begin{align}
    \mathbf{T}_\omega^2
    &=(\mathbf{G}_\omega^\ast \mathbf{P}\mathbf{G}_\omega)(\mathbf{G}_\omega^\ast \mathbf{P}\mathbf{G}_\omega)\\
    &=\mathbf{G}_\omega^\ast \mathbf{P}^2\mathbf{G}_\omega,
\end{align}
and induction gives
$\mathbf{T}_\omega^k=\mathbf{G}_\omega^\ast \mathbf{P}^k\mathbf{G}_\omega$ for every $k\geq0$. Therefore
\begin{align}
    f(\mathbf{T}_\omega)
    &=\sum_{k\geq0}c_k\mathbf{T}_\omega^k\\
    &=\sum_{k\geq0}c_k\mathbf{G}_\omega^\ast \mathbf{P}^k\mathbf{G}_\omega\\
    &=\mathbf{G}_\omega^\ast\left(\sum_{k\geq0}c_k\mathbf{P}^k\right)\mathbf{G}_\omega\\
    &=\mathbf{G}_\omega^\ast f(\mathbf{P})\mathbf{G}_\omega.
\end{align}
Absolute matrix-norm convergence justifies passing the fixed similarity
transform through the infinite sum. Taking entry $(u,v)$ produces
$[f(\mathbf{T}_\omega)]_{uv}=\mathrm{e}^{\mathrm{i}\omega(s_v-s_u)}[f(\mathbf{P})]_{uv}$.
\end{proof}

\begin{corollary}[Paths and trees]
\label{cor:path-tree}
Every antisymmetric edge field on an undirected tree is a gradient.
For the nonnegative-coefficient AW filters with $0\leq z<1$, every nonzero
entry has phase $\mathrm{e}^{\mathrm{i}\omega(s_j-s_i)}$. In particular, on a path supplied
with $a_{i,i+1}=1$ and hence $s_i=i$, every nonzero coefficient of the truncated or
infinite AW filter has phase $\mathrm{e}^{\mathrm{i}\omega(j-i)}$, exactly the relative phase
of sequence RoPE.
\end{corollary}

\label{app:proof-path-tree}

\begin{proof}[Proof of Corollary~\ref{cor:path-tree}]
Choose a root $r$ and set $s_r=0$. For each node $v$, let
$r=v_0,\ldots,v_m=v$ be the unique tree path and define
\begin{equation}
    s_v=\sum_{j=0}^{m-1}a_{v_jv_{j+1}}.
\end{equation}
For an oriented tree edge $u\to v$ away from the root,
$s_v=s_u+a_{uv}$, hence $a_{uv}=s_v-s_u$. Antisymmetry gives the same
identity in the reverse orientation. Theorem~\ref{thm:gradient-factorization}
therefore applies. For
$f(\mathbf{P})=\sum_{k=0}^{L}z^k\mathbf{P}^k$ or $f(\mathbf{P})=(\mathbf{I}-z\mathbf{P})^{-1}$,
$[f(\mathbf{P})]_{ij}$ is real and nonnegative. Whenever it is nonzero, the only
complex phase is $\mathrm{e}^{\mathrm{i}\omega(s_j-s_i)}$. Only for the supplied unit-oriented
path, where $s_i=i$ up to an additive constant, does this specialize to
$\mathrm{e}^{\mathrm{i}\omega(j-i)}$.
\end{proof}

\begin{remark}[Relative phase and feature mixing]
The corollary recovers the \emph{relative rotary character} of every transport
coefficient. AW-RoPE is not identical to a node-wise sequence RoPE layer:
$f(\mathbf{P})$ also mixes feature magnitudes along walks.
\end{remark}

\subsection{Gauge Covariance and Cycle Circulation}
\label{sec:gauge-holonomy}

\begin{proposition}[Gauge covariance]
\label{prop:gauge-covariance}
Let $f$ be a polynomial or an absolutely norm-convergent power series
as in Theorem~\ref{thm:gradient-factorization}.
For arbitrary real node phases $\phi_u$, define
\begin{equation}
    a'_{uv}=a_{uv}+\phi_v-\phi_u,
    \qquad
    \mathbf{D}_\omega=\operatorname{diag}(\mathrm{e}^{\mathrm{i}\omega\phi_u}).
    \label{eq:gauge-transform}
\end{equation}
Then
\begin{equation}
    \mathbf{T}'_\omega=\mathbf{D}_\omega^\ast \mathbf{T}_\omega \mathbf{D}_\omega,
    \qquad
    f(\mathbf{T}'_\omega)=\mathbf{D}_\omega^\ast f(\mathbf{T}_\omega)\mathbf{D}_\omega.     \label{eq:gauge-covariance}
\end{equation}
Therefore $|[f(\mathbf{T}'_\omega)]_{uv}|=|[f(\mathbf{T}_\omega)]_{uv}|$.
\end{proposition}

\label{app:proof-gauge-covariance}

\begin{proof}[Proof of Proposition~\ref{prop:gauge-covariance}]
Entrywise,
\begin{align}
    [\mathbf{T}'_\omega]_{uv}
    &=\mathbf{P}_{uv}\mathrm{e}^{\mathrm{i}\omega(a_{uv}+\phi_v-\phi_u)}\\
    &=e^{-i\omega\phi_u}[\mathbf{T}_\omega]_{uv}\mathrm{e}^{\mathrm{i}\omega\phi_v}\\
    &=[\mathbf{D}_\omega^\ast \mathbf{T}_\omega \mathbf{D}_\omega]_{uv}.
\end{align}
Because $\mathbf{D}_\omega \mathbf{D}_\omega^\ast=\mathbf{I}$, the same cancellation as in
Appendix~\ref{app:proof-gradient-factorization} gives
$(\mathbf{T}'_\omega)^k=\mathbf{D}_\omega^\ast \mathbf{T}_\omega^k\mathbf{D}_\omega$ and hence
$f(\mathbf{T}'_\omega)=\mathbf{D}_\omega^\ast f(\mathbf{T}_\omega)\mathbf{D}_\omega$. Entry $(u,v)$ is
multiplied only by $\mathrm{e}^{\mathrm{i}\omega(\phi_v-\phi_u)}$, whose modulus is one.
\end{proof}

Covariant feature transport uses $\mathbf X'=\mathbf{D}_\omega^*\mathbf X$,
so $f(\mathbf T'_\omega)\mathbf X'=\mathbf{D}_\omega^*f(\mathbf T_\omega)\mathbf X$.
This identity concerns the supplied field transformation; a learned scorer
need not respect that transformation of its input features.

For a directed cycle $C=(v_0,v_1,\ldots,v_m=v_0)$, define its circulation
\begin{equation}
    \Phi_C=\sum_{j=0}^{m-1}a_{v_jv_{j+1}}.              \label{eq:cycle-circulation}
\end{equation}

\begin{proposition}[Flatness criterion]
\label{prop:flatness}
On a connected undirected graph, a real antisymmetric field is a gradient if and only if
$\Phi_C=0$ for every cycle $C$.
\end{proposition}

\label{app:proof-flatness}

\begin{proof}[Proof of Proposition~\ref{prop:flatness}]
If $a_{uv}=s_v-s_u$, the cycle sum telescopes:
\begin{align}
    \Phi_C
    &=(s_{v_1}-s_{v_0})+(s_{v_2}-s_{v_1})+\cdots
      +(s_{v_0}-s_{v_{m-1}})\\
    &=0.
\end{align}
Conversely, fix a root $r$. For any node $v$, choose a path
$p:r\leadsto v$ and define $s_v=A(p)$. If $p$ and $q$ are two such paths,
following $p$ and then the reverse of $q$ forms a closed walk. It decomposes
into cycles with zero circulation, so $A(p)-A(q)=0$ and $s_v$ is
path-independent. Extending a root-to-$u$ path by $u\to v$ gives
$s_v=s_u+a_{uv}$, hence $a_{uv}=s_v-s_u$.
\end{proof}

The primitive rotary transport around a cycle is
\begin{equation}
    \prod_{(u,v)\in C}\mathrm{e}^{\mathrm{i}\omega a_{uv}}
    =\mathrm{e}^{\mathrm{i}\omega\Phi_C}.                                \label{eq:aw-holonomy}
\end{equation}
This quantity is invariant under Eq.~\eqref{eq:gauge-transform}, because the
added node potentials telescope. It is the part of the learned geometry that
cannot be removed by changing node phases.
At a single frequency, nonzero real circulation is observable in the rotary
factor only when $\omega\Phi_C\notin2\pi\mathbb{Z}$; real flatness and
frequency-specific phase flatness must not be conflated.

\subsection{Relation to magnetic operators and sheaf transport}
\label{sec:connection-relation}
The one-step matrix $\mathbf{T}_\theta$ is a learnable analogue of a
unitary connection: its entries carry $\mathrm{U}(1)$ phases
$\mathrm{e}^{\mathrm{i}\omega a_{uv}}$ on a row-stochastic base, and
$\mathbf{F}_\infty=(\mathbf I-\mathbf B)^{-1}$ is the resolvent of the
damped connection. A magnetic Laplacian fixes these phases from a
potential (edge direction times a charge, or spectral coordinates) and
builds a Hermitian operator whose spectrum is then filtered
\citep{zhang2021magnet}. AW-RoPE instead generates each phase from the
two endpoint features through a shared bounded scorer, which preserves
permutation equivariance, permits the zero, gradient and unrestricted
field families of Appendix~\ref{app:field-control-results}, and keeps
every phase a differentiable function of the current hidden state.
Neural sheaf diffusion learns restriction maps between node and edge
feature spaces \citep{bodnar2022sheaf}: its transport mixes learned
matrix-valued features rather than rotating rotary channel pairs, at a
per-edge parameter cost that grows with the feature dimension, whereas
AW-RoPE carries one shared scalar displacement per edge with per-channel
frequencies. Learned restriction maps can in principle express the same
feature-dependent nonlocal transport, so we do not claim uniqueness
among learned transports. Neither family pairs the complete operator
with a truncated evaluation that carries the forward and derivative
bounds of Appendix~\ref{app:exact-aw-training}.

\paragraph{Two-route interference.}
Suppose two paths (or two groups whose members each share a phase) from $u$ to $v$ contribute nonnegative magnitudes
$c_1,c_2$ and accumulated angles $\theta_1,\theta_2$. Their combined term is
\begin{equation}
    A_{uv}=c_1\mathrm{e}^{\mathrm{i}\theta_1}+c_2\mathrm{e}^{\mathrm{i}\theta_2},
    \qquad
    |A_{uv}|^2=c_1^2+c_2^2+2c_1c_2\cos(\theta_1-\theta_2).
    \label{eq:two-path-interference}
\end{equation}
Equal phases maximize the magnitude; a phase gap of $\pi$ makes the two terms
cancel when $c_1=c_2$. For the two paths, the gap is $\omega\Phi_C$ modulo
$2\pi$, where the closed route takes one path forward and the other backward.

\label{app:proof-interference}

The squared magnitude in Eq.~\eqref{eq:two-path-interference} expands as
\begin{align}
    |A_{uv}|^2
    &=(c_1\mathrm{e}^{\mathrm{i}\theta_1}+c_2\mathrm{e}^{\mathrm{i}\theta_2})
      (c_1e^{-i\theta_1}+c_2e^{-i\theta_2})\\
    &=c_1^2+c_2^2+c_1c_2\mathrm{e}^{\mathrm{i}(\theta_1-\theta_2)}
      +c_1c_2e^{-i(\theta_1-\theta_2)}\\
    &=c_1^2+c_2^2+2c_1c_2\cos(\theta_1-\theta_2),
\end{align}
where the last step uses $\mathrm{e}^{\mathrm{i}x}+e^{-ix}=2\cos x$. Traversing the first route
and the reverse of the second produces a closed route with circulation
$\Phi_C=A(p_1)-A(p_2)$, so $\theta_1-\theta_2=\omega\Phi_C$ modulo
$2\pi$. This formulation also covers $\omega=0$ without division by zero.

\subsection{Random-Walk and Laplacian Views of AW-RoPE}
\label{sec:operator-laplacian-view}

The random-walk and Laplacian descriptions of AW-RoPE are two exactly
equivalent views of the same online operator. Define the ordinary
random-walk Laplacian and its phase-lifted counterpart by
\begin{equation}
    L_{\mathrm{rw}}=\mathbf{I}-\mathbf{P},
    \qquad
    \mathbf{L}_{\omega,a}=\mathbf{I}-\mathbf{T}_{\omega,a},
    \qquad
    [\mathbf{T}_{\omega,a}]_{uv}=\mathbf{P}_{uv}\mathrm{e}^{\mathrm{i}\omega a_{uv}}.
    \label{eq:phase-laplacian}
\end{equation}
Substituting $\mathbf{T}_{\omega,a}=\mathbf{I}-\mathbf{L}_{\omega,a}$ into the truncated analytic
filter gives the identity
\begin{align}
    \mathbf{F}_{z,\omega}^{(L)}(\mathbf{T}_{\omega,a})\mathbf{X}
    &=\sum_{k=0}^{L}z^k\mathbf{T}_{\omega,a}^{k}\mathbf{X} \\
    &=\sum_{k=0}^{L}z^k(\mathbf{I}-\mathbf{L}_{\omega,a})^{k}\mathbf{X}.  \label{eq:laplacian-aw-filter}
\end{align}
Hence the implementation may be read either as phase-weighted random-walk
propagation or as a polynomial filter of a connection random-walk Laplacian.
For the zero field, $a_{uv}=0$, Eq.~\eqref{eq:phase-laplacian} reduces to
$\mathbf{L}_{\omega,0}=L_{\mathrm{rw}}$. For a gradient field
$a_{uv}=s_v-s_u$, letting
$\mathbf{G}_\omega=\operatorname{diag}(\mathrm{e}^{\mathrm{i}\omega s_1},\ldots,\mathrm{e}^{\mathrm{i}\omega s_N})$
gives, entry by entry,
\begin{equation}
    \mathbf{T}_{\omega,a}=\mathbf{G}_\omega^*\mathbf{P}\mathbf{G}_\omega,
    \qquad
    \mathbf{L}_{\omega,a}=\mathbf{G}_\omega^*L_{\mathrm{rw}}\mathbf{G}_\omega.
    \label{eq:gauge-laplacian-similarity}
\end{equation}
This recovers the familiar node-factorized case while making clear what
changes when the edge field has nonzero circulation.

\label{app:proof-laplacian-view}

From Eq.~\eqref{eq:phase-laplacian},
$\mathbf{I}-\mathbf{L}_{\omega,a}=\mathbf{T}_{\omega,a}$. Consequently, for each nonnegative integer
$k$,
\begin{equation}
    \mathbf{T}_{\omega,a}^k=(\mathbf{I}-\mathbf{L}_{\omega,a})^k,
\end{equation}
and multiplication by $z^k$, application to $\mathbf{X}$, and summation from zero to
$L$ gives Eq.~\eqref{eq:laplacian-aw-filter}. In the gradient case,
$\mathbf{T}_{\omega,a}=\mathbf{G}_\omega^*\mathbf{P}\mathbf{G}_\omega$, so
\begin{align}
    \mathbf{L}_{\omega,a}
    &=\mathbf{I}-\mathbf{G}_\omega^*\mathbf{P}\mathbf{G}_\omega\\
    &=\mathbf{G}_\omega^*\mathbf{I}\mathbf{G}_\omega-\mathbf{G}_\omega^*\mathbf{P}\mathbf{G}_\omega\\
    &=\mathbf{G}_\omega^*(\mathbf{I}-\mathbf{P})\mathbf{G}_\omega\\
    &=\mathbf{G}_\omega^*L_{\mathrm{rw}}\mathbf{G}_\omega,
\end{align}
where the second line uses $\mathbf{G}_\omega^*\mathbf{G}_\omega=\mathbf{I}$. Thus the walk and
Laplacian forms are not two approximations; they are the same operator
written in transition and Laplacian coordinates.

Laplacian eigenpairs used by LapPE or WIRE can be cached as backbone
features. AW instead applies the learned connection operator to current
features; its field is recomputed once per layer call, and its $L$ sparse
steps are included in online cost.

\subsection{Non-Backtracking and Multiscale Variants}
\label{sec:aw-variants}

Ordinary powers include immediate reversals $u\to v\to u$. We retain the
row-action convention $(\mathbf{T}\mathbf{X})_u=\sum_v\mathbf{T}_{uv}\mathbf{X}_v$: the edge index $u\to v$
specifies a walk leaving $u$, while its feature contribution is gathered
from $v$ into row $u$. Define
\begin{align}
    m_{u\to v}^{(1)}&=\mathbf{P}_{uv}\mathbf{R}(\omega a_{uv})\mathbf{X}_v,\\
    \mathbf{S}_u^{(k)}&=\sum_{v:(u,v)\in E}m_{u\to v}^{(k)}.
\end{align}
The non-backtracking update, in exactly this row convention, is
\begin{equation}
    m_{u\to v}^{(k+1)}
    =\mathbf{P}_{uv}\mathbf{R}(\omega a_{uv})
      \left(\mathbf{S}_v^{(k)}-m_{v\to u}^{(k)}\right).
    \label{eq:nonbacktracking-recurrence}
\end{equation}
The subtracted term removes continuations whose first step immediately
returns from $v$ to $u$; it is zero when no reverse edge exists. The output
is $\mathbf{X}+\sum_{k=1}^Lz^k\mathbf{S}^{(k)}$. This matches the implementation's target
gather, source-row sum, and reverse-edge subtraction; it is not an
incoming-message action of $\mathbf{T}^\top$. This implementation never forms the
$M\times M$ non-backtracking matrix.

A multiscale variant mixes $J$ decay factors $0\le z_j<1$ with learned
softmax weights $\beta_j\ge0$, $\sum_{j=1}^J\beta_j=1$. Write $\mathsf{W}_0\mathbf{X}=\mathbf{X}$;
$\mathsf{W}_k\mathbf{X}=\mathbf{T}^k\mathbf{X}$ for ordinary walks, and $\mathsf{W}_k\mathbf{X}=\mathbf{S}^{(k)}$ for
the non-backtracking recurrence above. Then
\begin{align}
    \mathbf{F}_{\mathrm{MS}}^{(L)}\mathbf{X}
    &=\sum_{j=1}^{J}\beta_j
      \sum_{k=0}^{L}z_j^k\mathsf{W}_k\mathbf{X}\\
    &=\sum_{k=0}^{L}
      \left(\sum_{j=1}^{J}\beta_jz_j^k\right)\mathsf{W}_k\mathbf{X}.
    \label{eq:multiscale-filter}
\end{align}
The second line shows that all scales share the same propagated states
$\mathsf{W}_k\mathbf{X}$; only their scalar coefficients differ. In particular,
the evaluated AW-NB-MS uses non-backtracking states, not ordinary powers of $\mathbf{T}$.

\subsection{Complexity}
\label{sec:complexity}

\begin{proposition}[\emph{Sparse} cost]
\label{prop:complexity}
Given the edge field and transition weights, width $d$ and $L\geq1$ steps
require $\mathcal{O}(L(N+M)d)$ arithmetic and $\mathcal{O}((N+M)d)$ forward
working memory. When $N=O(M)$, the arithmetic simplifies to
$\mathcal{O}(LMd)$. The bound is for the AW transport itself, requires no
spectral decomposition, and excludes the edge scorer and subsequent attention
kernel. It describes streaming forward evaluation, not peak training memory:
ordinary reverse-mode autodiff can retain $L$ steps of intermediate states.
\end{proposition}

\label{app:proof-complexity}

\begin{proof}[Proof of Proposition~\ref{prop:complexity}]
At one step, Eq.~\eqref{eq:sparse-step} reads one endpoint feature, applies a
constant-size rotation to each channel pair, multiplies by $\mathbf{P}_{uv}$, and
accumulates once for every directed edge. Edge work costs $\mathcal{O}(Md)$;
initializing and adding node states costs $\mathcal{O}(Nd)$, including on
edgeless graphs. Repeating it $L\geq1$ times costs
$\mathcal{O}(L(N+M)d)$, or $\mathcal{O}(LMd)$ when $N=O(M)$.
The forward recurrence stores current
and accumulated node states in $\mathcal{O}(Nd)$ memory, together with edge
indices, weights, and cached edge rotations in $\mathcal{O}(Md)$ memory. It
never stores $\mathbf{T}^k$ or a dense $N\times N$ positional matrix. Non-backtracking
adds edge states of the same $\mathcal{O}(Md)$ order. This storage count does
not include learned-field construction, attention, or a reverse-mode tape;
without recomputation/checkpointing, autodiff may retain intermediate node
and edge states from all $L$ steps.
\end{proof}

Full attention still costs $\mathcal{O}(N^2d)$, whereas Performer supplies its
own near-linear attention approximation \citep{choromanski2021performer}.
AW-RoPE adds the same online sparse transport to either. Query and key tensors
are concatenated only to share kernels and then split; this is a constant
factor optimization, not a mathematical coupling. We account for these
per-layer costs separately from one-time graph preprocessing, following the
distinction in Section~\ref{sec:operator-laplacian-view}.

\section{Notation and Implementation Details}
\label{app:notation}
Bold upright symbols denote feature vectors, matrices and linear maps.
Scalar variables remain italic, while named operations and text subscripts
are upright.

The complete resolvent and finite truncation below are the operators
introduced in Section~\ref{sec:core-aw}, with their decay and frequency
dependence displayed explicitly. All operator applications to real feature
matrices are interpreted independently on their paired channels.

\begin{table}[ht]
    \centering
    \caption{Notation used in the derivations.}
    \small
    \setlength{\tabcolsep}{4pt}
    \begin{tabular}{ll}
        \toprule
        Symbol & Meaning \\
        \midrule
        $N,M$ & number of nodes and directed edges \\
        $\mathbf{P}$ & row-normalized nonnegative transition \\
        $a_{uv}$ & antisymmetric directed-edge displacement \\
        $\omega$ & rotary frequency of one complex channel \\
        $\mathbf{T}_\omega$ & $\mathbf{P}$ lifted by $\mathrm{e}^{\mathrm{i}\omega a_{uv}}$ \\
        $z$ & geometric walk-decay factor in $[0,1)$ \\
        $L$ & largest evaluated walk length \\
        $\mathbf{F}_{z,\omega}$ & infinite analytic resolvent \\
        $\mathbf{F}^{(L)}_{z,\omega}$ & degree-$L$ truncation \\
        $A(p),\pi(p)$ & displacement and transition weight of walk $p$ \\
        $\Phi_C$ & signed circulation around cycle $C$ \\
        \bottomrule
    \end{tabular}
\end{table}

\subsection{Sparse Forward Algorithm}

The same graph edge list is used at every step. Query and key states are
concatenated along the feature axis so the propagation kernel is launched
once, but they are not mixed.

\begin{methodalgorithm}[tbp]
\caption{\emph{Sparse} AW-RoPE forward recurrence.}
\label{alg:aw-forward}
\small
\begin{algorithmic}[1]
\Require edge list $E$, row weights $\mathbf{P}$, displacement $a$, input $\mathbf{X}=[\mathbf{Q};\mathbf{K}]$,
frequencies $\omega$, decay $z$, depth $L$
\State $\mathbf{S}\gets \mathbf{X}$; $\mathbf{Y}\gets \mathbf{X}$; $c\gets 1$
\For{$k=1,\ldots,L$}
    \State rotate each gathered endpoint pair:
      $m_{uv}\gets \mathbf{P}_{uv}\mathbf{R}(\omega a_{uv})\mathbf{S}_v$
    \State sum into row $u$: $\mathbf{S}_u\gets\sum_{v:(u,v)\in E}m_{uv}$
    \State $c\gets cz$
    \State $\mathbf{Y}\gets \mathbf{Y}+c\mathbf{S}$
\EndFor
\State split $\mathbf{Y}$ back into $\widetilde{\mathbf{Q}}$ and $\widetilde{\mathbf{K}}$
\end{algorithmic}
\end{methodalgorithm}

\subsection{Optional Finite-Sum Normalization}

When desired, the implementation divides the finite filter by its scalar mass
\begin{equation}
    C_L(z)=\sum_{k=0}^{L}z^k=\frac{1-z^{L+1}}{1-z}.
\end{equation}
The normalized operator is
\begin{equation}
    \widehat{\mathbf{F}}^{(L)}_{z,\omega}
    =\frac{1}{C_L(z)}\mathbf{F}^{(L)}_{z,\omega}.
\end{equation}
By Eq.~\eqref{eq:feature-bound}, its induced infinity norm is at most one.
This normalization changes the amplitude but not the phase in the gradient
case. The Cluster ablation shows that the bound can improve stability without
necessarily improving task accuracy; the main attention model therefore uses
the unnormalized form unless stated otherwise.

\subsection{Directed-Edge View of Non-Backtracking AW}

Index states by directed edges. The phase-lifted non-backtracking matrix is
\begin{equation}
    \mathbf{B}_{(u\to v),(x\to y)}
    =\mathbf 1[x=v]\mathbf 1[y\neq u]
      \mathbf{P}_{xy}\mathrm{e}^{\mathrm{i}\omega a_{xy}}.
\end{equation}
The first indicator joins consecutive edges; the second forbids immediate
reversal. This matrix acts on column states by gathering the next edge into
the current edge row. To relate it explicitly to the node recurrence, set
$e_{u\to v}^{(0)}=\mathbf{X}_v$ and
$m_{u\to v}^{(k)}=\mathbf{P}_{uv}\mathbf{R}(\omega a_{uv})e_{u\to v}^{(k-1)}$.
Then
\begin{align}
e_{u\to v}^{(k)}
 &=\mathbf{S}_v^{(k)}-m_{v\to u}^{(k)}\\
 &=\sum_{y\neq u}\mathbf{P}_{vy}\mathbf{R}(\omega a_{vy})e_{v\to y}^{(k-1)}
  =(\mathbf{B}e^{(k-1)})_{u\to v},
\end{align}
where complex phases and real pair rotations are equivalent representations.
Thus Eq.~\eqref{eq:nonbacktracking-recurrence} obtains the same result without
forming $\mathbf{B}$, using a target gather, source-row sum, and reverse-edge
subtraction. The node output is $\mathbf{X}+\sum_{k=1}^Lz^k\mathbf{S}^{(k)}$.

\subsection{Integration with Attention and GIN}
\label{sec:attention-integration}

For each complex channel $\ell$, the attention version computes
\begin{equation}
    \widetilde{\mathbf{q}}^{(\ell)}
    =\mathbf{F}_{z,\omega_\ell}^{(L)}\mathbf{q}^{(\ell)},
    \qquad
    \widetilde{\mathbf{k}}^{(\ell)}
    =\mathbf{F}_{z,\omega_\ell}^{(L)}\mathbf{k}^{(\ell)}.
    \label{eq:aw-transform}
\end{equation}
The transformed tensors enter the usual scaled dot-product kernel or the
chosen Performer feature map. Dataset encoder, value path, attention kernel,
optimizer, and prediction head can therefore remain identical to the WIRE
baseline.

To test the analytic operator outside attention, we also insert it as a
residual GIN branch \citep{xu2019gin}:
\begin{equation}
    \mathbf{H}_{\mathrm{in}}=\mathbf{H}+\gamma\left(
      \mathbf{F}_{z,\omega}^{(L)}\mathbf{H}-\mathbf{H}\right).
    \label{eq:gin-aw-branch}
\end{equation}
The GIN block then returns $\mathbf H'=\operatorname{GINBlock}(\mathbf H_{\mathrm{in}})$,
as in Figure~\ref{fig:aw-pipeline}. The legacy branch uses a sigmoid gate initialized to $\sigma(-2)$. The ReZero
variant \citep{bachlechner2021rezero} uses an unconstrained $\gamma$ initialized
exactly to zero, so the initial network is functionally identical to its
paired NoPE GIN. This GIN study tests the broader transport operator and
complements the Performer/WIRE positional comparison. Unless
otherwise stated, the attention implementation uses the unnormalized finite
sum, $L=8$, $z=0.8$, a learnable $z$, and injection at every attention layer.

\subsection{Linear-attention contraction and batching}
\label{app:linear-attention-interface}
The AW action is evaluated before the attention feature map. Here node
query, key and value vectors are columns; $\mathbf o_u$ denotes the attention
output at node $u$, whose row is written with an explicit transpose.
For one graph, write the mapped query/key features
as $\varphi_Q$ and $\varphi_K$, allowing for the backend's role-specific
FAVOR+ stabilization. The usual factorized contraction is
\begin{equation}
 \mathbf{o}_u^\top=
 \frac{\varphi_Q(\widetilde{\mathbf{q}}_u)^\top
          \sum_v\varphi_K(\widetilde{\mathbf{k}}_v)\mathbf{v}_v^\top}
      {\varphi_Q(\widetilde{\mathbf{q}}_u)^\top
          \sum_v\varphi_K(\widetilde{\mathbf{k}}_v)} .
 \label{eq:core-linear}
\end{equation}
The noncausal backend uses positive feature-map offsets for numerical
stability (ReLU: $10^{-3}$; FAVOR+: $10^{-4}$), rather than adding a
separate constant to this denominator.

In the evaluated Graph-RoPE backend, mini-batches are padded to the largest
graph. Padded keys and values are zeroed before the feature map. Since
$\varphi_K(0)$ need not vanish, padded positions can still contribute to the
denominator while their zero values contribute nothing to the numerator.
Thus the batched implementation uses the same contraction over all padded
token slots, not a strict masked sum over real nodes alone. This inherited
batching convention is shared by the local Graph-RoPE positional variants;
AW transport itself operates only on the original graph nodes.

\subsection{Fixed-field controls}
\label{app:static-flat}
The archived Static and Flat controls use
the following fixed-field protocol. On the simple undirected support,
let $\mathbf{P}$ be its row-normalized adjacency and form
$\mathbf{r}_u=[\log(1+d_u),(\mathbf{P}^2)_{uu},\ldots,(\mathbf{P}^8)_{uu}]$.
Each channel is standardized within the graph using population SD;
channels with SD at most $10^{-8}$ become zero. With
$\mathbf{J}=\operatorname{diag}(\mathbf{J}_2,\ldots,\mathbf{J}_2)$,
$\mathbf{J}_2=\left(\begin{smallmatrix}0&1\\-1&0\end{smallmatrix}\right)$,
the fixed displacement is
\[
 b_{uv}=\mathbf{r}_u^\top \mathbf{J} \mathbf{r}_v,\qquad
 a^{\rm static}_{uv}=\pi\tanh\!\left(\frac{b_{uv}}{s+10^{-8}}\right).
\]
Here $s$ is the median of $|b_{uv}|$ over one orientation per edge,
falling back to RMS and then 1 when degenerate.
Reversed edges negate the field. The field may be zero on symmetric
graphs; nontrivial circulation is not guaranteed.
For Flat, let $\mathbf{C}$ be the oriented incidence matrix and solve
$\psi=\arg\min_\psi\|\mathbf{C}\psi-a^{\rm static}\|_2^2$.
Use $a^{\rm flat}=\mathbf{C}\psi$ without additional clipping. This unweighted
least-squares projection has zero real circulation around every cycle.
Both controls use sparse AW transport; neither is the newly trained
\emph{Exact} arm.

\section{Numerical Checks of Operator Identities}

\label{app:numerical-audit}
\label{app:fixed-numerical-checks}

The deterministic audit fixes the random generator state. A
12-node connected weighted graph is formed from a ring plus random symmetric
chords. Its transition is row-normalized; the edge field is obtained by
antisymmetrizing a random matrix. All calculations use complex128.

\begin{table}[ht]
    \centering
    \caption{Maximum absolute discrepancies in the independent numerical
    identity audit.}
    \small
    \begin{tabular}{ll}
        \toprule
        Identity & Maximum error \\
        \midrule
        Gradient operator factorization & $1.11\times10^{-16}$ \\
        Gradient resolvent factorization & $4.44\times10^{-16}$ \\
        Gauge covariance, operator & $1.31\times10^{-16}$ \\
        Gauge covariance, resolvent & $4.48\times10^{-16}$ \\
        Gauge-invariant magnitude & $4.44\times10^{-16}$ \\
        Permutation equivariance & $1.49\times10^{-15}$ \\
        Two-route expansion & $2.22\times10^{-16}$ \\
        Path resolvent factorization & $5.98\times10^{-16}$ \\
        \bottomrule
    \end{tabular}
\end{table}

These tests validate the implementation of the stated identities, not
the empirical usefulness of the assumptions.

\subsection{Forward and Derivative Truncation}
\begin{figure}[t]
 \centering
 \includegraphics[width=\linewidth]{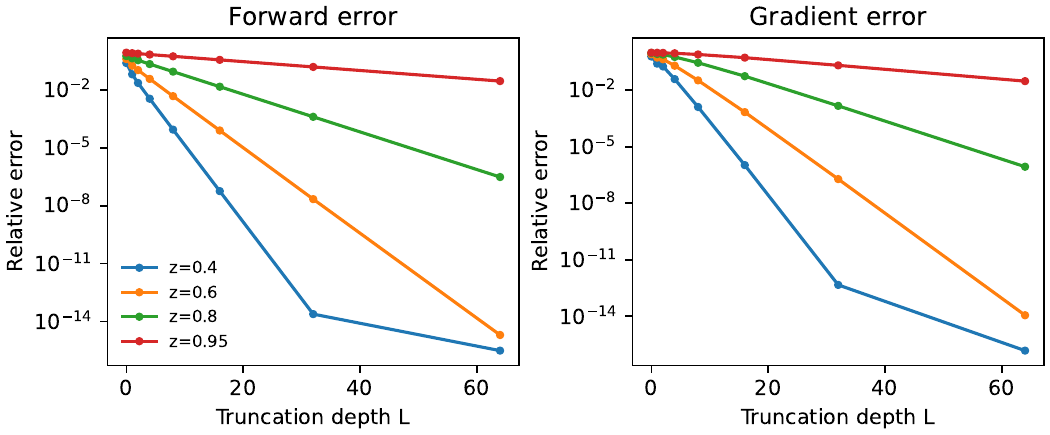}
 \caption{Deterministic \emph{Exact}-AW versus sparse-AW diagnostics on a fixed
 graph, input, and linear readout, in double precision. Left: maximum
 complex-feature error relative to the complete output. Right: relative
 Euclidean error of the readout gradient jointly with respect to inputs,
 edge displacements, frequencies, edge weights, and decay. Each curve
 changes only $L$ at the indicated $z$; 32 configurations including $L=0$
 are audited. These are operator/derivative checks, not trained-model
 comparisons. Near-unit decay requires greater depth; numerical floors
 need not decrease monotonically.}
 \label{fig:exact-aw-forward-gradient}
\end{figure}
The independent bound check at $z=0.65$ also satisfies
Eq.~\eqref{eq:tail-bound}; its full depth sweep is retained with the artifact. Memory and timing
measurements
are consolidated in Appendix~\ref{app:scaling}.

\section{Matched Synthetic Study and Trained-Checkpoint Analysis}
\label{app:exact-aw-experiments}

\paragraph{Matched predictive and efficiency protocol.}
The matched protocol freezes the existing five synthetic tasks' AW learning
rates, initial decay, and truncation depths. Both arms learn the field,
frequencies, and decay with identical full-model initialization and
250-epoch budgets over three matched runs: 15 \emph{Exact}-AW and 15 sparse-AW
trials. Checkpoints are selected on validation and evaluated at the same
test epoch. Ten one-epoch probes precede training. Separate timing repeats
use one trial per GPU, one warm-up and three measured training epochs,
for both arms and all five tasks with three repeats. These measurements
must be distinguished from packed concurrent training throughput.
The separate size sweep is in Appendix~\ref{app:scaling}.

\paragraph{Completed paired results and provenance.}
The 30 formal runs and 30 isolated timing repeats passed the complete
artifact audit. All formal validation/test trajectories contain epochs
0--249; selected and final checkpoints are retained. Each of the 15
dataset--run pairs has identical initialization, parameter count, split
sizes, and layer configuration. The implementation solves a $2N$ real-block system with a cuSOLVER
library preference; small batched kernels may use cuBLAS.
The publication artifact contains the frozen configuration, per-result and
trajectory hashes, all full-precision metrics, and a regenerating script.
Normalized RMSE divides the logged RMSE by 25 for the monochromatic tasks
and by 10 for Watts--Strogatz SPD, consistently with the earlier tables.

\begin{table}[t]
 \centering
 \caption{Every run in the completed \emph{Exact}--\emph{Sparse} study. Test nRMSE
 is paired with the best validation epoch (zero-based); lower is better.
 No run is removed. Three-decimal values are displayed here; the artifact
 retains full precision.}
 \label{tab:exact-sparse-seeds}
 \small
 \begin{tabular}{llllll}
  \toprule
  & & \multicolumn{2}{c}{\emph{Exact} AW-RoPE} & \multicolumn{2}{c}{\emph{Sparse} AW-RoPE} \\
  Dataset & Seed & Test nRMSE & Epoch & Test nRMSE & Epoch \\
  \midrule
  Monochromatic-0 & 0 & 0.040 & 224 & 0.040 & 224 \\
Monochromatic-0 & 1 & 0.040 & 219 & 0.039 & 219 \\
Monochromatic-0 & 2 & 0.051 & 212 & 0.055 & 215 \\
Monochromatic-5 & 0 & 0.035 & 195 & 0.035 & 230 \\
Monochromatic-5 & 1 & 0.034 & 214 & 0.035 & 208 \\
Monochromatic-5 & 2 & 0.039 & 200 & 0.039 & 185 \\
Monochromatic-10 & 0 & 0.027 & 219 & 0.027 & 228 \\
Monochromatic-10 & 1 & 0.027 & 235 & 0.027 & 228 \\
Monochromatic-10 & 2 & 0.027 & 226 & 0.043 & 200 \\
Monochromatic-15 & 0 & 0.023 & 217 & 0.024 & 221 \\
Monochromatic-15 & 1 & 0.030 & 224 & 0.031 & 210 \\
Monochromatic-15 & 2 & 0.026 & 195 & 0.029 & 221 \\
Watts--Strogatz SPD & 0 & 0.016 & 104 & 0.022 & 74 \\
Watts--Strogatz SPD & 1 & 0.017 & 99 & 0.014 & 155 \\
Watts--Strogatz SPD & 2 & 0.014 & 187 & 0.015 & 106 \\
\bottomrule

 \end{tabular}
\end{table}

\paragraph{Measured cost and its scope.}
Table~\ref{tab:exact-sparse-efficiency} averages three independent timing
repeats per arm and task; each repeat averages three measured training
epochs after one warm-up. One trial occupies each allocated H200 during
timing. Epoch time includes training and logging, but excludes validation,
test evaluation, checkpoints, and cold preprocessing. Peak allocated GPU
memory is the maximum over the measured epochs, summarized across repeats.
Both arms contain the same 39,575 trainable parameters and dense-attention
backbone. These are current-implementation small-graph costs, not the packed
formal-run elapsed times. \emph{Exact} is faster on all five tasks, and \emph{Sparse} is
not uniformly more memory-efficient. This does not contradict the sparse
asymptotic bound: batched dense solves can be efficient at small sizes,
while a finite recurrence incurs repeated kernel work. Appendix~\ref{app:scaling} supplies the separate size benchmark and
measured allocation failures. Those layer-level measurements must not be
conflated with the full-model epoch measurements here.

\begin{table}[t]
 \centering
 \caption{Isolated full-model H200 cost for the matched synthetic study.
 Entries are three-decimal mean and sample SD over three timing repeats.
 $L$ is the \emph{Sparse} depth; \emph{Exact} is untruncated. Smaller is better.}
 \label{tab:exact-sparse-efficiency}
 \small
 \setlength{\tabcolsep}{3pt}
 \begin{tabular}{llllll}
  \toprule
  & & \multicolumn{2}{c}{Training epoch (s)} & \multicolumn{2}{c}{Peak allocated (MiB)} \\
  Dataset & $L$ & \emph{Exact} & \emph{Sparse} & \emph{Exact} & \emph{Sparse} \\
  \midrule
  Monochromatic-0 & 8 & $16.659_{\pm 0.253}$ & $22.974_{\pm 0.310}$ & $157.676_{\pm 0.000}$ & $135.950_{\pm 0.000}$ \\
Monochromatic-5 & 16 & $16.572_{\pm 0.075}$ & $34.561_{\pm 0.140}$ & $156.801_{\pm 0.000}$ & $161.184_{\pm 0.000}$ \\
Monochromatic-10 & 16 & $16.615_{\pm 0.190}$ & $34.641_{\pm 0.920}$ & $155.925_{\pm 0.000}$ & $153.277_{\pm 0.000}$ \\
Monochromatic-15 & 8 & $17.273_{\pm 1.044}$ & $22.980_{\pm 0.082}$ & $155.050_{\pm 0.000}$ & $123.479_{\pm 0.000}$ \\
Watts--Strogatz SPD & 8 & $16.932_{\pm 0.229}$ & $23.355_{\pm 0.154}$ & $107.952_{\pm 0.000}$ & $107.347_{\pm 0.000}$ \\
\bottomrule

 \end{tabular}
\end{table}

\subsection{Approximation at trained checkpoints}
\label{app:trained-checkpoint}

We check the operators learned in the matched synthetic study.
We reuse all 30 validation-best checkpoints: five tasks, both training
arms, over three matched runs.  Checkpoint hashes, original commands, validation selection,
and shared initialization are audited before the diagnostic.

For each checkpoint, its original \emph{Exact} or \emph{Sparse} network runs in
evaluation mode on the first four validation examples, in two fixed
batches of two. We capture $H$ and projected $X=[Q;K]$ at all four
insertion layers. A sweep then holds these local inputs, graph, scorer
parameters, frequencies and decay fixed while comparing \emph{Exact} with
$L\in\{0,1,2,4,8,16,32,64\}$. Each layer is checked separately:
changing $L$ does not change the inputs supplied to later-layer checks.
The captured float32 weights and activations are promoted to float64
for this numerical diagnostic.

Output error is the relative Euclidean norm of
$\mathcal F^{(L)}_\theta X-\mathcal F_\theta X$.
For derivatives, three deterministic, unit-norm Gaussian probes $R_j$
define the same local linear readouts $\langle R_j,Y\rangle$ for both
evaluations. We compare their vector--Jacobian products (VJPs) with
respect to independent local $X,H$, the field scorer's parameters,
frequencies, and decay logit. The audit records the joint relative Euclidean
error together with per-group absolute/relative errors.
A zero reference gradient gives an undefined relative error, recorded
as null rather than replaced with zero. These are sampled local
derivatives, not the full Jacobian norm or a downstream task-loss gradient.

The complete depth curves appear in Figure~\ref{fig:trained-checkpoint-main}.

\begin{table}[t]
 \centering
 \caption{Relative output error at trained checkpoints.
 Means and sample SDs across all three run-level context averages,
 with three-decimal mantissas. These are approximation errors, not nRMSE
 or errors between independently trained models.}
 \label{tab:trained-checkpoint-approximation}
 \small
 \resizebox{\linewidth}{!}{%
 \begin{tabular}{llllll}
 \toprule
 Task & Checkpoint origin & $L=8$ & $L=16$ & $L=32$ & $L=64$\\
 \midrule
 M0 & Exact-trained & $(7.244_{\pm 3.230})\times10^{-3}$ & $(1.971_{\pm 1.284})\times10^{-4}$ & $(2.545_{\pm 2.514})\times10^{-7}$ & $(8.153_{\pm 12.055})\times10^{-13}$ \\
M0 & Sparse-trained & $(6.738_{\pm 3.501})\times10^{-3}$ & $(2.095_{\pm 1.904})\times10^{-4}$ & $(4.470_{\pm 6.379})\times10^{-7}$ & $(3.682_{\pm 6.271})\times10^{-12}$ \\
M5 & Exact-trained & $(6.530_{\pm 4.548})\times10^{-2}$ & $(8.044_{\pm 6.941})\times10^{-3}$ & $(1.399_{\pm 1.465})\times10^{-4}$ & $(5.112_{\pm 6.865})\times10^{-8}$ \\
M5 & Sparse-trained & $(5.961_{\pm 4.192})\times10^{-2}$ & $(7.044_{\pm 6.310})\times10^{-3}$ & $(1.247_{\pm 1.407})\times10^{-4}$ & $(6.065_{\pm 8.981})\times10^{-8}$ \\
M10 & Exact-trained & $(1.014_{\pm 0.126})\times10^{-1}$ & $(1.539_{\pm 0.193})\times10^{-2}$ & $(4.049_{\pm 0.313})\times10^{-4}$ & $(3.794_{\pm 1.652})\times10^{-7}$ \\
M10 & Sparse-trained & $(1.075_{\pm 0.081})\times10^{-1}$ & $(1.743_{\pm 0.246})\times10^{-2}$ & $(5.857_{\pm 2.700})\times10^{-4}$ & $(1.281_{\pm 1.553})\times10^{-6}$ \\
M15 & Exact-trained & $(1.214_{\pm 0.209})\times10^{-1}$ & $(2.234_{\pm 0.911})\times10^{-2}$ & $(1.147_{\pm 1.098})\times10^{-3}$ & $(7.851_{\pm 12.238})\times10^{-6}$ \\
M15 & Sparse-trained & $(1.124_{\pm 0.173})\times10^{-1}$ & $(1.970_{\pm 0.418})\times10^{-2}$ & $(8.428_{\pm 2.877})\times10^{-4}$ & $(2.789_{\pm 1.681})\times10^{-6}$ \\
WS & Exact-trained & $(1.502_{\pm 0.146})\times10^{-1}$ & $(2.938_{\pm 0.636})\times10^{-2}$ & $(1.232_{\pm 0.614})\times10^{-3}$ & $(2.838_{\pm 2.956})\times10^{-6}$ \\
WS & Sparse-trained & $(1.460_{\pm 0.195})\times10^{-1}$ & $(2.913_{\pm 0.731})\times10^{-2}$ & $(1.356_{\pm 0.621})\times10^{-3}$ & $(3.952_{\pm 3.052})\times10^{-6}$ \\
\bottomrule

 \end{tabular}}
\end{table}

Across the ten task--checkpoint-origin groups, at $L=8$ the mean relative output errors range from $0.674\%$ to $15.023\%$, and mean VJP errors from $1.345\%$ to $25.174\%$. At $L=64$, the largest group-mean output and VJP errors are $7.851\times10^{-6}$ and $3.504\times10^{-5}$, respectively. Learned layer decays span $0.468$--$0.867$. These observations support useful trained-operator approximation at sufficient depth on the sampled inputs, while showing that shallow truncation need not already be close to \emph{Exact}. Group means are not worst-case error guarantees.

The audit contains 240 fixed layer/batch contexts, 1,920 output comparisons,
and 5,760 probe-VJP comparisons. All checkpoint origins and depths
are retained. \emph{Exact} linear-system residuals and the forward tail bound
are checked, and checkpoint/module/input hashes verify that the sweep
does not change parameters or inputs. The diagnostic uses four validation examples per task and checkpoint. Learned decay is not reset to its
initial value.

\section{Exact--Sparse Size Scaling and Timing Audit}
\label{app:scaling}

The size experiment measures the production positional-attention layer,
not an entire dataset training run. All methods use float32 on an isolated
H200, hidden width 64 and four heads. \emph{Exact} and \emph{Sparse} share learned
antisymmetric fields, learnable frequencies and decay initialized at
$z=0.8$, identical inputs, and identical layer initial states within
each repeat. The current \emph{Exact} implementation solves the same real-block
linear system used by the completed synthetic study. Its implementation
and cuSOLVER library preference are fixed across sizes; PyTorch's internal
kernel dispatch can change with matrix dimensions.

For the small-graph positional-layer tables, batch size is eight and graph sizes
are 64, 128 and 256. NoPE, WIRE, \emph{Exact} and \emph{Sparse} ($L=16$) are measured
under both ReLU and softmax FAVOR+ attention. WIRE spectral coordinates
are prepared before timing. The size sweep instead uses batch size one,
ReLU attention, nodes in
$\{64,128,256,512,1024,2048,4096,8192,16384\}$, degrees 6 and 18, and
\emph{Sparse} depths $L\in\{8,16,32\}$. The graphs are bidirectional rings with
fixed-offset chords, so degree and node count are controlled separately.
These are synthetic cost probes, not a real-data accuracy benchmark.

Each cell has a fresh process and one GPU. We reset the RNG before model
construction and record hashes of parameters, graph and inputs.
Within a repeat we run three warm-up steps and ten measured steps at
fixed parameters. Each measured step computes a complete layer forward,
a scalar squared-output loss, and backward; it does not update an
optimizer. CUDA synchronization surrounds the forward/backward regions,
and peak allocated memory is reset before the forward. Finite outputs
and parameter/input gradients are checked outside the timed region.
The mean and sample SD summarize the three independent repeats.
Packed training throughput is not used as a timing observation.

The declared grid contains 288 cells: 72 for the small-graph layer table
and 216 for the size sweep. A CUDA allocation failure is recorded as OOM;
a predeclared 900-second cell limit is recorded separately as timeout.
Neither is converted into an arbitrary runtime or an accuracy score.
All cells and statuses are retained. \emph{Sparse} transport error relative
to \emph{Exact} at the same parameters is also measured outside the timed
region for sizes up to 1024; this is not predictive error.

\begin{table*}[t]
 \centering
 \caption{\textbf{Matched learnable positional-layer cost on H200.}
 ReLU linear attention, batch size eight, width 64, four heads, $z=0.8$.
 Forward/backward times and peak allocated memory are summarized over
 three independent repeats, each averaging ten warmed steps.
 \emph{Exact} is untruncated; \emph{Sparse} uses $L=16$. All values have three decimals.
 All 72 small-graph cells (including the companion FAVOR+ table) are complete.
 The 256-node \emph{Exact} timing jump is discussed below.}
 \label{tab:matched-efficiency}
 \small
 \setlength{\tabcolsep}{4pt}
 \begin{tabular}{lllllll}
 \toprule
 Nodes/graph & Method & $L$ & Forward ms & Backward ms & Total ms & Peak MiB\\
 \midrule
 64 & NoPE & -- & $1.430_{\pm 0.054}$ & $1.276_{\pm 0.049}$ & $2.706_{\pm 0.101}$ & $99.006_{\pm 0.000}$ \\
64 & WIRE & -- & $1.621_{\pm 0.018}$ & $1.577_{\pm 0.024}$ & $3.198_{\pm 0.038}$ & $99.195_{\pm 0.000}$ \\
64 & Exact AW-RoPE & -- & $3.133_{\pm 0.068}$ & $2.460_{\pm 0.065}$ & $5.593_{\pm 0.127}$ & $218.241_{\pm 0.000}$ \\
64 & Sparse AW-RoPE & 16 & $3.495_{\pm 0.094}$ & $6.096_{\pm 0.214}$ & $9.591_{\pm 0.306}$ & $180.102_{\pm 0.000}$ \\
128 & NoPE & -- & $1.469_{\pm 0.049}$ & $1.311_{\pm 0.033}$ & $2.780_{\pm 0.082}$ & $101.854_{\pm 0.000}$ \\
128 & WIRE & -- & $1.591_{\pm 0.056}$ & $1.566_{\pm 0.031}$ & $3.157_{\pm 0.073}$ & $102.231_{\pm 0.000}$ \\
128 & Exact AW-RoPE & -- & $5.484_{\pm 0.040}$ & $2.823_{\pm 0.031}$ & $8.307_{\pm 0.069}$ & $660.393_{\pm 0.000}$ \\
128 & Sparse AW-RoPE & 16 & $3.721_{\pm 0.195}$ & $6.401_{\pm 0.413}$ & $10.122_{\pm 0.485}$ & $265.532_{\pm 0.000}$ \\
256 & NoPE & -- & $1.409_{\pm 0.043}$ & $1.279_{\pm 0.032}$ & $2.688_{\pm 0.075}$ & $107.551_{\pm 0.000}$ \\
256 & WIRE & -- & $1.643_{\pm 0.037}$ & $1.675_{\pm 0.076}$ & $3.318_{\pm 0.100}$ & $108.302_{\pm 0.000}$ \\
256 & Exact AW-RoPE & -- & $433.971_{\pm 0.386}$ & $5.592_{\pm 0.064}$ & $439.563_{\pm 0.404}$ & $2409.198_{\pm 0.000}$ \\
256 & Sparse AW-RoPE & 16 & $3.849_{\pm 0.552}$ & $6.400_{\pm 0.230}$ & $10.249_{\pm 0.780}$ & $430.892_{\pm 0.000}$ \\
\bottomrule

 \end{tabular}
\end{table*}

\begin{table}[t]
 \centering
 \caption{Companion softmax FAVOR+ positional-layer timings. Protocol,
 hardware, initialization and reporting match the ReLU measurements in
 Table~\ref{tab:matched-efficiency}.
 Entries are mean$_{\pm\mathrm{SD}}$, three decimals; \emph{Exact} is untruncated.}
 \label{tab:scaling-favor}
 \small
 \resizebox{\linewidth}{!}{%
 \begin{tabular}{lllllll}
 \toprule
 Nodes & Method & $L$ & Forward ms & Backward ms & Total ms & Peak MiB\\
 \midrule
 64 & NoPE & -- & $1.532_{\pm 0.035}$ & $1.511_{\pm 0.037}$ & $3.043_{\pm 0.068}$ & $99.702_{\pm 0.000}$ \\
64 & WIRE & -- & $2.009_{\pm 0.513}$ & $1.753_{\pm 0.023}$ & $3.762_{\pm 0.492}$ & $99.891_{\pm 0.000}$ \\
64 & Exact AW-RoPE & -- & $3.717_{\pm 0.825}$ & $2.703_{\pm 0.018}$ & $6.420_{\pm 0.834}$ & $218.241_{\pm 0.000}$ \\
64 & Sparse AW-RoPE & 16 & $3.799_{\pm 0.213}$ & $6.618_{\pm 0.467}$ & $10.417_{\pm 0.678}$ & $180.797_{\pm 0.000}$ \\
128 & NoPE & -- & $1.565_{\pm 0.023}$ & $1.543_{\pm 0.004}$ & $3.108_{\pm 0.019}$ & $103.246_{\pm 0.000}$ \\
128 & WIRE & -- & $1.708_{\pm 0.008}$ & $1.774_{\pm 0.020}$ & $3.482_{\pm 0.015}$ & $103.622_{\pm 0.000}$ \\
128 & Exact AW-RoPE & -- & $5.598_{\pm 0.039}$ & $3.044_{\pm 0.009}$ & $8.642_{\pm 0.032}$ & $660.393_{\pm 0.000}$ \\
128 & Sparse AW-RoPE & 16 & $3.694_{\pm 0.127}$ & $6.524_{\pm 0.237}$ & $10.219_{\pm 0.364}$ & $266.923_{\pm 0.000}$ \\
256 & NoPE & -- & $1.635_{\pm 0.127}$ & $1.534_{\pm 0.040}$ & $3.168_{\pm 0.166}$ & $110.333_{\pm 0.000}$ \\
256 & WIRE & -- & $1.725_{\pm 0.024}$ & $1.871_{\pm 0.046}$ & $3.596_{\pm 0.068}$ & $111.084_{\pm 0.000}$ \\
256 & Exact AW-RoPE & -- & $435.142_{\pm 0.440}$ & $5.969_{\pm 0.103}$ & $441.111_{\pm 0.532}$ & $2409.198_{\pm 0.000}$ \\
256 & Sparse AW-RoPE & 16 & $3.854_{\pm 0.042}$ & $6.892_{\pm 0.206}$ & $10.746_{\pm 0.216}$ & $434.924_{\pm 0.000}$ \\
\bottomrule

 \end{tabular}}
\end{table}

\paragraph{The 256-node timing jump.}
\emph{Exact} forward rises from 5.484 to 433.971 ms between 128 and 256 nodes,
while the 256-node backward takes 5.592 ms. FAVOR+ shows the same jump
(435.142 ms). The recorded PyTorch build
(\texttt{2.7.0a0+ecf3bae40a.nv25.02}, CUDA 12.8) solves real blocks of
twice the node count. At matrix dimension 512, its LU dispatch switches
from batched cuBLAS to looped cuSOLVER; ordinary first-order backward
reuses the LU factors.\footnote{Version-matched PyTorch
\href{https://github.com/pytorch/pytorch/blob/ecf3bae40a/aten/src/ATen/native/cuda/linalg/BatchLinearAlgebra.cpp\#L1489-L1496}{LU dispatch}
and \href{https://github.com/pytorch/pytorch/blob/ecf3bae40a/torch/csrc/autograd/FunctionsManual.cpp\#L5551-L5564}{solve backward}.}
This source-level switch is consistent with the observed jump; the timings
do not isolate its runtime contribution. Peak allocation grows by a factor
of 3.65, close to the quadratic growth of dense matrices. The observed
crossing therefore reflects this software stack and batching configuration.

\paragraph{Measured capacity and approximation.}
All 288 cells terminate: 276 successful measurements and 12 actual
CUDA OOMs, with no timeouts or software failures. The OOMs are precisely
\emph{Exact} at 8,192 and 16,384 nodes, both degrees, all three repeats.
\emph{Sparse} completes every size and depth. At degree six and $L=16$,
16,384-node forward/backward takes 27.287 ms with 2.731 GiB peak
allocated memory. \emph{Exact} has lower measured latency at 64 nodes; this
configuration-specific observation is separate from the isolated full-model
training times in Table~\ref{tab:exact-sparse-efficiency}.

At 1,024 nodes and degree six, the same-input relative feature $\ell_2$
errors for $L=8,16,32$ average $0.079086$, $0.011728$, and $0.000286$,
respectively. These are operator checks at the initialized parameters,
not errors after independently training each version. They complement
the worst-case bound with the measured depth--accuracy trade-off.

\section{Search Spaces and Selected Hyperparameters}
\label{app:hyperparameters}

Table~\ref{tab:synthetic-hyperparameters} lists the shared hyperparameters of
the matched synthetic study summarized in
Section~\ref{sec:synthetic-experiments}, following the Graph-RoPE
protocol \citep{reid2026graphrope}; \emph{Exact} and \emph{Sparse} use
these settings with matched initializations and identical trainable-parameter
counts.

\begin{table}[ht]
    \centering
    \caption{Shared hyperparameters of the matched synthetic study
    (Section~\ref{sec:synthetic-experiments}). Both \emph{Exact} and
    \emph{Sparse} arms use these settings with matched initializations.}
    \label{tab:synthetic-hyperparameters}
    \small
    \setlength{\tabcolsep}{5pt}
    \begin{tabular}{ll}
        \toprule
        Hyperparameter & Value \\
        \midrule
        Backbone layers & 4 \\
        Hidden width & 32 \\
        Attention heads & 1 \\
        Dropout & 0.2 \\
        Batch size & 16 \\
        Attention & full \\
        Pooling & mean \\
        Auxiliary PE & LapPE \\
        Optimizer & Adam \\
        Epochs & 250 \\
        Weight decay & $10^{-4}$ \\
        LR schedule & cosine decay \\
        Trainable parameters & 39,575 \\
        \emph{Sparse} depth $L$ & 8 or 16 \\
        \bottomrule
    \end{tabular}
\end{table}

The completed 110-trial GIN synthetic search uses learning rates
$\{10^{-4},2\times10^{-4}\}$, $L\in\{8,16\}$ for the family grid,
and $z\in\{0.6,0.8\}$ where applicable. NoPE has no effective $L$ or $z$;
the values in its configuration files are ignored. Table~\ref{tab:gin-winners}
records the validation winner for every dataset and transport family so the
ablation is reproducible.

\begin{table*}[ht]
    \centering
    \caption{Validation-selected configurations in the completed GIN
    synthetic grid. Metric is normalized RMSE (lower is better).}
    \label{tab:gin-winners}
    \small
    \setlength{\tabcolsep}{3.5pt}
    \begin{tabular}{llllll}
        \toprule
        Dataset & Method & LR & $L$ & $z$ & Validation \\
        \midrule
        Mono-0 & NoPE & $1\mathrm{e}{-4}$ & -- & -- & 0.046553 \\
        Mono-0 & AW & $2\mathrm{e}{-4}$ & 16 & 0.6 & 0.031604 \\
        Mono-0 & AW-NB & $2\mathrm{e}{-4}$ & 16 & 0.8 & 0.029987 \\
        Mono-0 & AW-NB-MS & $2\mathrm{e}{-4}$ & 8 & 0.8 & 0.032441 \\
        \addlinespace
        Mono-5 & NoPE & $2\mathrm{e}{-4}$ & -- & -- & 0.044055 \\
        Mono-5 & AW & $2\mathrm{e}{-4}$ & 16 & 0.8 & 0.026159 \\
        Mono-5 & AW-NB & $2\mathrm{e}{-4}$ & 16 & 0.8 & 0.031583 \\
        Mono-5 & AW-NB-MS & $2\mathrm{e}{-4}$ & 16 & 0.8 & 0.040315 \\
        \addlinespace
        Mono-10 & NoPE & $2\mathrm{e}{-4}$ & -- & -- & 0.036058 \\
        Mono-10 & AW & $2\mathrm{e}{-4}$ & 8 & 0.8 & 0.020206 \\
        Mono-10 & AW-NB & $2\mathrm{e}{-4}$ & 8 & 0.6 & 0.022442 \\
        Mono-10 & AW-NB-MS & $2\mathrm{e}{-4}$ & 16 & 0.8 & 0.023225 \\
        \addlinespace
        Mono-15 & NoPE & $2\mathrm{e}{-4}$ & -- & -- & 0.026099 \\
        Mono-15 & AW & $2\mathrm{e}{-4}$ & 16 & 0.8 & 0.014582 \\
        Mono-15 & AW-NB & $2\mathrm{e}{-4}$ & 8 & 0.8 & 0.015003 \\
        Mono-15 & AW-NB-MS & $2\mathrm{e}{-4}$ & 16 & 0.8 & 0.016190 \\
        \addlinespace
        Watts SPD & NoPE & $2\mathrm{e}{-4}$ & -- & -- & 0.016607 \\
        Watts SPD & AW & $2\mathrm{e}{-4}$ & 8 & 0.8 & 0.006559 \\
        Watts SPD & AW-NB & $2\mathrm{e}{-4}$ & 8 & 0.8 & 0.010137 \\
        Watts SPD & AW-NB-MS & $1\mathrm{e}{-4}$ & 16 & 0.8 & 0.016255 \\
        \bottomrule
    \end{tabular}
\end{table*}

\begin{table}[t]
    \centering
    \caption{\textbf{GIN synthetic confirmation.} Test nRMSE $\downarrow$,
  mean$_{\pm\mathrm{SD}}$ over two matched runs after 250 epochs, with
  configurations frozen from the validation-only grid. Bold marks
  the lowest unrounded mean.}
 \label{tab:gin-synthetic-results}
    \small
    \setlength{\tabcolsep}{5pt}
    \begin{tabular}{llllll}
        \toprule
        Dataset & GIN-NoPE & GIN-AW-RoPE & GIN-AW-NB & GIN-AW-NB-MS & Winner \\
        \midrule
        Monochromatic-0 & $0.047_{\pm 0.000}$ & $0.035_{\pm 0.001}$ & $\mathbf{0.034_{\pm 0.002}}$ & $0.036_{\pm 0.002}$ & AW-NB \\
Monochromatic-5 & $0.041_{\pm 0.000}$ & $\mathbf{0.026_{\pm 0.001}}$ & $0.031_{\pm 0.001}$ & $0.032_{\pm 0.005}$ & AW \\
Monochromatic-10 & $0.034_{\pm 0.000}$ & $\mathbf{0.020_{\pm 0.000}}$ & $0.022_{\pm 0.001}$ & $0.022_{\pm 0.000}$ & AW \\
Monochromatic-15 & $0.024_{\pm 0.000}$ & $\mathbf{0.016_{\pm 0.000}}$ & $0.016_{\pm 0.000}$ & $0.016_{\pm 0.000}$ & AW \\
Watts--Strogatz SPD & $0.013_{\pm 0.001}$ & $\mathbf{0.006_{\pm 0.001}}$ & $0.007_{\pm 0.002}$ & $0.011_{\pm 0.001}$ & AW \\
\bottomrule

    \end{tabular}
\end{table}

\subsection{GIN Optimization Ablation}
\label{app:cluster-stability}

This ablation isolates how an AW branch enters a deep network. On
Cluster, all arms use four GIN layers, hidden width 32, batch size 16,
$L=16$, $z=0.6$, learning rate $2\times10^{-4}$, 60 epochs, and four paired
matched runs. Checkpoints are selected only by validation accuracy.
Table~\ref{tab:cluster-stability-ablation} reports mean $\pm$ sample SD of
test accuracy at that epoch.

\begin{table}[htbp]
    \centering
    \caption{\textbf{AW-RoPE Cluster stability ablation, four matched runs.}
    All arms are paired and budget-matched. This table isolates \emph{Sparse} AW-RoPE injection. $\Delta$ is relative to NoPE test accuracy.}
    \label{tab:cluster-stability-ablation}
    \small
    \setlength{\tabcolsep}{3.2pt}
    \begin{tabular}{llll}
        \toprule
        Injection & Validation $\uparrow$ & Test $\uparrow$ & $\Delta$ \\
        \midrule
        NoPE
          & $0.4487{\pm}0.0286$ & $0.4502{\pm}0.0276$ & -- \\
        AW-RoPE ReZero
          & $\mathbf{0.6457{\pm}0.0022}$
          & $\mathbf{0.6461{\pm}0.0024}$ & $+0.1959$ \\
        AW-RoPE ReZero + norm.
          & $0.6265{\pm}0.0057$ & $0.6263{\pm}0.0057$ & $+0.1762$ \\
        AW-RoPE Single + norm.
          & $0.5402{\pm}0.0128$ & $0.5382{\pm}0.0138$ & $+0.0880$ \\
        \bottomrule
    \end{tabular}
\end{table}

All 12 completed AW trials in this ablation beat their matched NoPE arm.
The zero-initialized, per-layer ReZero branch is best; finite-sum normalization
reduces its mean by 1.97 points, and a single pre-backbone injection loses a
further 8.81 points. The learned ReZero gates move away from zero, while the
within-graph node cosine no longer collapses near one. This supports a
specific interpretation: exact identity initialization fixes an optimization
problem, whereas the analytic transport supplies the eventual gain.

\FloatBarrier

\section{Completed Mechanism, Sensitivity, and Efficiency Study}
\label{app:future-grid}

The revision study is complete: 27 mechanism runs, 24 additional sensitivity
runs, and 45 retained isolated timing runs. The sensitivity grid reuses the three
unrestricted $(L,z)=(8,0.8)$ first-run mechanism runs, giving 27 cells
from 51 unique accuracy runs. The sensitivity grid uses the first run
of each arm; additional repeats were deferred for budget. All retained results, source hashes and plot/table generation are
archived with the artifact.

\subsection{Frozen field controls and circulation}
\label{app:field-control-results}

Monochromatic-0 uses the synthetic dense-attention backbone, 250 epochs,
batch size 16, and learning rate $10^{-4}$. Peptides-struct uses 200 epochs,
batch size 128, and learning rate $3\times10^{-4}$; MolHIV uses 100 epochs,
batch size 32, and learning rate $10^{-4}$. Both real tasks use the frozen
softmax FAVOR+ backbone. Every field arm uses three matched runs,
$L=8$, fixed $z=0.8$, learned rotary frequencies, and all insertion layers.
The backbone, optimizer settings, epoch budget, and split are shared within
each task. Scorer parameters and initialization RNG consumption are matched;
the zero-field scorer is inactive. The gradient
scorer produces bounded node potentials whose edge differences are flat.
The unrestricted scorer produces bounded antisymmetric edge values.
Neither the gradient arm nor the zero arm removes sparse propagation.

Each run contributes test at its validation-best checkpoint. Synthetic
RMSE is divided by 25; bars and subscripts show the sample SD across all
three matched runs. The arms vary the field family, which can change both phase
and transport amplitude.

\begin{figure}[t]
    \centering
    \includegraphics[width=\linewidth]{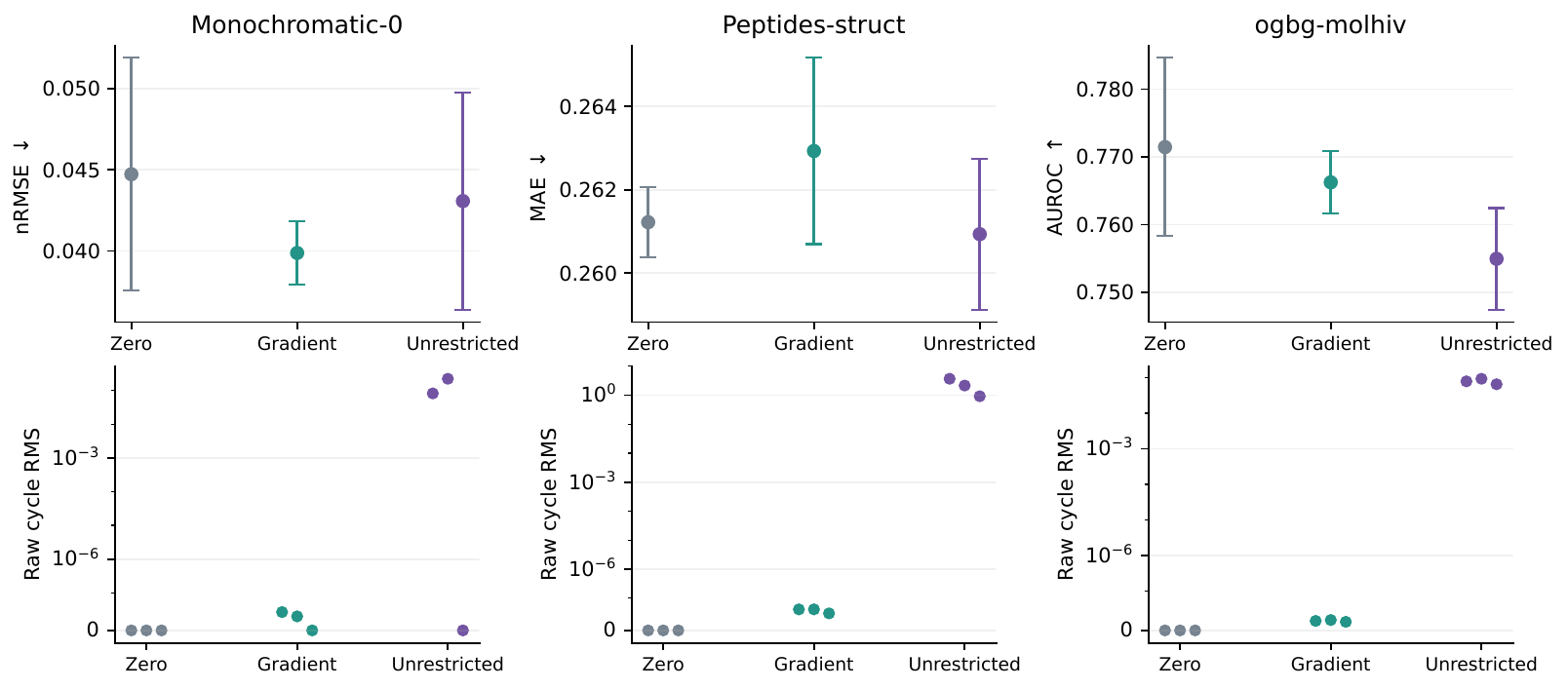}
    \caption{\textbf{Prediction and circulation under different edge fields.} Top: paired test metrics at each validation-best checkpoint,
    mean and sample SD over three runs. Bottom: one point per run, obtained
    by averaging raw fundamental-cycle RMS across insertion layers with
    cycles in the first validation batch. The symmetric-log axis retains exact zeros. The
    one unrestricted Monochromatic-0 run is flat on this diagnostic batch.}
    \label{fig:field-control-results}
\end{figure}

For the fixed first validation batch at the selected checkpoint, we build a
deterministic breadth-first spanning forest on the simple undirected support.
For every non-tree edge, its discrepancy from the integrated tree potential
is the signed sum along the corresponding fundamental cycle. We record all
these cycle sums at each insertion layer. Across the gradient controls the
largest absolute cycle sum is $7.153\times10^{-7}$, consistent with
floating-point residuals; the zero controls are exactly zero.
Eight unrestricted runs have a nonzero layer-averaged cycle RMS on this
batch; one Monochromatic-0 run has zero. Cycle RMS is measured on this
batch in the chosen fundamental basis,
using raw $\Phi_C$ rather than frequency-scaled phases modulo $2\pi$.

The gradient field has the best Monochromatic-0 mean. On PascalVOC-SP,
the unrestricted field attains the highest mean
($0.37568_{\pm 0.01188}$, against $0.37408_{\pm 0.01999}$ for zero and
$0.37228_{\pm 0.01248}$ for gradient). For Peptides-struct,
unrestricted is only marginally better than zero in mean and wins on one of
three individual runs. For MolHIV, zero outperforms unrestricted on all
three runs, despite the latter's nonzero measured circulation. The benefit of learned phases therefore varies by task and is not
monotone in circulation magnitude.

\subsection{Validation sensitivity}

The retained grid is $L\in\{4,8,16\}$ and $z\in\{0.4,0.6,0.8\}$ on all
three tasks, with the first run fixed in advance. The field is unrestricted and
the task settings match the preceding experiment. Each cell is the best
validation metric over the full epoch budget. Test scores are not used to
rank cells or update the headline configuration.

\begin{figure}[t]
    \centering
    \includegraphics[width=\linewidth]{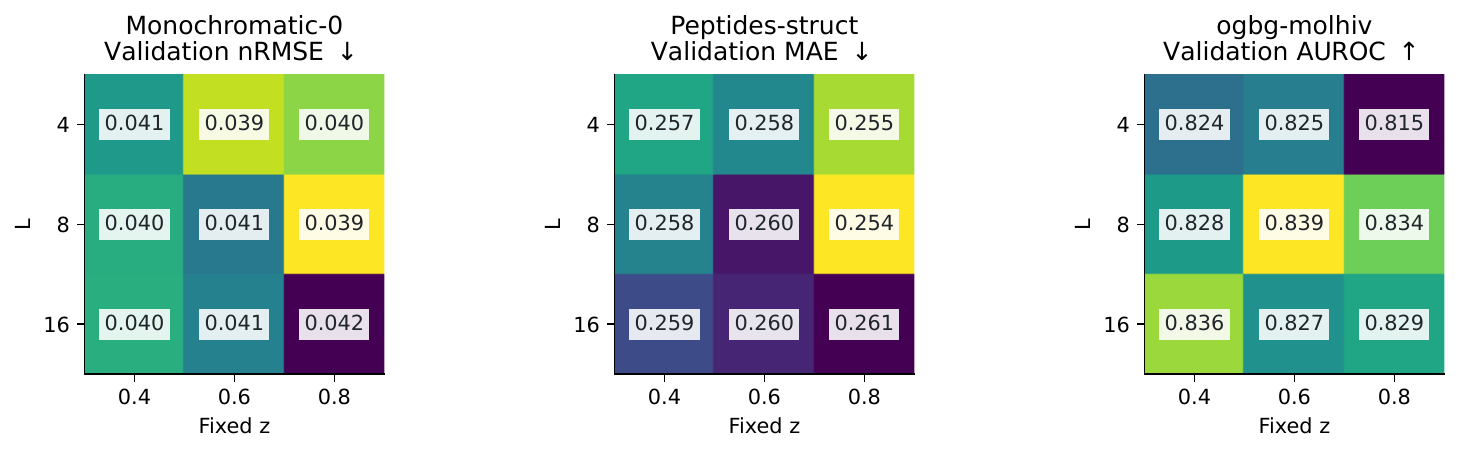}
    \caption{\textbf{Exploratory $L\times z$ sensitivity, fixed run.}
    Each cell shows the best validation score of one run, to three decimals. Color scales are
    independent across tasks. Brighter cells indicate better scores.
    All 27 cells are retained, including the three reused mechanism cells.}
    \label{fig:section46-sensitivity}
\end{figure}

The best grid cells are $(8,0.8)$ on Monochromatic-0 and Peptides-struct,
and $(8,0.6)$ on MolHIV. Their respective validation ranges are
$0.039$--$0.042$ nRMSE, $0.254$--$0.261$ MAE, and $0.815$--$0.839$ AUROC.
The validation score varies non-monotonically with walk length.

\subsection{Isolated full-model time and allocated memory}

Each of the 15 retained dataset--method configurations has three independent timing
runs on an isolated H200 GPU. We exclude one warm-up epoch, average three
subsequent complete training epochs within a run, then report the mean and
sample SD across runs. Each timing process occupies its GPU alone. The timed
region includes data iteration, host-to-device transfer, full-model
forward/backward, optimizer updates, and training logger updates. Validation,
test evaluation, and checkpoint writes are outside this training-epoch
measurement. The peak is PyTorch allocated device memory over the measured
epochs.

\begin{table}[t]
    \centering
    \caption{\textbf{Isolated H200 full-model cost.} Three timing runs,
    mean$_{\pm\mathrm{SD}}$ to three decimals. Peak GiB is each run's maximum
    allocated memory over its three measured epochs. Loader seconds use
    existing prepared caches and are listed separately. The real tasks use FAVOR+.}
    \label{tab:section46-efficiency}
    \scriptsize
    \setlength{\tabcolsep}{3pt}
    \begin{tabular}{lllll}
        \toprule
        Dataset & Method & Train s/epoch & Peak GiB & Loader s \\
        \midrule
        Monochromatic-0 & NoPE & $7.818_{\pm 0.076}$ & $0.097_{\pm 0.000}$ & $8.601_{\pm 0.049}$ \\
Monochromatic-0 & WIRE & $10.292_{\pm 0.535}$ & $0.098_{\pm 0.000}$ & $8.580_{\pm 0.047}$ \\
Monochromatic-0 & AW-RoPE, $L=4$ & $16.548_{\pm 0.202}$ & $0.115_{\pm 0.000}$ & $8.807_{\pm 0.102}$ \\
Monochromatic-0 & AW-RoPE, $L=8$ & $21.516_{\pm 0.115}$ & $0.130_{\pm 0.000}$ & $8.730_{\pm 0.295}$ \\
Monochromatic-0 & AW-RoPE, $L=16$ & $31.770_{\pm 0.365}$ & $0.159_{\pm 0.000}$ & $8.789_{\pm 0.196}$ \\
\midrule
Peptides-struct & NoPE & $3.493_{\pm 0.043}$ & $2.801_{\pm 0.026}$ & $31.090_{\pm 0.673}$ \\
Peptides-struct & WIRE & $3.766_{\pm 0.050}$ & $3.027_{\pm 0.032}$ & $31.233_{\pm 0.329}$ \\
Peptides-struct & AW-RoPE, $L=4$ & $5.213_{\pm 0.066}$ & $4.740_{\pm 0.058}$ & $32.311_{\pm 0.928}$ \\
Peptides-struct & AW-RoPE, $L=8$ & $6.414_{\pm 0.024}$ & $6.276_{\pm 0.110}$ & $31.448_{\pm 0.460}$ \\
Peptides-struct & AW-RoPE, $L=16$ & $8.958_{\pm 0.048}$ & $9.341_{\pm 0.219}$ & $30.926_{\pm 0.378}$ \\
\midrule
ogbg-molhiv & NoPE & $74.993_{\pm 0.446}$ & $0.540_{\pm 0.001}$ & $35.733_{\pm 0.173}$ \\
ogbg-molhiv & WIRE & $83.769_{\pm 0.840}$ & $0.549_{\pm 0.005}$ & $35.583_{\pm 0.454}$ \\
ogbg-molhiv & AW-RoPE, $L=4$ & $116.799_{\pm 1.356}$ & $0.731_{\pm 0.013}$ & $36.010_{\pm 0.123}$ \\
ogbg-molhiv & AW-RoPE, $L=8$ & $141.839_{\pm 1.941}$ & $0.870_{\pm 0.024}$ & $35.705_{\pm 0.741}$ \\
ogbg-molhiv & AW-RoPE, $L=16$ & $189.056_{\pm 0.652}$ & $1.144_{\pm 0.036}$ & $35.001_{\pm 0.391}$ \\
\bottomrule

    \end{tabular}
\end{table}

AW time and allocated memory increase with $L$ on all three measured tasks.
At $L=8$, AW takes $2.091\times$, $1.703\times$, and $1.693\times$ WIRE's
epoch time on Monochromatic-0, Peptides-struct, and MolHIV, respectively,
and allocates more memory in all three cases. Size-dependent layer timings
and their backend dispatch effects are examined separately in Appendix~\ref{app:scaling}.

Loader initialization uses existing prepared caches. The inventory covers
enumerable PyG processed files; spectral preparation and the separate AW-only
ReLU replay are outside these timings.

\section{Multi-Route Capability and Matched Mixing Controls}
\label{app:multiroute-controls}

This study compares learned phase transport with ordinary multi-hop mixing,
using a fixed-graph diagnostic and a paired predictive experiment.

\subsection{A fixed-input two-route phase intervention}
\label{app:two-route-intervention}
\begin{figure}[t]
\centering
\input{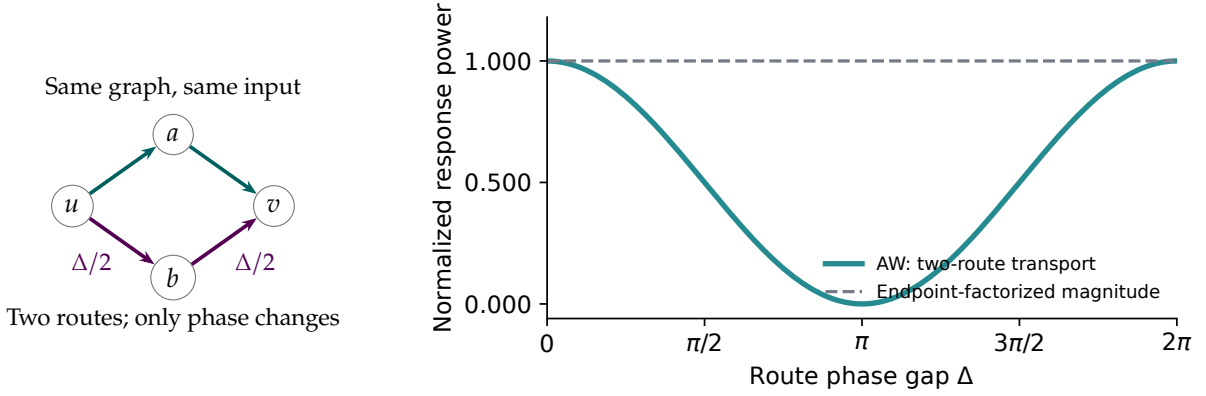}
\caption{\textbf{Route phase changes transport on a fixed graph.}
The diamond has both edge orientations, uniform row transition weights,
and a unit real input at $v$. Arrows show walk indices from $u$ toward $v$;
the row action gathers that input back to $u$. Only $a_{ub}=a_{bv}=\Delta/2$
and their negatives on reversed edges vary; all other displacements are
zero, with $\omega=1$. Every displacement remains in $[-\pi,\pi]$.
At $L=2$, the two routes have equal weight $1/4$.
AW's normalized response power changes from reinforcement to cancellation.
The dashed line is the magnitude of an endpoint-factorized coefficient
with the same phase-free mixing weights, not a measured WIRE network.
}
\label{fig:two-route-response}
\end{figure}

For the graph in Figure~\ref{fig:two-route-response}, a feature supported
only at $v$ reaches $u$ for the first time at depth two. With $z=0.8$,
\[
 [\mathcal F^{(2)}_{z,1}]_{uv}
   =\frac{z^2}{4}(1+e^{i\Delta}),\qquad
 \frac{|[\mathcal F^{(2)}_{z,1}]_{uv}|^2}{z^4/4}
   =\cos^2(\Delta/2).
\]
We evaluate 129 uniformly spaced phases using the production real
paired-channel sparse recurrence in float64 and compare against this
closed form. The dashed comparison line in Figure~\ref{fig:two-route-response} is the magnitude of an endpoint-factorized coefficient with the same phase-free mixing weights, not a measured WIRE network.
The graph, input and phase-free transition weights stay fixed; only the
edge phases vary. A scalar
endpoint factor $\overline{g(u)}g(v)$ has unit magnitude and cannot change
the coefficient magnitude by itself. The diagnostic concerns that rotary coefficient, while a WIRE backbone
can also encode graph structure in its features.

\subsection{Two complementary prediction tasks}
\label{app:constructed-prediction}

\begin{table}[t]
\centering
\small
\setlength{\tabcolsep}{4pt}
\caption{\textbf{Capability ladder on the constructed tasks.} Each row
adds one ingredient. Check marks solve at high accuracy, crosses sit at
chance. ``(proof)'' marks entries covered by an impossibility statement;
the route task uses endpoint readout and full networks, the
feature-conditional task uses full networks on unseen sizes. The ladder
compares field classes within the walk-transport operator interface;
learned restriction-map transports are discussed in
Appendix~\ref{sec:connection-relation}.}
\label{tab:capability-ladder}
\begin{tabular}{@{}llcc@{}}
\toprule
Field class & Ingredients & Route & Feature-cond.\\
\midrule
Mixing (Zero) & transport only & $\times$ & $\times$ (proof)\\
WIRE & endpoint phases & $\times$ & $\times$\\
Static field (random, learned) & circulation & -- & $\times$ (proof)\\
Dynamic local gradient & adaptivity, local & $\times$ & $\times$ (proof)\\
Dynamic nonlocal Flat & adaptivity, nonlocal & -- & \checkmark\ (unstable)\\
\emph{Sparse} AW (ours) & circulation + adaptivity & \checkmark & \checkmark\\
\bottomrule
\end{tabular}
\end{table}

\begin{table}[t]
\centering
\small
\caption{The two constructed tasks test different capabilities under
specified readouts. These are comparison axes, not a nested hierarchy
of all possible graph models.}
\label{tab:constructed-axes}
\begin{tabularx}{\linewidth}{@{}lXX@{}}
\toprule
Task & Collision or empirical failure & Successful construction\\
\midrule
Route interference, endpoint & Mixing (Zero), WIRE and pointwise
gradient transport & AW route-phase interference\\
Feature-conditioned cycle, full block & Static shell transport and
pointwise gradients; WIRE and random fixed fields fail empirically &
Local edge AW or dynamically projected nonlocal Flat\\
\bottomrule
\end{tabularx}
\end{table}

\paragraph{Route-interference protocol (Table~\ref{tab:route-interference}).}
Two length-eight routes connect marked endpoints, forming an undirected
16-node cycle. Two separated $A\!\to\!B$ motifs occupy internal nodes.
The label distinguishes both motifs on one route from one motif on each
route. Moving a complete motif produces an opposite-label pair with the
same endpoint features and feature multisets at each endpoint-distance
shell. Six input channels encode the endpoints, motif types and two
Gaussian nuisance values (SD $0.1$); motif strengths vary in
$[0.95,1.05]$. Pairs share nuisance, node permutations and independently
rotated/reflected bases of the first two nontrivial Laplacian eigenspaces.
The latter enter WIRE rotations, not the pointwise encoder.

The endpoint classifier applies one transport to pointwise encodings and
reads the marked endpoint; the attention classifier uses a full block.
Both use width 16, fixed $z=0.8$, and $L=8$ for Sparse. Exact replaces
only the truncated transport by a complete resolvent solve. All arms have
2,618 parameters with endpoint readout or 4,810 with attention; smaller
positional modules receive active pointwise adapters. Shared network
initializations match within runs. There are 1,024/256/1,024
train/validation/test pairs; pairs stay within splits. Adam uses learning
rate $0.003$, no weight decay, 32 pairs per update and 2,000 updates.
Selection minimizes validation BCE, with the earliest tie, before test
evaluation. All 50 Sparse/control and ten Exact fits are retained.

For endpoint readout, $[p(P)]_{uw}$ depends only on the cycle distance
from $u$ to $w$. Identical shell-feature multisets therefore give
identical outputs for any phase-free polynomial $p(P)$. A pointwise
gradient $a_{ij}=b(x_j)-b(x_i)$ gives the gauge-conjugated action
$G^*p(P)G$, where $G_{ii}=\exp(i\omega b(x_i))$, and preserves the
same collision. Endpoint-factorized rotation also preserves it.
AW can instead accumulate different phases on the two routes, exposing
the response power in Figure~\ref{fig:two-route-response}. With a full
attention block, mixing and gradient controls succeed too. This task
isolates a transport capability under its endpoint readout condition.

\subsection{Feature-conditioned classification on anonymous cycles}
\label{app:inductive-fields}

\paragraph{Graph family and held-out information.}
On $C_n$, $n\geq12$, the two color arrangements are
\[
 A_n=0^{n-12}\,012001100100,\qquad
 B_n=0^{n-12}\,011001200100,
\]
with labels zero and one. Here $0^{n-12}$ denotes a run of zero-colored
nodes, not a numerical power. Each node receives a constant, a three-way
color indicator and two color-specific Gaussian nuisance values with
SD $0.1$. All nodes of a color have identical features within a graph;
the nuisance draw is shared within each opposite-label pair. A pair has
identical topology, size and feature multiset, but different feature
arrangements. Node order is randomly permuted, and no persistent node or
edge identities enter the model. Graph-only spectral bases are randomly
rotated and reflected within each of the first two nontrivial eigenspaces,
shared within the pair, for the WIRE arms.

Training sizes are $\{12,16,20,24,28,32\}$ with 64 pairs per size;
validation sizes are $\{14,18,22,26,30\}$ with 32 pairs per size;
test sizes are $\{15,17,19,21,23,25,27,29,31\}$ with 128 pairs per size.
Each split uses independent paired instances and disjoint sizes. The
primary test measures unseen sizes within the training size range;
separate 48/64-node diagnostics assess extrapolation beyond that range.

\paragraph{Eight matched arms.}
Let $D=I+zP+z^2P^2$, $z=0.8$, and let $h_i$ denote the pointwise
encoded/adapted feature before transport. Mixing (Zero) uses $D$ on Q and K.
WIRE uses nodewise rotation from the spectral coordinates; WIRE~+~mixing
applies $D$ before this rotation. The learned static polynomial has
separate arbitrary trainable matrices $M_k^Q,M_k^K$ for $k=0,1,2$:
$Q'=\sum_k(P^kQ)M_k^Q$ and likewise for K. The same matrices apply at
every graph size. It includes, and is stronger than, the deterministic
static shell transports in Proposition~\ref{prop:inductive-collision}.

The random fixed field draws independent $N(0,0.8^2)$ displacements on
one orientation of each edge and negates them on reversal. Each new
paired instance receives a new draw, shared across its labels and fixed
thereafter, including at validation/test sizes. Its rotary frequencies
are trainable. This supplies random symmetry breaking but no persistent
edge identity. It is an empirical control, outside the deterministic
static-field theorem. The topology-only Static and Flat fields of
Appendix~\ref{app:static-flat} both equal zero on cycles, so their result
is the Mixing (Zero) row rather than an additional independent fit.

AW uses the production local antisymmetric scorer on $(h_i,h_j)$.
Dynamic local gradient uses the same scorer parameters to form a
bounded pointwise potential $b(h_i)$, then takes its exact difference
$a_{ij}=b(h_j)-b(h_i)$. The dynamic nonlocal Flat arm uses the same
initial edge scorer as AW and projects its output onto gradients.
Choose either cycle orientation, let $\widetilde a_i$ be the scorer's
displacement on edge $i\to i+1$, and set
\begin{equation}
 a_i^{\mathrm{flat}}
 =\widetilde a_i-\frac{1}{n}\sum_{j=0}^{n-1}\widetilde a_j,
 \qquad a_{i+1,i}^{\mathrm{flat}}=-a_i^{\mathrm{flat}}.
 \label{eq:inductive-flat-projection}
\end{equation}
This is the unweighted incidence-matrix least-squares projection on a
cycle. Its sum is zero, so a potential exists; both that potential and
the projection depend on features beyond a single node. Reversing the
chosen orientation leaves the projected directed field unchanged. It
provides no canonical node identity. This dynamic Flat is distinct from
the fixed Flat in Appendix~\ref{app:static-flat}.

\paragraph{Training and selection.}
All arms have 4,810 trainable parameters, width 16, one full softmax
attention head, residual/normalization and feedforward processing, and
mean pooling with a trainable classifier. There is one graph-aware Q/K
transport with $L=2$, not a frozen endpoint readout. Smaller positional
modules receive active pointwise adapters. Common network tensors have
identical initial-state hashes within each run; the three dynamic-field
arms have identical full initial tensors. Matrix initializations use a
shared gain of 3. Adam uses learning rate $0.0003$, no weight decay and
10,000 updates. Each update samples one training size uniformly and 16
pairs with replacement. With graph logits $\ell^-,\ell^+$ and ordinary
binary cross-entropy $\mathcal L_{\rm BCE}$, every arm minimizes
\[
 \mathcal L=\mathcal L_{\rm BCE}
 +\mathbb E_{\rm pairs}\operatorname{softplus}
       \bigl(-100(\ell^+-\ell^-)\bigr).
\]
Pairing uses training labels only. The inference function processes one
graph independently; neither another graph nor the pair index is an
input. Every 100 updates, validation BCE is averaged equally across
sizes. Its minimum selects the checkpoint, with the earliest tie;
test instances are generated only after training and selection finish.
The primary metric is the equally weighted size-average accuracy.

Development used fresh runs and validation data. An earlier ordinary-BCE
cohort gave AW $83.91\pm21.35\%$ and is retained separately.
The final protocol freezes fresh data and fresh runs; the two protocols
are not pooled, and no run is excluded.

\begin{proposition}[Complete-block collision for static and local-potential transport]
\label{prop:inductive-collision}
For paired $A_n,B_n$, a pointwise encoder followed by a degree-two static
matrix-polynomial Q/K transport, one full attention block and invariant
pooling has identical logits for both labels, for every choice of its
parameters. The same holds for a degree-two transport whose connection
is the gradient of a pointwise node potential. Consequently each such
classifier has exactly $50\%$ accuracy on the balanced paired set.
\end{proposition}

\begin{proof}
Define a node signature as its color together with the color histograms
at distances one and two. The two patterns have identical multisets of
these signatures at $n=12$. Inserting $k=n-12\geq1$ zeros into the
common length-three zero run removes, in both patterns, one signature
$(0;2,0,0;0,2,0)$, adds two $(0;2,0,0;1,1,0)$, and adds $k-1$
$(0;2,0,0;2,0,0)$. Thus equality holds for all $n\geq12$.
The patterns are nonetheless nonisomorphic: the unique color-2 node
fixes an isomorphism and its differently colored neighbors fix the
orientation, but the next four colors towards its color-0 neighbor
are $0011$ in $A_n$ and $0010$ in $B_n$.

Pointwise features depend only on color within each graph. On a cycle,
$Ph$ depends on the distance-one histogram, and $P^2h$ on the node's
color and distance-two histogram, including the backtracking term
$h_i/2$. Hence each static transported query/key is a function of this
signature, even with arbitrary channel matrices. For a pointwise
gradient, the gauge factors in $G^*(I+zP+z^2P^2)G$ also depend only on
the endpoint colors, preserving the signature dependence.
The multisets of complete tuples $(h_i,Q'_i,K'_i,V_i)$ therefore coincide.
Global attention, residual/normalization and pointwise feedforward
processing commute with the matching permutation, and invariant pooling
gives identical graph logits.
\end{proof}

The static-field scope follows from cycle symmetries. A deterministic
topology-only real antisymmetric displacement that is equivariant to
node relabeling is constant on oriented edges by rotations/reflections,
and antisymmetry makes it zero. A directly specified $U(1)$ connection
also admits the constant $-1$ connection: Hermitian reversal requires
$U=\overline U$, so $U=\pm1$. Both cases give static shell polynomials
covered by the proposition, including the uniform $\pi$ connection on
odd cycles. Random edge marks, spectral symmetry breaking, additional
graph-aware layers, larger $L$, or arbitrary nonlocal potentials are
outside the proof. WIRE's near-chance result is empirical, not this
theorem applied to all spectral graph networks.

\paragraph{Seed variability and extrapolation.}
Table~\ref{tab:inductive-runs} retains every run. Dynamic Flat's low
run is a partial fit, not a collapse to chance, with $72.81\%$
validation and $74.26\%$ test accuracy. Its other four runs
are at least $99.91\%$. The observed spread describes sensitivity to
the shared training protocol; it establishes neither an AW--Flat
accuracy advantage nor an intrinsic stability ordering. The 48/64-node
results also limit the size-generalization claim to the primary
within-range test. No normalization mechanism is inferred from these
training outcomes.

\begin{table}[t]
\centering
\small
\caption{All AW and dynamic nonlocal Flat runs, accuracy (\%). Test is
the primary unseen-size split; 48/64 nodes is a separate extrapolation
diagnostic. Zeroing phases gives $50\%$ for every run on the primary
test.}
\label{tab:inductive-runs}
\begin{tabular}{@{}llrrrrr@{}}
\toprule
Arm & Evaluation & 20 & 21 & 22 & 23 & 24\\
\midrule
\emph{Sparse} AW & Test & 100.00 & 100.00 & 100.00 & 93.45 & 99.91\\
\emph{Sparse} AW & 48/64 nodes & 92.97 & 81.64 & 100.00 & 50.00 & 71.88\\
Dynamic nonlocal Flat & Test & 99.91 & 100.00 & 99.96 & 74.26 & 100.00\\
Dynamic nonlocal Flat & 48/64 nodes & 50.78 & 82.81 & 100.00 & 50.00 & 97.27\\
\bottomrule

\end{tabular}
\end{table}

\paragraph{Field diagnostics and frozen-model interventions.}
We inspect the first validation pair at $n=18$ for every fixed run;
Figure~\ref{fig:inductive-mechanism} shows one representative. The learned AW edge
field changes between feature contexts. For dynamic Flat, reconstruct
the zero-mean potential by accumulating $a_i^{\rm flat}$ around the
cycle. The maximum residual in $a_i^{\rm flat}=b_{i+1}-b_i$ is below
$2\times10^{-15}$ across the runs. Nodes with identical raw
features nevertheless have potentials differing by $1.71$--$3.14$
within a graph; the local-gradient control has no such variation up to
$10^{-10}$. This verifies a nonlocal potential rather than assuming one.
There is no prescribed teacher potential to recover.

AW has nonzero cycle sums on these pairs (cycle RMS $1.91$--$2.79$),
but each pair's two labels have the same total circulation. The task
therefore probes the arrangement of feature-conditioned transport, not
a label-dependent total holonomy. Dynamic Flat's successful runs also
show that nonzero circulation is not necessary here. Setting all edge
displacements to zero while holding each fitted AW or dynamic Flat
network fixed returns every run to $50\%$ test accuracy. This
intervention establishes reliance on phases in the fitted solution;
the collision proof and independently trained controls establish the
specified model-class separation. Checkpoint predictions, validation
choices, source/data hashes, permutation covariance and graph-wise
inference independence have been replayed.

\begin{figure}[t]
\centering
\includegraphics[width=\linewidth]{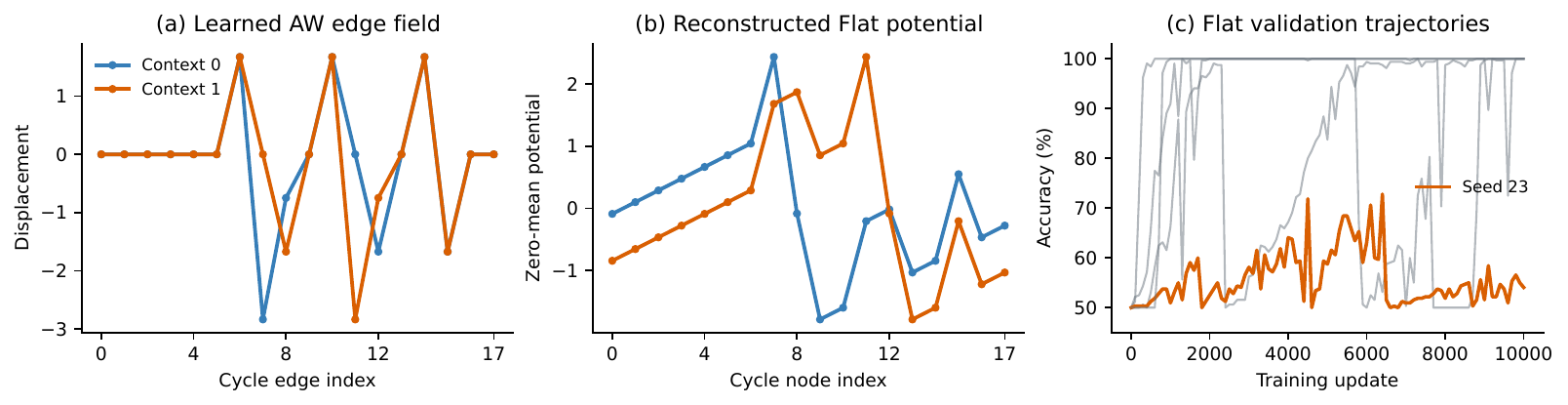}
\caption{\textbf{Learned fields and Flat run variability.}
(a) AW edge displacements for the first paired validation example at
$n=18$. (b) The dynamic Flat control's reconstructed zero-mean
potential on that pair. Cycle indices are used only for this display,
not supplied as model features. (c) All five Flat validation trajectories;
the partially fitted run is highlighted. No smoothing is applied.}
\label{fig:inductive-mechanism}
\end{figure}

\subsection{Completed matched predictive study}
 Monochromatic-0 and Peptides-struct each use all four arms over three
matched runs.
The epoch budgets, learning rates, and batch sizes are inherited from
Appendix~\ref{app:field-control-results}. Synthetic attention remains dense;
the real task uses the same generalized ReLU linear-attention kernel.
The field scorer configuration is inherited without tuning.

Write $D_{L,z}=\sum_{k=0}^L z^kP^k$, and let $R_g$ denote WIRE's
node-coordinate rotation. With $L=8$ and fixed $z=0.8$, the query/key
actions are:
\begin{center}
\begin{tabular}{ll}
\toprule
Arm & Action on each projected query or key $X$\\
\midrule
WIRE & $R_g X$\\
Phase-free mixing & $D_{L,z}X$\\
WIRE + mixing & $R_g D_{L,z}X$\\
AW-RoPE & $\mathcal F^{(L)}_{z,\omega}X$\\
\bottomrule
\end{tabular}
\end{center}
Mixing precedes WIRE rotation in the third arm. Values, backbone widths, local message
passing, and input encodings are unchanged across arms. All arms construct
the same coordinate encoder to preserve shared initialization and training
RNG consumption; its output is unused by AW and phase-free mixing.
Positional-module initialization cannot advance the RNG stream for later
shared backbone layers. Initial shared-state hashes are checked within
each task and run.

WIRE and WIRE + mixing have exactly the same parameters; the latter adds
parameter-free propagation. AW retains its scorer and frequency parameters.
Table~\ref{tab:multiroute-parameters} gives allocated, positional and unused
encoder counts. Shared $L$ and $z$ match the propagation budget.

The protocol freezes all configurations before new results.
Every epoch is checkpointed; the test value is paired with the
validation-best epoch. All three runs contribute to
$\mathrm{mean}_{\pm\mathrm{SD}}$, with three-decimal formatting.

All runs are complete and pass trajectory, checkpoint and shared
initialization checks. Table~\ref{tab:multiroute-completed} retains every
arm and run of the study.

\begin{table}[t]
 \centering
 \caption{Completed matched mixing controls. Test at validation-best,
 mean$_{\pm\mathrm{sample\ SD}}$ over three matched runs, three decimals.
 Mixing is phase-free propagation; WIRE + mixing applies propagation
 before WIRE rotation. Bold marks the best unrounded mean.}
 \label{tab:multiroute-completed}
 \small
 \setlength{\tabcolsep}{3pt}
 \resizebox{\linewidth}{!}{%
 \begin{tabular}{llllll}
 \toprule
 Dataset & Metric & WIRE & Mixing & WIRE + mixing & \emph{Sparse} AW-RoPE\\
 \midrule
 Monochromatic-0 & nRMSE $\downarrow$ & $0.073_{\pm 0.008}$ & $0.044_{\pm 0.005}$ & $0.045_{\pm 0.006}$ & $\mathbf{0.043}_{\pm 0.006}$ \\
Peptides-struct & MAE $\downarrow$ & $0.257_{\pm 0.002}$ & $0.255_{\pm 0.002}$ & $0.257_{\pm 0.001}$ & $\mathbf{0.253}_{\pm 0.001}$ \\
\bottomrule

 \end{tabular}}
\end{table}

AW-RoPE attains the lowest mean on both tasks. The matched arms
attribute most of the gain to walk transport itself, with phase
learning contributing task-dependent improvements.

\begin{table}[t]
 \centering
 \caption{Parameter accounting for the matched mixing study. Counts
 are identical across the three runs of each row. Allocated includes
 modules whose outputs are unused. The final column identifies unused coordinate-encoder parameters.}
 \label{tab:multiroute-parameters}
 \small
 \begin{tabular}{lllll}
 \toprule
 Dataset & Arm & Allocated & Positional & Unused encoder\\
 \midrule
 Monochromatic-0 & WIRE & 39644 & 320 & 0 \\
Monochromatic-0 & Mixing & 39324 & 0 & 85 \\
Monochromatic-0 & WIRE + mixing & 39644 & 320 & 0 \\
Monochromatic-0 & Sparse AW-RoPE & 39656 & 332 & 85 \\
Peptides-struct & WIRE & 506729 & 1920 & 0 \\
Peptides-struct & Mixing & 504809 & 0 & 270 \\
Peptides-struct & WIRE + mixing & 506729 & 1920 & 0 \\
Peptides-struct & Sparse AW-RoPE & 529837 & 25028 & 270 \\
\bottomrule

 \end{tabular}
\end{table}

\subsection{Training curves and attention on fixed validation graphs}
\label{app:synthetic-visuals}

We visualize the existing Monochromatic-0 WIRE and AW-RoPE runs from
the matched mixing study above. Following the graph-based presentation
of Graph-RoPE \citep{reid2026graphrope}, we show the task inputs alongside
learned attention, and plot the recorded training loss. All panels use the local logs and checkpoints.

The training panel of Figure~\ref{fig:diagnostics} includes all 250 epochs of the three runs.
Thin lines show individual runs and thick lines their arithmetic means,
using recorded MSE on the unnormalized target with a logarithmic axis
and no smoothing. AW-RoPE reaches lower training MSE in these runs;
Table~\ref{tab:multiroute-completed} gives the validation-selected test results.
Both methods use dense attention and the same backbone and training
settings, with different active positional parameter counts.

For Figure~\ref{fig:synthetic-attention}, we fix the first run and the first three
validation graphs before inspecting attention. Each method uses its own
validation-best checkpoint. We extract the final layer's single-head
softmax attention in evaluation mode and display the weights from node 0
to every key node. No parameters are updated. The graph layouts and
color scale are shared, and the full attention matrices are retained.
These are descriptive comparisons of independently trained models.

\begin{figure}[!htbp]
 \centering
 \includegraphics[width=.80\linewidth]{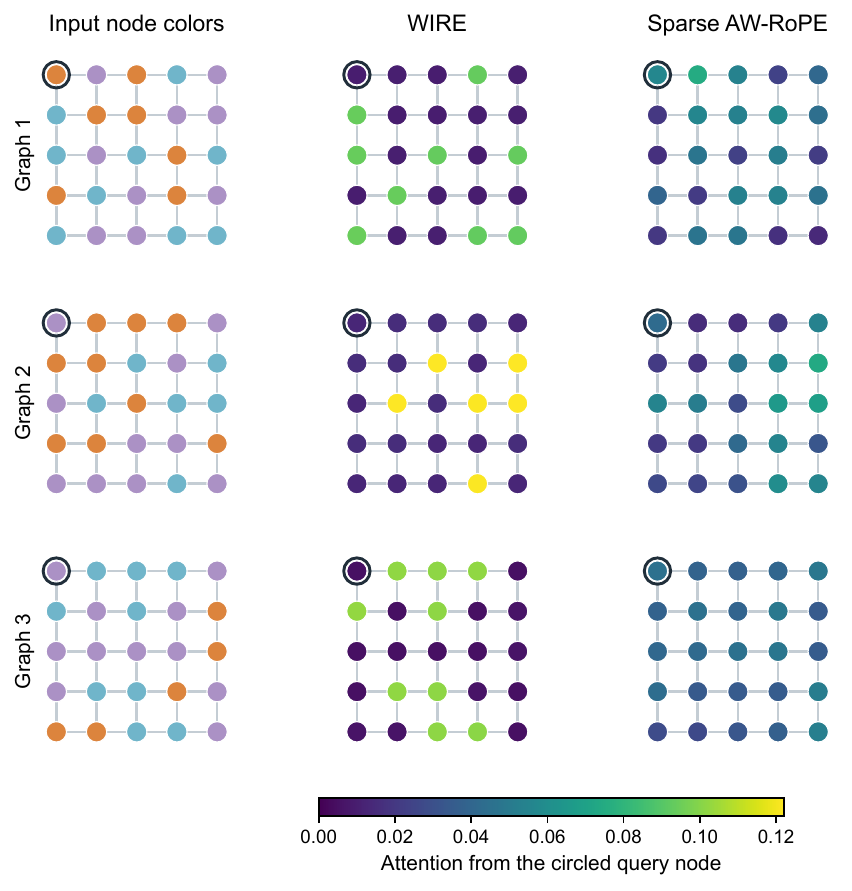}
 \caption{\textbf{Attention on fixed Monochromatic-0 validation examples.}
 Left: the three categorical node colors supplied to the task. Center and
 right: final-layer attention probabilities from the circled node for
 WIRE and \emph{Sparse} AW-RoPE. Rows are the first three validation graphs in
 their stored order. All six attention panels use one common color scale,
 without per-panel rescaling. Seed 0 and validation-best checkpoints are
 fixed.}
 \label{fig:synthetic-attention}
\end{figure}

The six complete training trajectories, checkpoint hashes, fixed graph
inputs, and attention matrices are stored with the artifact. The exporter
checks the original shared-initialization hashes, validation-best epoch,
checkpoint loading, and unchanged model state during attention
extraction.

\end{document}